\documentclass{article} %
\usepackage{iclr2027_conference,times}

\usepackage{amsmath,amsfonts,bm}

\def\eqref#1{equation~\ref{#1}}

\def\1{\bm{1}}

\def\eps{{\epsilon}}

\DeclareMathAlphabet{\mathsfit}{\encodingdefault}{\sfdefault}{m}{sl}
\SetMathAlphabet{\mathsfit}{bold}{\encodingdefault}{\sfdefault}{bx}{n}

\newcommand{\R}{\mathbb{R}}

\usepackage{hyperref}
\usepackage{url}
\usepackage{graphicx} %
\usepackage{mathtools, amssymb}
\usepackage{amsthm, amssymb}
\usepackage{caption,subcaption}

\usepackage[dvipsnames]{xcolor}
\usepackage[most]{tcolorbox}

\newtcolorbox{subfigbox}{
  enhanced,
  arc=2.5mm,               %
  boxrule=0.6pt,            %
  colframe=gray!60,         %
  colback=white,            %
  boxsep=0pt,
  left=-1pt, right=3pt,      %
  top=1.5pt, bottom=-3pt       %
}

\newtcolorbox{highlightbox}[1][]{
  enhanced,
  breakable,                %
  arc=2.5mm,               %
  boxrule=1pt,            %
  colframe=Dandelion!50,         %
  colback=yellow!3!white,     %
  top=3pt, bottom=3pt,
  left=3pt, right=3pt,
  #1                        %
}

\newtcolorbox{highlightbox2}[1][]{
  enhanced,
  breakable,                %
  arc=2.5mm,               %
  boxrule=1.2pt,            %
  colframe=MidnightBlue!70,         %
  colback=Green!3!white,     %
  top=4pt, bottom=4pt,
  left=4pt, right=4pt,
  #1                        %
}

\usepackage{pgfplots}
\usetikzlibrary{arrows.meta}

\theoremstyle{thm}
\newtheorem{definition}{Definition}[section]
\newtheorem{definition*}{Definition}
\newtheorem{proposition}{Proposition}[section]
\newtheorem{lemma}{Lemma}[section]
\newtheorem{theorem}{Theorem}[section]
\newtheorem{corollary}{Corollary}[section]

\theoremstyle{definition}
\newtheorem{remark}{Remark}[section]
\newtheorem{example}{Example}
\newtheorem*{example*}{Example}
\newtheorem*{remark*}{Remark}

\usepackage{wrapfig}
\usepackage{float}

\newcommand{\N}{\mathbb{N}}

\newcommand{\A}{\mathcal{A}}
\newcommand{\F}{\mathbb{F}}
\newcommand{\LTL}{\textbf{LTL}[\textbf{P}]}
\newcommand{\PFO}{\textbf{PFO}^2[<]}

\renewcommand{\S}{|\Sigma|}
\newcommand{\Pa}{\mathbf{P}}
\newcommand{\Past}{\Pa}
\newcommand{\SigmaStar}{\Sigma^{*}}
\newcommand{\Sub}{\operatorname{Sub}}
\newcommand{\PastSub}{\Sub_{\Past}}
\newcommand{\prof}{\operatorname{prof}}
\newcommand{\FPT}{\mathcal{FPT}}

\newcommand{\edits}[1]{#1}
\newcommand{\editsMiH}[1]{#1}

\title{Quantitative Generalization Bounds for Multi-Layer Transformers: A Few Long Strings can Beat Many Short Ones} %
\iclrfinalcopy

\author{%
  Marius Henry \textsuperscript{1,3}
  \And
  Michael Rizvi-Martel \textsuperscript{2}
  \And
  Michael Hahn \textsuperscript{1}
  \AND
      \makebox[\dimexpr\textwidth-2\tabcolsep\relax][c]{%
    \normalfont
    \begin{tabular}{l}
      \textsuperscript{1} \href{https://lacoco-lab.github.io}{LaCoCo}, Saarland University, Saarland Informatics Campus \\
      \textsuperscript{2} Mila \& Université de Montréal \qquad
      \textsuperscript{3} Zuse School ELIZA 
    \end{tabular}
  }
}

\begin{document}

\maketitle

\begin{abstract}
Transformer-based language models exhibit an astonishing capacity to generalize to longer and more difficult problem instances than those seen in training.
However, theoretical understanding of this phenomenon still lags behind. 
Unlike most machine learning models, the generalisation abilities of transformers  inherently require considering two axes. On one hand, one must examine in-distribution \textit{statistical generalisation}, which concerns the ability of a model to broaden a function from training samples to novel samples of the same distribution. On the other hand, given the sequential nature of language data, one must also consider $\textit{length generalisation}$, which reflects whether a model trained on inputs of a given length can extrapolate to longer ones.
These two notions have mostly been studied separately, leaving open how they interact. In this work, we give a joint treatment of both axes for fixed-precision transformers. We \editsMiH{give} generalisation guarantees in terms of the parameter sizes of the involved models, as well as the quantization space and the alphabet size that supply the sequences. \edits{Finally, we illustrate how the requirements in the sample count and training length trade off}. In particular, we find that \textbf{increasing training length can be exponentially more efficient than increasing the sample size}.
Experiments support predictions of the theory.
\end{abstract}

\section{Introduction}
\begin{wrapfigure}[22]{r}{0.45\textwidth}
\vspace{-16pt}
\centering
\includegraphics[width=\linewidth]{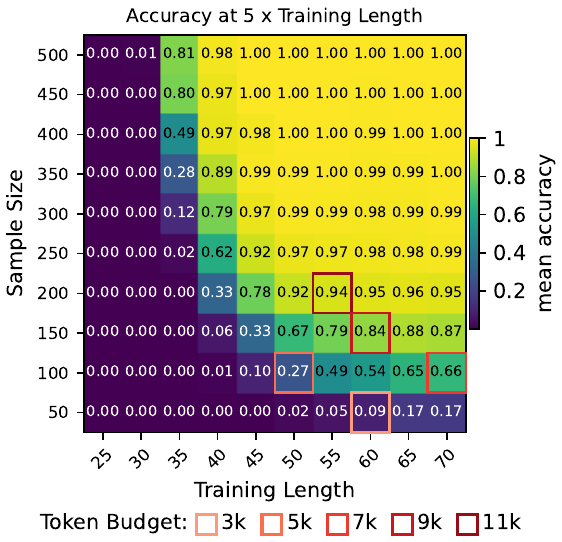}

\vspace{-4pt}
\caption{Our results predict that increasing the length of the training words can be much more token-efficient than increasing \edits{their number;} test accuracy for the language $A_{15}$ (Sec.~\ref{sec:experiments}) on words five times longer than the training length. We highlight \edits{in red} the highest-accuracy cell \edits{at five different} training token budgets.}
\end{wrapfigure}
The success of large language models (LLMs) has prompted research into the learning and generalization properties of the underlying transformer architecture.
Classical statistical learning theory provides a wealth of techniques that bound the generalization error as a function of the data size and the complexity of the hypothesis class. Recent work has brought these to bear on transformers (Covering numbers: \citet{edelman2022inductive, trauger2024sequence}, Bayesian methods: \citet{lintilhac2026sharperpicturegeneralizationtransformers,rizvi2026narrow}, VC Dimension: \citet{yang2026tight}, \emph{inter alia}).
These results primarily address the \emph{i.i.d. regime} -- for transformers, this in particular implies that the test inputs are drawn from the same length distribution as the training inputs. %

However, for transformer learning theory, the i.i.d. paradigm is arguably insufficient: it is highly desirable to understand when transformers can extrapolate to sequences \emph{longer than seen during training}. This phenomenon is referred to as \emph{length-generalization} \citep[e.g.][]{bhattamishra2020ability,anil2022exploring,kazemnejad2023impact}, and is understood to be diagnostic of when a transformer actually learns an underlying algorithmic solution, as opposed to length-dependent heuristics \citep[e.g.][]{zhou2024algorithms, huang2026discovering}. It also indicates when LLMs can be expected to perform a task reliably across long contexts \citep{jobanputra2026born}.

Length generalization is governed by two quantities: the number of training samples, and their length. Sample count determines how reliably the training data pins down the target function, as in classical statistical learning theory; sample length determines how well the model extrapolates, since the longest input length at which the target function's behavior is observed dictates to which point a model will extrapolate beyond it.
For a given target function and learning procedure, we thus need to consider not only its \textbf{sample complexity} (the required number of training samples), but also its \textbf{length complexity} (the required length of training samples).

Here, we take up the question of providing quantitative bounds on generalization \emph{along both axes} for multi-layer transformers in the fixed-precision regime.
We study the questions:
\begin{center}
\emph{How do properties of a target function influence the length complexity and data complexity of the task?}

\emph{When both sample size and length are bounded, how do they trade off? In particular, under a fixed token budget, how to best allocate sample size and length?}
\end{center}
and address these by providing generalization bounds \textbf{depending on both sample size and input length}.
The main idea is that training length determines which behaviors become observable; sample size determines how reliably those behaviors are learned. Below a task-dependent length threshold, some competing behaviors cannot be distinguished. Around the relevant transition, longer examples can make informative events much more common; increasing length can thus strongly reduce sample complexity in this regime. Beyond that transition, the benefit of additional length diminishes and additional independent examples matter more.
\section{Background}
\subsection{Transformers}\label{def:main-FPT}

We study transformers in the \textbf{fixed-precision regime}, without positional encodings, as formalized by \citet{li_cotterell2025fpt,li2026characterizing} (see also \citet{li2024chain, london2025pause}).

This class is expressive enough to accommodate key algorithmic primitives relevant to transformers, such as retrieving past symbols, checking global symbol alternation patterns, or the ``induction head'' construction. %
Fixed-precision transformers are less expressive than log- or unbounded-precision transformers; this turns out to be necessary for obtaining quantitative results, because length generalization bounds for the latter are known to be uncomputable in the worst case by results of \cite{yang2026length}.
We discuss the detailed comparison across models of transformers in Appendix~\ref{sec:app:faq}.

Our model of transformers \edits{derives from} \cite{li_cotterell2025fpt}: we consider transformers with softmax attention, fixed precision, no positional encodings and strict causal masking. However, we generalize over \cite{li_cotterell2025fpt} in viewing transformers not as binary classifiers, but as providing an output at every position.
We provide details in Appendix~\ref{sec:fpt-definition}.
All activations are computed in $\F$,  a finite subset of $\mathbb{R}$ extended with $-\infty$ and $\infty$. The result of every operation is rounded to $\F$. 

\edits{Fix a} finite alphabet $\Sigma$. A \emph{fixed-precision transformer} $T$, also just called \emph{transformer} consists of token embeddings in dimension $D$, a sequence of $L$ layers each consisting of $H$ softmax attention heads with a ReLU feedforward network, and finally a linear output projection.
Overall, it takes a word  $w \in \Sigma^*$  input and returns vectors of dimension $\nu$ at every input position: $T(w) \in \F^{\nu \times |w|}$.  We suppose that \edits{$H$ always divides $D$}, and that the dimension per head $D_h$ is uniform among the heads. \edits{We omit layer norm, as it is not essential to our results and requires extra care in fixed precision; see Appendix~\ref{sec:app:normalization-constraints} for more details.} Assuming that $\Sigma, \F$ and the output dimension $\nu$ are \edits{shared}, we refer to these transformers as $\mathcal{FPT}(L, D, H)$.
See Appendix~\ref{sec:fpt-definition} for details.

\subsection{Learning framework}

We follow prior sample complexity bounds in studying an idealized learning model that abstracts away from the algorithmic process of learning.
Capacity-based learning bounds \citep{edelman2022inductive, trauger2024sequence, yang2026tight} apply to \textbf{Empirical Risk Minimization} (ERM), which finds parameters minimizing training error under some constraint on the hypothesis class.
We contrast it with the \textbf{Guess \& Check algorithm} \citep[G\&C, ][]{chiang2023loss, buzaglo_narrow_teachers_2024}, which samples parameter settings from the parameter space until a setting interpolating the training set is found.
Guess \& Check expands over ERM by modeling the inductive bias of SGD-like optimizers towards solutions occupying large volumes in parameter space \citep{chiang2023loss,buzaglo_narrow_teachers_2024}.
This is an important difference over bounds based on ERM  \citep{edelman2022inductive,yang2026tight, trauger2024sequence}, as it (i) allows us to explicitly take into account the effect of both the hypothesis class and the ground-truth function, (ii) while bounds depend more mildly on the hypothesis model's size than capacity-based bounds. 
This enables stronger results when the target function can be represented by small models, and we find it to be borne out empirically in experiments.
G\&C is closely related to a Bayesian interpretation of learning, as it samples from a posterior over the parameter space conditional on the training set \citep[e.g.][]{haussler1994bounds}.
We throughout work in the realizable regime, where the training data can be exactly fitted by an element of the hypothesis class. Here, ERM and G\&C can be defined as follows:
\begin{definition}

Let a hypothesis class $\mathcal{H} = \mathcal{FPT}(L_s, D_s, H_s)$ (where ``$\dots_s$'' stands for ``student''), and a training set $(X^{(n)}, Y^{(n)})_{n=1,\dots,N}$, be given.

    The \textbf{ERM Learner} selects a transformer $T_s \in \mathcal{H}$ such that $T_s$ perfectly fits the training set: $T_s(X^{(n)}) \equiv Y^{(n)}\ \forall (X^{(n)}, Y^{(n)})_{n=1,\dots,N}$.
    The \textbf{Guess \& Check Learner} returns a sample from the posterior
    \begin{equation}
        P\left(T_s\ |\ T_s(X^{(n)}) \equiv Y^{(n)}\ \forall (X^{(n)}, Y^{(n)})_{n=1,\dots,N}\right)
    \end{equation}
    where $P(T_s)$ is the uniform distribution over $\mathcal{H}$.
\end{definition}

Intuitively, G\&C adds an inductive bias towards hypotheses represented by a larger volume in parameter space.
See Appendix~\ref{sec:app:faq} for discussion beyond the interpolating regime.

\begin{remark}
\label{remark:learning-assumptions}
    Throughout this paper, we consider the student--teacher learning framework between two transformers $T_s$ and $T_t$, with $L_s \geq L_t, D_s \geq D_t$, $H_s \geq H_t$ \edits{and $D_{h,s} \geq D_{h,t}$} for realizability. %
\end{remark}

\section{Length Generalization}

We first derive length generalization bounds when unbounded data samples are available (Section~\ref{sec:len-gen-only}); and
statistical generalization bounds at a fixed length (Section~\ref{sec:statistical-gen-only}).
We then derive joint generalization bounds~(Section \ref{sec:joint-generalization}).

\subsection{Deterministic Length Generalization}\label{sec:len-gen-only}

Our first result states that, if two fixed-precision transformers exactly agree up to a certain length bound $C$, they must agree everywhere. We in fact obtain two distinct bounds with different rates:

\begin{highlightbox}
\begin{lemma}[Deterministic Length Generalization]
\label{thm:deterministic-length-generalization}
\label{thm:equal-short-then-long}
Let $r_\F$ be a well chosen constant and let $T_s \in \mathcal{FPT}(L_s,D_s,H_s)$ and $T_t \in \mathcal{FPT}(L_t,D_t,H_t)$ be fixed-precision transformers (\ref{def:main-FPT}). 
Set
\begin{equation}\label{eq:chain-bound-in-long-short-lemma}
  C_0 \coloneqq r_\F \left[L_sD_s |\F|^{D_{h,s}+1}
                 + L_tD_t |\F|^{D_{h,t}+1}\right]+1, \qquad C_1 \coloneqq \binom{r_\F L_s+\S}{\S}
\end{equation}
and let $C \coloneqq \min(C_0, C_1)$. Then, if $T_s$ and $T_t$ agree on every word in $\Sigma^{\leq C}$, they are identical on every word in $\Sigma^*$:
\begin{equation}\label{eq:from-short-to-long}
    T_s|_{\leq C} \equiv T_t|_{\leq C} \quad \implies \quad T_s \equiv T_t.
\end{equation}
\end{lemma}
\end{highlightbox}

This statement yields a first length generalization guarantee, in the setting where unbounded data samples are available: In this case, if a learned model $T_s$ exactly fits data representing a target function expressed by $T_t$ up to length $C$, then we know $T_s \equiv T_t$ even on longer length strings, i.e., length generalization occurs.

See Appendix \ref{app:length-generalization-proof} for the formal proof; here, we provide an informal discussion to build intuition.
Given a teacher transformer $T_t$ implementing the target function and a student $T_s$ from the hypothesis class, we introduce the \textit{difference function}, which checks for any given word if the output vectors of both transformers match at every position. 
Previously, \cite{li_cotterell2025fpt} showed that fixed-precision transformers with a binary classifying output layer are expressively equivalent to $\LTL$, a fragment of linear temporal logic, and further can be represented as a \textit{partially-ordered deterministic finite automaton} (poDFA) (Figure~\ref{fig:podfas}). A poDFA is a DFA where the reachability relation defines a partial order on the state set; intuitively, such automata have no nontrivial loops. Intuitively, the state space describes the dynamic of the transformer's KV-cache; as soon as some activation is added to the KV-cache, it cannot be removed anymore, giving rise to the partially ordered loop-free structure.

Building on this, we express the difference function \edits{between the student and the teacher} as a logic formula in $\LTL$, \edits{from which we obtain a poDFA $\mathcal{A}_\Delta$}. The student and the teacher transformers produce a different output somewhere on a word $w_1 \dots, w_n$ if and only if $\mathcal{A}_\Delta$, run on the word, ends up in the unique \edits{absorbing disagreement state} $\bot$ \edits{(Figure~\ref{fig:difference-podfa})}. As a supplement, we mention in Appendix~\ref{sec:logic}, \ref{app:length-generalization-proof} how transformers exist as ordered Mealy machines (Def. \ref{def:mealy-machine}), providing a similar underlying ordered structure on individual transformers, even outside of the learning framework.

We then upper-bound the \emph{height} of $\mathcal{A}_\Delta$ -- i.e., the length of its longest transition chain. The key idea here is that observing behaviour on strings up to roughly this length should be important for length generalization, because strings of this length should be minimally required to be able to observe the whole state space of $\mathcal{A}_\Delta$ -- and thus the whole range of behaviour expressed by the student and teacher models.
We obtain $C_0, C_1$ as two upper bounds on the height of $\mathcal{A}_\Delta$ (eq. \ref{eq:chain-bound-in-long-short-lemma}); whichever bound is smaller provides a tighter bound in a given setting, though neither dominates the other globally. Notably, we show some kind of tightness in $L$ and $\F$ of our bounds in Appendix \ref{app:ExamplesTightness}, where we prove that any upper bound on the height must depends at least linearly on $|\F|$ and $L$.

Now, if $T_s$ and $T_t$ agree up to length $C = \min\{C_0, C_1\}$, $\mathcal{A}_\Delta$ cannot visit a $\bot$ state on any word of this length. But given that $C$ upper-bounds the height of the automaton, the $\bot$ state cannot be reachable even by longer inputs, and $T_s \equiv T_t$ across $\Sigma^*$.

The above entails a first length-generalization bound in the limiting case where all data up to some length is available; it is an illustrative first special case of our general result. 

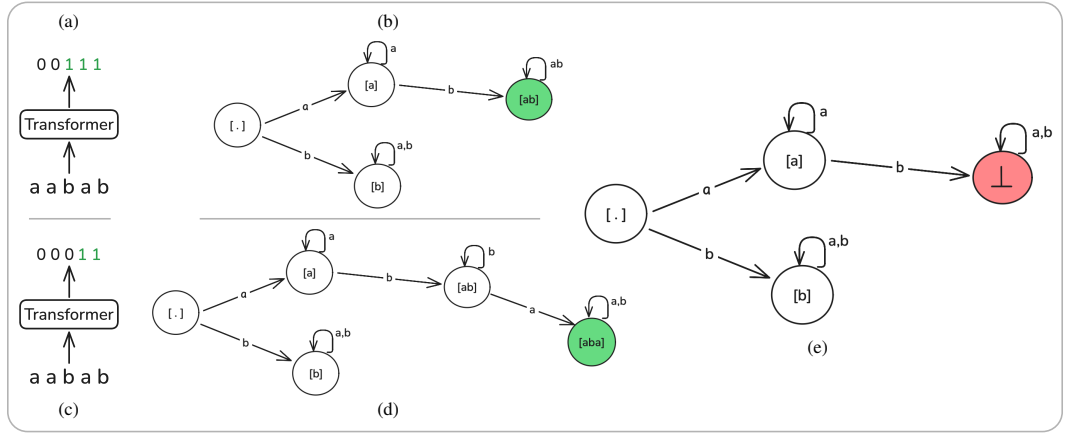
\begin{figure}[htbp]
  \centering
  \captionsetup{skip=5pt}
  \captionsetup[subfigure]{position=top,skip=0pt}
  \begin{subfigbox}
    \begin{tikzpicture}[x=\linewidth,y=-\linewidth,
      every node/.style={inner sep=0pt,outer sep=0pt}]
      \path[use as bounding box] (0,0) rectangle (1,0.412);
      \node[anchor=north west,blend mode=multiply] at (0.0125,0.049)
        {\includegraphics[width=0.095\linewidth]{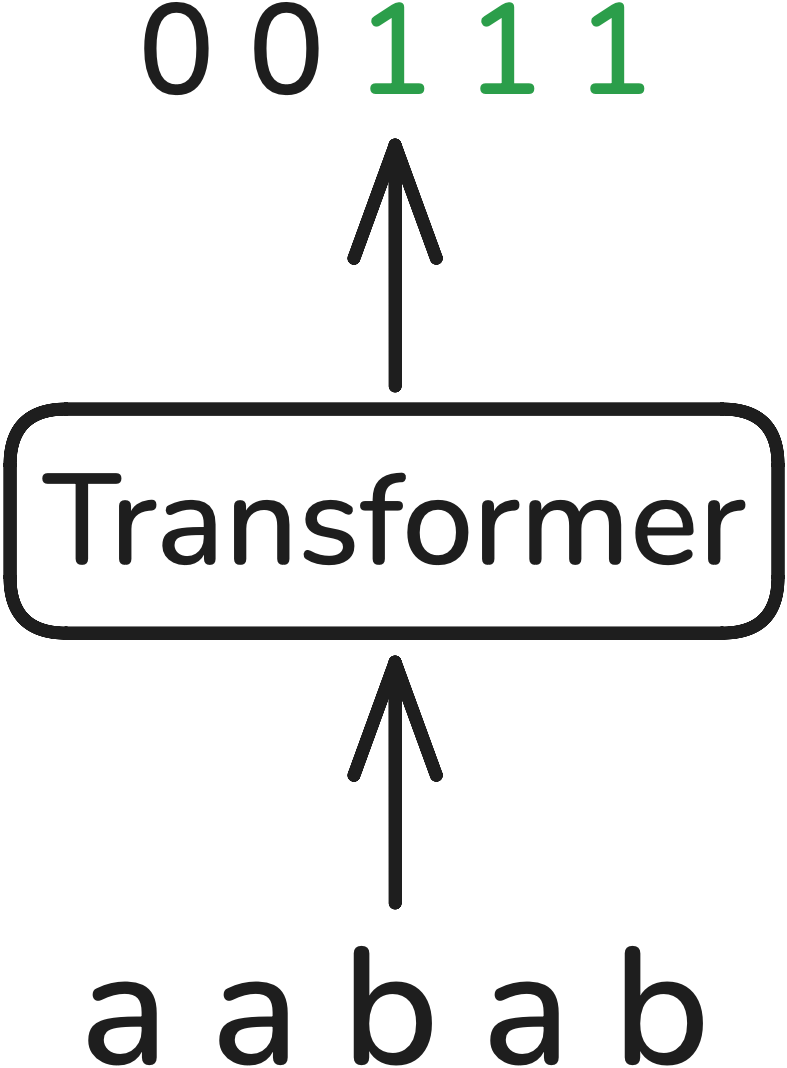}};
      \node[anchor=north west,blend mode=multiply] at (0.1964,0.030)
        {\includegraphics[width=0.3372\linewidth]{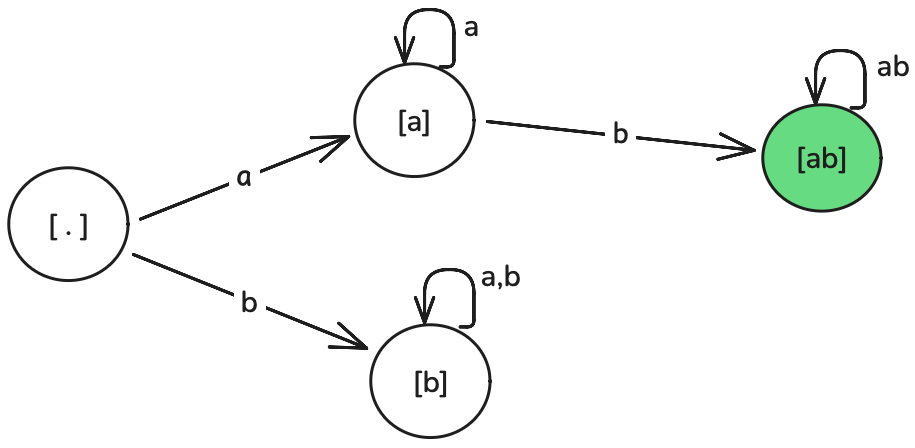}};
      \node[anchor=north west,blend mode=multiply] at (0.0125,0.230)
        {\includegraphics[width=0.095\linewidth]{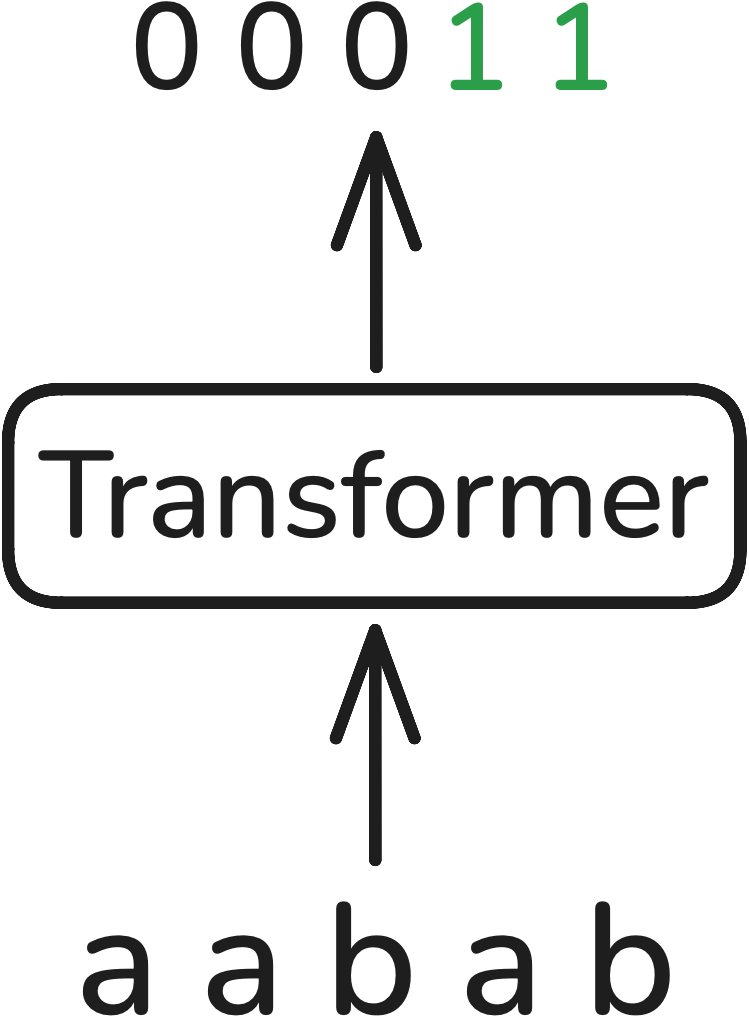}};
      \node[anchor=north west,blend mode=multiply] at (0.135,0.209)
        {\includegraphics[width=0.460\linewidth]{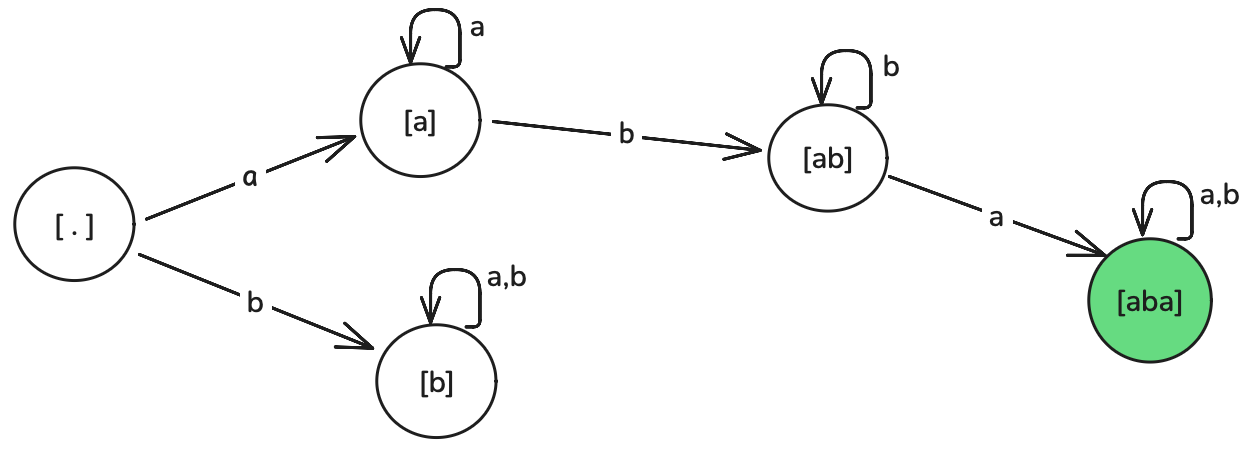}};
      \node[anchor=north west,blend mode=multiply] at (0.550,0.088)
        {\includegraphics[width=0.450\linewidth]{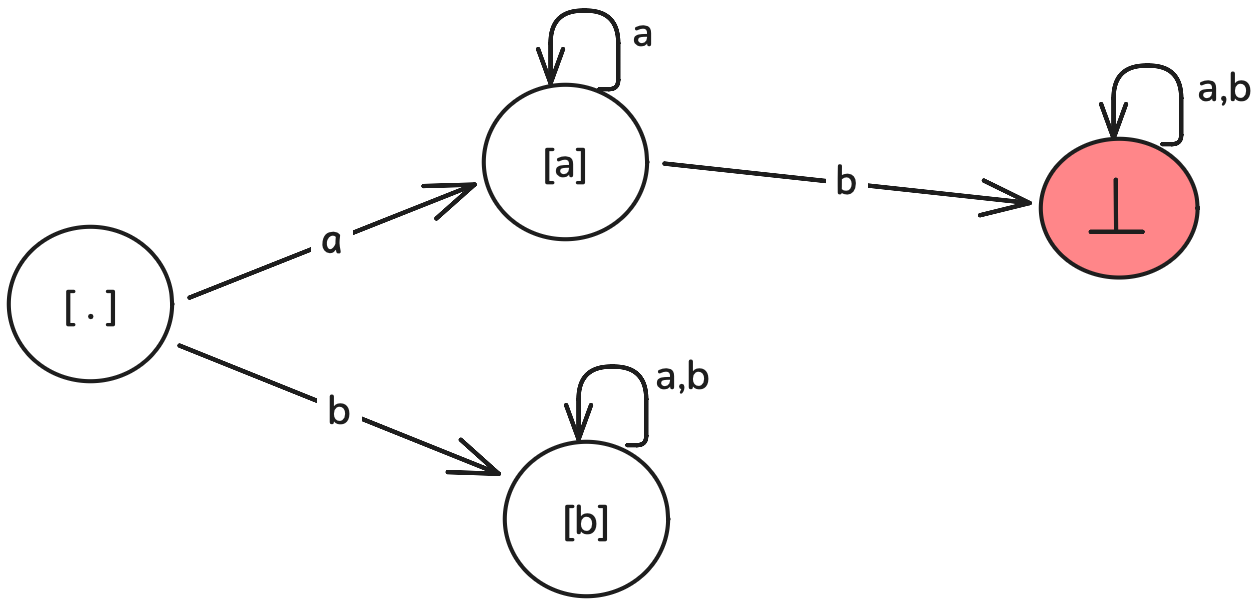}};
      \draw[gray!60,line width=0.6pt] (0.022,0.202) -- (0.1,0.202);
      \draw[gray!60,line width=0.6pt] (0.185,0.202) -- (0.510,0.202);
      \node[anchor=north] at (0.060,0.005) {\begin{subfigure}{24pt}
        \caption{}\label{fig:transformerL2}\end{subfigure}};
      \node[anchor=north] at (0.365,0.005) {\begin{subfigure}{24pt}
        \caption{}\label{fig:podfaL2}\end{subfigure}};
      \node[anchor=north] at (0.060,0.375) {\begin{subfigure}{24pt}
        \caption{}\label{fig:transformerL3}\end{subfigure}};
      \node[anchor=north] at (0.365,0.375) {\begin{subfigure}{24pt}
        \caption{}\label{fig:podfaL3}\end{subfigure}};
      \node[anchor=north] at (0.775,0.316) {\begin{subfigure}{24pt}
        \caption{}\label{fig:difference-podfa}\end{subfigure}};
    \end{tikzpicture}
  \end{subfigbox}
     \caption{\edits{Examples of poDFAs; the vertices are the states and the arrows are the transition functions $\delta$. Figure (a) represents the transformer that outputs ``1'' as soon as the input word has entered the language $L_2 \coloneqq a\Sigma^*b\Sigma^*$ with (b) being the associated poDFA. Similarly, (c) depicts the transformer that recognizes the language $L_3 \coloneqq a\Sigma^*b\Sigma^*a\Sigma^*$ and (d) is its corresponding poDFA. Figure (e) is the poDFA $\mathcal{A}_\Delta$ of the difference function between these two transformers. Any disagreement sends the run to the absorbing final state $\bot$.}}
  \label{fig:podfas}
  \end{figure}

\subsection{Extension to the probabilistic framework}
We next ask for a similar length generalization guarantee when $T_s$ and $T_t$ are equal only on \emph{most} strings up to length $K$. 
Setting the data distribution $\mathcal{D}_{K} := \operatorname{Unif}(\Sigma^{\leq K})$ and using the bound $C$ from Lemma~\ref{thm:deterministic-length-generalization}, we indeed find a probabilistic extension of the guarantee:

\begin{highlightbox}
\begin{lemma}[Probabilistic Length Generalization]
\label{lem:absorption-tail}
Let \edits{$0 < \eps < 1$, $\S \geq 2$} and consider the training words of length up to $K$, with
\begin{equation}\label{eq:absorption-tail}
  K\geq \S\left(
    C+\log\frac1\eps+
    \sqrt{\log^2\frac1\eps+2C\log\frac1\eps}
      \right) \sim \mathcal{O}\big(\S(C+\log\frac{1}{\eps})\big).
\end{equation} 
Then
\[
\operatorname{Pr}_{\mathcal{D}_{K}} \left[ T_s(w) \neq T_t(w) \right] \leq \eps \quad \implies \quad \forall M\geq K,\ \operatorname{Pr}_{\mathcal{D}_{M}}\left[T_s(w) \neq T_t(w)\right] \edits{< 3\eps}.
\]
\end{lemma}
\end{highlightbox}

See Appendix \ref{sec:probabilistic-LG} for the formal proof. Intuitively, this holds for two reasons: (i) the poDFA $\mathcal{A}_\Delta$ ultimately reaches a sink \textit{absorbing} state; (ii) the input distribution $\mathcal{D}_{K}$ is uniform and so are the letters: this implies that for any non-absorbing state, the probability to leave the state by reading the next letter is at least of $1/\S$.
Therefore, after a big enough number of steps $K$, $\mathcal{A}_\Delta$ will be likely to live in an absorbing state. In particular, since the probability of a transition is at least $1/\S$, the overall probability of a word to reach an absorbing state after $K$ steps dominates the binomial law $Y\sim\operatorname{Bin}(K,1/\S)$. The lemma follows from the application of a Chernoff bound. %

\section{Joining Length and Statistical Generalization}
\subsection{Data complexity}\label{sec:statistical-gen-only}

We lift the deterministic statement to a statistical generalization statement. The key here is to combine Lemma~\ref{thm:deterministic-length-generalization} with statistical in-distribution generalization on the strings up to length $K$:

\begin{highlightbox}
\begin{proposition}[Statistical Generalization for ERM and G\&C]\label{thm:data-complexity}
Fix $\eps \in (0, 1)$ and $\delta \in (0, \tfrac{1}{5})$. 
Assume the target function is represented by a teacher $T_t \in \mathcal{FPT}(L_t, D_t, H_t)$, and the learner (G\&C or ERM) selects a student $T_s$ from the hypothesis class $\mathcal{FPT}(L_s, D_s, H_s)$. %
Let $S$ be an i.i.d. sample of $N$ strings from $\mathcal{D}_K$.
\begin{enumerate}
\item Let $T_s$ be obtained by the G\&C learner on $S$. If
\[
N \geq \frac{1}{\eps}\big(B \log|\F| + 3 \log \frac{2}{\delta}\big),
\qquad
B \sim \mathcal{O}\big( L_t D_t^2 + D_s D_t \edits{(L_s + L_t)} + |\Sigma|D_s\big),
\]
then, with probability at least $1-\delta$ over $S$ and the draw of $T_s$, 
\[
\Pr_{w \sim \mathcal{D}_K}\big[T_s(w) \neq T_t(w)\big] < \eps .
\]
\item Let $T_s$ be obtained by ERM on $S$. If
\[
N \geq \frac{1}{\eps}\big(B_{ERM} \log|\F| + \log(1/\delta)\big),
\qquad
B_{ERM} \sim \mathcal{O}\big( L_s D_s^2 \edits{+  |\Sigma|D_s}\big),
\]
then, with probability at least $1-\delta$ over $S$,
\[
\Pr_{w \sim \mathcal{D}_K}\big[T_s(w) \neq T_t(w)\big] < \eps .
\]
\end{enumerate}
\end{proposition}
\end{highlightbox}

Note here that $\mathcal{O}(\cdot)$ absorbs constants depending on $\F$ and $\nu$. The ERM bound has a similar flavor to bounds using capacity-based arguments~\citep{edelman2022inductive, trauger2024sequence, yang2026tight} in that it scales with parameters count of the hypothesis class. The G\&C bound has a similar dependency but scales with \edits{$D_sD_t$ rather than $D_s^2$}. This provides improved generalization \edits{when a large student network learns a} low-complexity target. %

This result extends the work of \citet{rizvi2026narrow} with some improvements; we \edits{give the exact values of $B$ and $B_{ERM}$ and} discuss the proof and extensions in Appendix~\ref{sec:app:data-complexity}. 
For ERM, this is a straightforward application of Appendix B.1 in \cite{buzaglo_narrow_teachers_2024}.
For G\&C, the proof strategy (Appendix~\ref{sec:app:data-complexity}) is to show that low-complexity teachers can be embedded in many equivalent ways into the parameter space of a higher-complexity hypothesis class; this means that G\&C is intrinsically more likely to find low-complexity interpolators. This insight (\cite{buzaglo_narrow_teachers_2024}) drives the appearance of a teacher complexity term, and a milder dependence on student complexity than based on hypothesis class parameter count alone.

\subsection{Main Result: Joint Statistical and Length Generalization}\label{sec:joint-generalization}

We now combine length and statistical generalization to obtain our \textbf{main result}.
The basic intuition is that generalization occurs when (i) the length is sufficiently long so that all representative behavior relevant at long lengths are already visible, (ii) the sample size is sufficient for achieving low test error at this length; combining (i) and (ii) then entails low test error at higher lengths. The result is:

\begin{highlightbox2}
\begin{theorem}[Statistical and Length Generalization]\label{thm:main-theorem}
Let the transformers $T_s \in \mathcal{FPT}(L_s, D_s, H_s)$ student and $T_t \in \mathcal{FPT}(L_t, D_t, H_t)$ teacher satisfying \ref{remark:learning-assumptions}, $C$ as in Eq.~\ref{eq:chain-bound-in-long-short-lemma}, $\eps \in (0, 1)$ and $\delta \in (0, \tfrac{1}{5})$.

    \edits{Let} $T_s$ be given by the Guess \& Check \edits{learner}, agreeing with  $T_t$ on $N$ i.i.d. samples from $\mathcal{D}_K$, with
    \begin{equation}
        N \geq \frac{1}{\eps}\big(B \log|\F| + 3 \log \frac{2}{\delta}\big)
        \sim \mathcal{O} \left ( \frac{1}{\eps}[(L_t D_t^2 + D_s D_t \edits{(L_s + L_t) + \S D_s}) + \edits{\log \frac{1}{\delta}}] \right)
    \end{equation}    
    and
    \begin{equation}
  K\geq \S\left(
    C+\log\frac1\eps+
    \sqrt{\log^2\frac1\eps+2C\log\frac1\eps}
      \right) \sim \mathcal{O}\big(\S(C+\log\frac{1}{\eps})\big),
      \end{equation} 
    \edits{where} 
    \begin{equation}
          \edits{C \coloneqq \min(C_0, C_1),\ \quad C_0 } \sim \mathcal O \left(L_sD_s \exp(D_{h,s})
                 + L_tD_t \exp(D_{h,t}) \right),\ \quad \edits{C_1 \sim  \mathcal O \left( 2^{\S+L_s}\right).}
    \end{equation}
    \edits{Then} with probability $1 - \delta$, we have  
    \begin{equation}
        \forall M \geq K, \qquad \operatorname{Pr}_{\mathcal{D}_{M}}\left[T_s(w) \neq T_t(w)\right] \edits{< 3} \eps.
    \end{equation}
\end{theorem}
\end{highlightbox2}

See Appendix~\ref{sec:app:full-pipeline} for the formal proof.
For ERM, an analogous result \edits{ensues} from bounding $N$ appropriately, following Proposition  \ref{thm:data-complexity}. In addition, the theorem can be adapted to alternative distributions $\widetilde{D}_K$ \edits{if they sample} enough long words; we give an example in Appendix~\ref{app:main-thm-uniform-training-sampling}.

\subsection{Tradeoff between Length and Data Complexity}
\label{sec:tradeoff}

The previous theorem provides conditions on the data complexity $N$ and the training length threshold $K$, such that low loss can be achieved. Assuming the same learning framework for student $T_s$, teacher $T_t$ \edits{and writing $s\coloneqq \S$}, we now investigate how these quantities trade off, by giving an error bound for any joint choice of $N$ and $K$:

\begin{highlightbox2}
\begin{theorem}[Joint Generalization Bound in Sample Size and Length]\label{thm:joint-bound-michaels-attempt-main-paper}
Define 
\begin{equation}
        a_C(K) \coloneqq \operatorname{Pr}\left(Bin(K,\frac{1}{s}) \geq C\right), 
        \quad \rho \coloneqq \frac{s}{s-1} \quad \text{and} \quad B_\delta \coloneqq B\log |\F| + 3 \log \frac{2}{\delta},
    \end{equation}
where $C$ is the bound given in Lemma~\ref{thm:deterministic-length-generalization} and $B$ is the sample  complexity term from Proposition~\ref{thm:data-complexity}. 

Now \editsMiH{assume} that the student is trained on N random samples from $\mathcal{D}_K$ and that $K$ is big enough so that $a_C(K) \neq 0$. Then with probability $1-\delta$, for any length $M \geq K$, the error rate between the student and the teacher is bounded as

    \begin{equation}\label{eq:multiplicative-bound}
 \epsilon \coloneqq \operatorname{Pr}_{\mathcal{D}_{M}}\left[T_s(w) \neq T_t(w)\right]\leq \frac{\rho B_\delta}{N \cdot a_C(K)}.
\end{equation}
\end{theorem}
\end{highlightbox2}

See Appendix~\ref{sec:multiplicative-upper-bound} for the proof.
This value $a_C(K)$ lower-bounds the probability of reaching an absorbing state by time $K$, and $a_C(K) = 0$ when $K < C$. This theorem holds for either of the Guess \& Check or the ERM learner, setting $B$ accordingly  following Proposition~\ref{thm:data-complexity}. 
Moreover, under an appropriate definition of $a_C(K)$, the theorem even remains valid if the input strings are generated from appropriate nonuniform distributions (Appendix~\ref{sec:nonuniform-distribution}).
We further complement the upper bound (\ref{eq:multiplicative-bound})  with a lower bound: there are families of target functions and input distributions such that the expected test error on long inputs is lower-bounded by $\frac{B}{8 N a_{C}(K)}$ for an appropriate choice of $B$ (Appendix, Theorem~\ref{thm:lower-bound-on-tradeoff}).

\paragraph{Takeaway: Importance of Long Training Samples for Transformers.}
This implies the following:
First, when $K< C$, there is no \edits{assurance of} length generalization.
Second, when $K$ just slightly exceeds $C$, length generalization can be guaranteed, but only at an extraordinary sample size on the order of \edits{$ \gtrapprox s^C$, with $s$ the size of the alphabet.}
Namely, $a_C(C) = s^{-C}$, leading to a sample size requirement $N_{req}(C) = N_0 s^C$ where $N_0 = \frac{\rho B_\delta}{\epsilon}$. Intuitively, this happens because only a small fraction of the samples will already have reached an absorbing state; most samples will thus not be informative about the target function's behavior on long strings.

Third, increasing $K$ by a constant can reduce the required sample size by a multiplicative factor.
Specifically, increasing length from $C$ to $sC$ reduces the required sample
size by a factor of at least $s^C/2$; \edits{a linear increase of $K$  permits an exponential saving on $N$}. %

Given a fixed training-token budget $Budget=NK$, the above considerations allow us to determine an optimal allocation to $N$ and $K$; we show in Appendix~\ref{app:explicit-choice-tradeoff} that, for $C \geq 2$, the following closed form achieves an error bound within a factor $1+\mathcal O(\sqrt{\log C/C})$ of the optimum:
\[
  \widehat K
  \coloneqq \left\lceil |\Sigma|\left(
    C+\log C+\sqrt{\log^2 C+2C\log C}
  \right)\right\rceil,
\]
\edits{provided $Budget \geq \widehat K$}. Intuitively, this means that close-to-optimal budget allocation can be achieved at around $K \sim |\Sigma| C$, exponentially more efficient \edits{than} at $K \sim C$.

\begin{wrapfigure}[19]{r}{0.50\textwidth}
\vspace{-14pt}
\centering
\begin{tikzpicture}
\begin{axis}[
    width=\linewidth,
    height=6cm,
    axis lines=left,
    xlabel={$K$},
    ylabel={$N$},
    xmin=0, xmax=11,
    ymin=0, ymax=42,
    xtick={1.7,6.3},
    xticklabels={$C$,$sC$},
    ytick={1,10,20,40},
    yticklabels={$N_0$,
                 $s^C N_0/4$,$s^C N_0/2$,$s^C N_0$},
    tick label style={font=\scriptsize},
    label style={font=\small},
    ylabel style={xshift=-15pt},
    xlabel style={yshift=3pt},
    legend style={draw=none,fill=none,font=\scriptsize,
                  at={(0.97,0.97)},anchor=north east},
    samples=200,
    domain=1.7:10.7,
    clip=false,
]
    \draw[densely dashed,black!40]
        (axis cs:1.7,0) -- (axis cs:1.7,41.5);
    \draw[densely dashed,black!40]
        (axis cs:6.3,0) -- (axis cs:6.3,2.1);

    \addplot[very thick,blue!75!black] {1+39*exp(-0.90*(x-1.7))};
    \addlegendentry{$\eps_\infty$}
    \addplot[very thick,orange!90!black] {0.5+19.5*exp(-0.90*(x-1.7))};
    \addlegendentry{$2\eps_\infty$}
    \addplot[very thick,green!50!black] {0.25+9.75*exp(-0.90*(x-1.7))};
    \addlegendentry{$4\eps_\infty$}

    \addplot[densely dotted,thick,blue!75!black,forget plot]
        coordinates {(0,1) (10.7,1)};
    \addplot[densely dotted,thick,orange!90!black,forget plot]
        coordinates {(0,0.5) (10.7,0.5)};
    \addplot[densely dotted,thick,green!50!black,forget plot]
        coordinates {(0,0.25) (10.7,0.25)};
    \addplot[only marks,mark=*,mark size=1.8pt,blue!75!black,forget plot]
        coordinates {(1.7,40)};
    \addplot[only marks,mark=*,mark size=1.8pt,orange!90!black,forget plot]
        coordinates {(1.7,20)};
    \addplot[only marks,mark=*,mark size=1.8pt,green!50!black,forget plot]
        coordinates {(1.7,10)};
\end{axis}
\end{tikzpicture}
\caption{Schematic error-budget limit lines with linear axes.
At $K=C$, every budget can be respected with $N \sim \mathcal{O}(s^C)$ training samples.
By length $K=sC$, at least half of the words have reached their terminal absorbing states, so each curve is within a factor $2$ of its
floor; further length gains saturate.}
\label{fig:multiplicative-tradeoff}
\end{wrapfigure}

Our result reveals that there are several regimes of generalization.
When $K$ is too small, there is in general no hope of length generalization no matter how large $N$ is.
When $K$ has just crossed the threshold given by the longest chain, one may need a very large $N$, because only few strings will have reached the final state.
In this regime, increasing $K$ further is much more efficient than increasing $N$: increasing $K$ by an additive constant can cut down $N$ by a multiplicative factor.

\subsection{Extension with Local Attention}
Local attention has emerged as an effective and popular addition to transformers \citep[e.g.][]{beltagy2020longformer}.
It is also relevant because positional information often specifically introduces a locality bias favoring attention to recent position \citep{press2021train,su2024roformer}.
Local attention thus simultaneously serves as a simplified account of the role of positional information (similar to positional functions in Limit Transformers, \cite{huang2025formal, izzo2025quantitative}).
In each layer, we add a certain number of local attention heads at some window size $W$; for these heads, context positions at distance $>W$ are masked.
Both $W$ and the number of local heads are hyperparameters.
We refer to \citet{li2026characterizing} for the  full formalization in the $\mathcal{FPT}$ context.

Formally, we can model an $\mathcal{FPT}$ with local attention in terms of a more complex automaton that arises as a \emph{cascade product} of a poDFA with a simple automaton that maintains the last $L\cdot W$ symbols.
That is, we can model the difference automaton $\mathcal{A}$ in terms of the cascade product of a poDFA representing the output difference between the two transformers with a simple automaton that maintains the last $L\cdot W$ symbols. More details in Appendix~\ref{appendix:local-attention}.

The statistical generalization guarantee (Proposition~\ref{thm:data-complexity}) is unaffected, because the number of parameters does not change when making some heads local.
The length generalization guarantee changes. First, the construction we provide is only valid for the bound $C_0$ in Eq.~\ref{eq:chain-bound-in-long-short-lemma}. Second this threshold $C_0$ receives a factor of \edits{$R \coloneqq WL + 1$}.
Third, the bounding binomial probability to go forward in the underlying poDFA \edits{is also reduced from $\S^{-1}$ to $\S^{-R}$}. Thus, as long as $W$ is fixed, we have the same qualitative predictions as above: model size impacts data and length complexity; and sample size and length trade off.

\section{Experiments}
\label{sec:experiments}

We introduce three families of tasks, and study (i) the impact of teacher and student complexity on length and data complexity, and (ii) the predicted interaction between length and sample size. 

\paragraph{Tasks.}
We defined three task families, each along a complexity scale indexed by $k \in \mathbb{N}$, as follows (i) \textbf{Alternating languages $L_k$:} Detecting whether $\geq k$ distinct blocks of $a$'s and $b$'s have occurred \citep{yang2025kneedeep} -- that is, checking membership in the languages $L_2 = a^+ b^+\Sigma^*; L_3 = a^+ b^+ a^+ \Sigma^*$, etc.
There is a teacher family with number of layers $L_k = \Theta(k)$, embedding dimension $D_k = O(1)$, head count $H_k = O(1)$ and fixed alphabet size $\S_k = O(1)$. Specifically, the letter sampling is not uniform, with probability to alternate block set to $0.05$. This still fits our theoretical framework, as explained in Appendix~\ref{app:exp-alternating-lang-specifics}.
(ii) \textbf{Alphabet languages $A_k$:} Detecting whether all symbols from a given alphabet $\Sigma_k = [k]$ have occurred.
There is a teacher family with $L_k = O(1)$, $D_k = O(k)$, $H_k = O(k)$, $\S_k = O(k)$. 
(iii) \textbf{Induction Head tasks $\mathit{IH}_k$:}  Outputting which next tokens can follow  the current token under the context's bigram distribution over the alphabet $\Sigma_k = [k]$ \citep[cf.][]{edelman2024evolution}; that is, given the current token $\alpha$ output a vector $\in \F^{|\Sigma|}$ indicating for which symbols $\beta$ the bigram $\alpha\beta$ has contextually occurred.
There is a teacher family with $L_k = O(1)$, $D_k = O(k)$, $H_k = O(1)$, $\S_k = O(k)$, with local attention.
See Appendix \ref{app:experiments} for teacher constructions.
Taken together, these tasks allow us to test the theoretical prediction that the teacher size (impacted by $L$, $H$, $D$, $|\Sigma|$) impacts both length and data complexity.

\paragraph{Setup.}
To test our predictions against ordinary gradient-based training, we train
FP32 transformers with Adam \cite{kingma2014adam} and binary
cross-entropy over $10$ seeds. We use binary classification because exact logit-equality in FP32 is not realistic. The difference function between the student and the teacher then checks for logit signs rather than exact values (Appendix \ref{app:bce-and-accuracy}). 
In addition, we retain only runs with perfect training accuracy---that is, runs that correctly classify every supervised target on
every training word. This enforces the \edits{interpolation conditioning} of Guess \& Check. Trained architectural sizes are fixed and common across a family, unless explicitly varied.
When sweeping $N$, we adjust $K$ across tasks to match the fraction of
words exhibiting the required blocks, letters, or bigrams, making
examples approximately comparable in informativeness (Appendix~\ref{app:experiments-adjusted-length}). Sweeps of $K$ instead compare tasks at common lengths.

\begin{figure}[b]
    \centering
      \includegraphics[width=\linewidth]{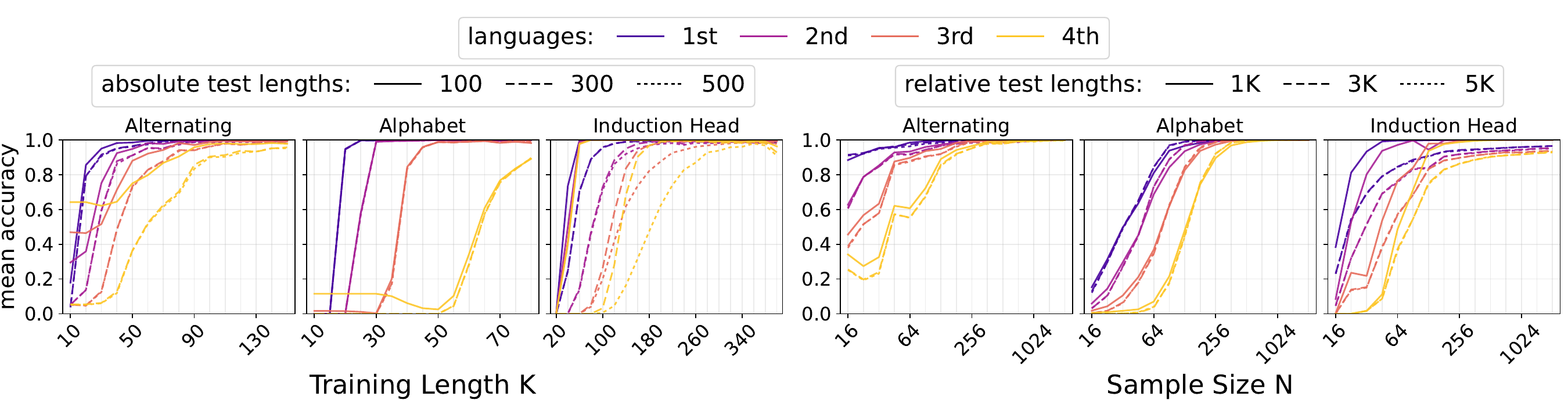}
  \vspace{-0.4cm}
  \caption{\textbf{Left:} Impact of teacher's size on length complexity. Larger teachers require training samples to be longer. The abscissa scale is linear, as suggested by theory.
  \textbf{Right:} Impact of teacher's size on data complexity. Larger teachers require more training samples for similar length generalization performance. The abscissa scale is \edits{logarithmic}, as suggested by theory. On both sides, the plotted tasks are respectively the  Alternating $L_4, L_5, L_6, L_7$, the Alphabet $A_{10}, A_{12}, A_{15}, A_{20}$ and the Induction Head $\mathit{IH}_5, \mathit{IH}_6, \mathit{IH}_7, \mathit{IH}_8$.}
    \label{fig:teacher-complexity-main}
\end{figure}

\paragraph{Effect of Transformers Size on Data and Length Complexity.}
Across the three families, increasing teacher size impacts the sample size $N$ and training length $K$ at which generalization becomes successful
(Figure~\ref{fig:teacher-complexity-main}). %
This is in line with our theoretical predictions, in particular, the G\&C-based bounds.
Interestingly, the student size had relatively small impact on sample size and length  (Appendix~\ref{app:experiments-impact-student}).
Thus, G\&C-based bounds are successful in capturing \edits{the fact teacher size is an important factor beyond student size}, but may \editsMiH{still be overly pessimistic about the impact of student size}. %

\paragraph{Tradeoff between Sample Size and Length.}
Longer training words can substantially reduce the sample size needed
for generalization (Figure~\ref{fig:aggregate_KxN_5K}). In particular, we find the non-linear interaction between the two quantities, qualitatively in line with the predictions from Figure \ref{fig:multiplicative-tradeoff}. This backs the idea that an efficient token-budget allocation should ensure that training samples are long enough. See Appendix~\ref{app:experiments-tradeoff} for further results showing the same pattern. 
\begin{figure}[h]
    \centering
  \includegraphics[width=0.95\linewidth]{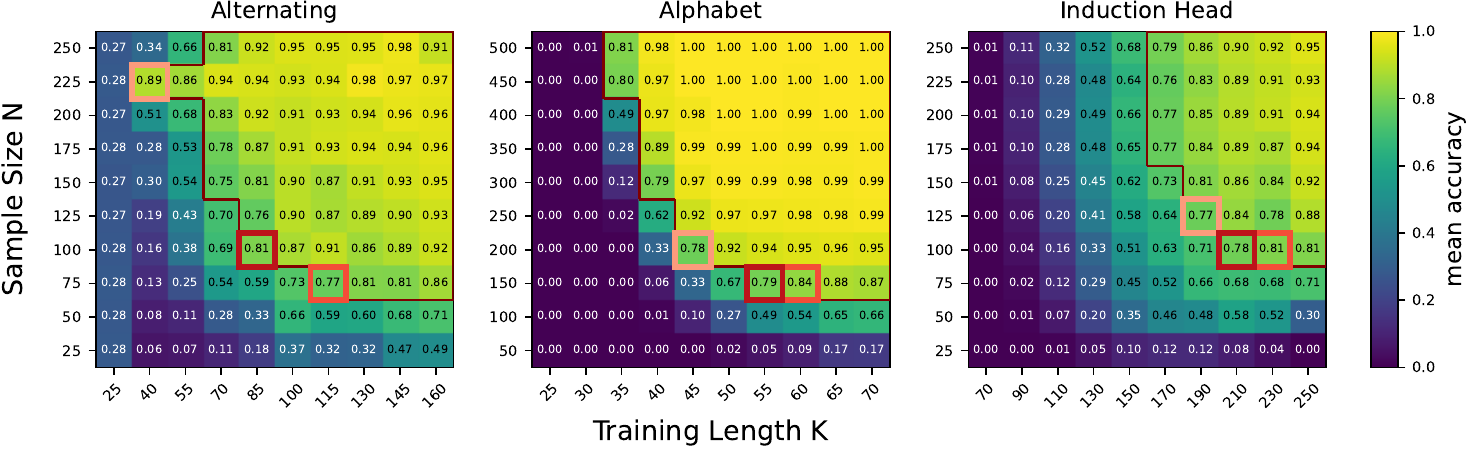}
  \vspace{-0.4cm}
  \caption{Tradeoff between the training \edits{length} and the sample size for one task of each language class: Alternating $L_6$, Alphabet $A_{15}$ and Induction Head $\mathit{IH}_7$. Evaluation at length 5 times the training length. \edits{Brown} contouring around \edits{all} cells with accuracy $\geq 0.75$\edits{. Under} this condition, the $3$ cells with the smallest token-budget are framed.
  }\label{fig:aggregate_KxN_5K}
\end{figure}

\section{Discussion}

Our result goes beyond established theory in two ways:
First, it goes beyond \textbf{statistical generalization} bounds for transformers by guaranteeing length generalization. 
Prior learning results for transformers characterize generalization error and sample size when the test and train samples are at the same sequence length \citep{edelman2022inductive,trauger2024sequence,li2026sharpergeneralizationboundstransformer,truong2024rank,yang2026tight,pli25nextTokenPrediction,rizvi2026narrow,lintilhac2026sharperpicturegeneralizationtransformers,mousavihosseini2025transformersoutperformfeedforwardrecurrent,hanneke2026samplecomplexityautoregressivereasoning,joshi2025theorylearningautoregressivechain}.
Our results instead consider the more general case of length generalization, in the setting of fixed-precision transformers.
We note that our main result ingests the statistical generalization bound essentially as a black box; hence, our results might analogously also be formalized on the basis of other statistical generalization bounds.

Second, our results go beyond prior work on \textbf{length generalization} by giving explicit and joint rates for sample size and length. 
So far, the most generally predictive account of length generalization has been the RASP-L Conjecture \citep{zhou2024algorithms} and its formalization via C-RASP and Limit Transformers \citep{huang2025formal}.
This theory has been highly predictive of which tasks empirically show length generalization \citep{yang2026algebraicdecompositiontheorytransformer}, but provides no quantitative predictions.
Notably, C-RASP is more powerful than $\mathcal{FPT}$; e.g. it can model the MAJORITY language.
However, C-RASP is a too loose target for quantitive learning bounds: \cite{yang2026length} show that, beyond the shallow-depth case \citep{izzo2025quantitative,chen2025nonasymptotic}, any length threshold would need to grow uncomputably with model complexity.
The positive fragment C-RASP$^+$ has expressivity equivalent to fixed-precision transformers and was studied in \citet{yang2026length, zetzsche2026lengthgeneralizationtransformerscompression}. We compare with these results in detail in Appendix~\ref{sec:app:faq}.
Importantly, our work innovates over these prior results by providing joint generalization guarantees in both sample size and length.

Like capacity-based generalization bounds \citep{edelman2022inductive, trauger2024sequence, yang2026tight}, our work treats an idealized learning model that abstracts away from the computational process of finding a model fitting the training data.
Understanding optimization dynamics for transformers has been difficult.
Progress has been made for specific task families and restricted models, such as single layers \citep[e.g.][]{yang2025multihead,mwigo2026generalization,huang2025transformers,wang24sparsetokenselection} (including one-layer models performing chain-of-thought, \citet{huang2025transformers}) or specific task families \citep[e.g.][]{nichani24causal}.
We view the study of G\&C as an (imperfect) proxy for part of the implicit inductive bias of SGD-like learning; understanding of optimization dynamics for a general multi-layer model and a general hypothesis class will require new methods.

The main limitation of our results is that they predict asymptotic rates but not yet concrete realistic numerical bounds for specific tasks and models. The primary barrier here is that the bounds involve large constants involving the fixed-precision set $\F$, which is large at full-precision training. It is an interesting question whether realistic numerical bounds can be obtained when training with strong quantization.
Further limitations are that fixed-precision transformers are less expressive than log-precision or infinite-precision transformers, we did not treat positional encodings, and studied G\&C and ERM as opposed to optimization dynamics.

\section{Conclusion}
We have provided generalization bounds for multi-layer fixed-precision transformers that account for the roles of both the number and the length of training samples.
We provided a joint treatment of both dimensions, showing that, under a given token budget, increasing length can sometimes be exponentially more efficient than increasing sample size.

\section*{Acknowledgments}
Funded in part by the Deutsche Forschungsgemeinschaft (DFG, German Research Foundation) – Project number  560456343.
M.H. is supported by the Konrad Zuse School of Excellence in Learning and Intelligent Systems \href{https://eliza.school/}{(ELIZA)} through the DAAD programme Konrad Zuse Schools of Excellence in Artificial Intelligence, sponsored by the Federal Ministry of Education and Research.

\section*{Contributions}
Ma.H. led the project. 
Mi.R. contributed Proposition 4.1. 
Mi.H. supervised the project.

\bibliography{iclr2027_conference}
\bibliographystyle{iclr2027_conference}

\newpage

\appendix

\tableofcontents

\newpage

\section{FAQ}\label{sec:app:faq}

\begin{enumerate}

    \item \textit{How do the results relate to \citet{yang2026length,zetzsche2026lengthgeneralizationtransformerscompression}?}

    \cite{yang2026length} give two core results: (i) an impossibility result for computable length generalization bounds for C-RASP, and (ii) a deterministic length generalization bound analogous to our Lemma~\ref{thm:equal-short-then-long} for $\text{C-RASP}^+$. 
    
    The first result indicates that C-RASP is too powerful a target for obtaining quantitative learning results for transformers: even though C-RASP has been successful in predicting generalization behaviour experimentally \citep{huang2025formal, yang2026algebraicdecompositiontheorytransformer}, there are C-RASP-expressible tasks where training data would need to include strings with astronomically large lengths.
    \citet{yang2026length, zetzsche2026lengthgeneralizationtransformerscompression} further show an analogue of our Lemma~\ref{thm:equal-short-then-long} for $\text{C-RASP}^+$, the positive fragment of C-RASP, giving a computable bound in the size of the defining C-RASP formula.
    Most recently, \cite{zetzsche2026lengthgeneralizationtransformerscompression} have further studied the role of sample size.
They show exponential sample size bounds for $\text{C-RASP}^+$ and another fragment, called $\text{C-RASP}_1$.

     $\text{C-RASP}^+$ has the same expressivity as fixed-precision transformers, though there is no direct 1-to-1 link between the transformer size parameters ($L$, $D$, $H$) and the size of defining C-RASP$^+$ formulas. As a consequence, length and sample complexities that we obtain may not be directly comparable to bounds obtained for  $\text{C-RASP}^+$ in \citet{yang2026length, zetzsche2026lengthgeneralizationtransformerscompression}.

     An important difference of our results, besides directly applying to a transformer model rather than the C-RASP$^+$ logic, is that we obtain joint bounds allowing us to understand the tradeoff between sample length and sample size.

\item \textit{How do the experiments differ from the theoretical learning model? What are the implications of these differences?}

The point of our experiments is to test theoretical predictions (obtained for idealized learners G\&C or ERM on fixed-precision transformers) in a realistic training setup using gradient-based optimization (Adam) on high-precision (FP32) transformers.
Whereas the theoretical setup asks for exact match of the output vectors, the experiments instead train for per-dimension binary cross-entropy.
In experiments, we select only runs with perfect training accuracy, enforcing interpolation, as assumed in theory.

\item \textit{Why study fixed-precision transformers? Why not study C-RASP / Limit Transformers?}

It is a well-established empirical finding that transformers show strong length generalization on the class of languages expressed by fixed-precision transformers \citep{li_cotterell2025fpt, yang2026algebraicdecompositiontheorytransformer}.
They are thus a natural target for joint treatment of sample and length complexity.

Fixed-precision transformers define a strict subset of C-RASP \citep{yang2024counting} and Limit Transformers \citep{huang2025formal, izzo2025quantitative}.
On the one hand, there is empirical evidence that transformers also show length generalization on tasks that are expressible in C-RASP but not by fixed-precision transformers, such as MAJORITY or certain regular languages such as bounded-depth Dyck languages \citep{yang2026algebraicdecompositiontheorytransformer}, which require a transformer's numerical precision to grow with the input length.
On the other hand, as discussed above, \cite{yang2026length} show that C-RASP cannot admit computable length generalization  bounds, i.e., any threshold $C$ satisfying (\ref{eq:from-short-to-long}) for all of C-RASP must grow faster in the size of the C-RASP formula (or Limit Transformer) than any computable function.
This indicates that C-RASP is too broad as a target for quantitative length generalization bounds \citep{yang2026length,zetzsche2026lengthgeneralizationtransformerscompression}.
It is an interesting question for future work whether tractable generalization bounds exist for some natural class between fixed-precision transformers and C-RASP.

\item \textit{How do the results relate to \citet{izzo2025quantitative, chen2025nonasymptotic}?}

\citet{izzo2025quantitative,chen2025nonasymptotic} provide quantitative length generalization bounds for C-RASP or Limit Transformers fragments with only one or two layers.
On the one hand, fixed-precision transformers form a strict subset of C-RASP/Limit Transformers; on the other hand, our results apply to the multi-layer setup, where computable quantitative bounds for C-RASP  are impossible (see Items 1 and 3).

    \item \textit{The paper studies ERM and G\&C. Why not study optimizers like SGD or Adam?}

Our work follows prior work \citep[e.g.][]{edelman2022inductive, yang2026length, chen2025nonasymptotic, izzo2025quantitative, trauger2024sequence} in studying idealized learners abstracting away from gradient dynamics. ERM in particular is a highly popular standard approach for deriving learning bounds.
Guess-and-check models the inductive bias of realistic training towards solutions occupying high volume in parameter space \citep{buzaglo_narrow_teachers_2024}, unlike ERM.
\editsMiH{Prior and concurrent work has provided evidence that parameter space volume is useful for understanding transformer generalization, finding that certain hard-to-learn functions occupy small volume in parameter space \citep{hahn2024sensitive,kover2026understanding,zhu2026exact}; G\&C is consistent with this by predicting easier generalization for functions occupying high parameter space volume.}
Neither G\&C nor ERM should be considered full models of learning. For instance, for matrix factorization, \citet{Alexander2025Neural} show that SGD can in some cases of matrix factorization generalize even better than G\&C.
Understanding gradient dynamics for general mutli-layer transformers is extremely difficult. Developing quantitative generalization bounds for gradient-based training is a major task for future work; the 1-layer setup might be a tractable starting point.

    \item \textit{Could the framework be extended to handle positional encodings?}

Prior bounds for statistical generalization (not length generalization) of transformers use fixed absolute positional encodings \citep{edelman2022inductive, yang2026tight,lintilhac2026sharperpicturegeneralizationtransformers}. That would also be possible to include in our setup, though one would need to understand their effects on the automaton structure.
Learnable absolute positional encodings are harder to account for as they lead to rapid growth of the parameter count with the input length (as acknowledged by \citet{edelman2022inductive}), and are inherently undefined outside of the training length range. \citet{huang2025formal} impose translation invariance constraints to constrain the hypothesis space in their analysis. %

Relative positional encodings can be easier to account for, as they do not involve an unbounded number of parameters.
\citet{yang2025kneedeep} show that  ALiBi \citep{press2021train} can be understood (in the context of C-RASP) as the addition of local predicates; we expect the effect on fixed-precision transformers to be very similar to the effect provided by local attention. We opt to analyze local attention, as it is nowadays popular in LLM design and has been rigorously formalized in the fixed-precision setup by \cite{li2026characterizing}.
RoPE \citep{su2024roformer} is harder to capture theoretically in the context of length generalization: while it can be formalized in terms of modular predicates \citep{yang2025kneedeep} (which admit exact automata-theoretic treatment in an appropriate extension of our poDFA framework, with appropriate modifications to the length complexity term), this does not characterize conventional RoPE implementation, which rather behave like fixed-length backwards-lookup operators \citep{jerad2026disentangling}, and may in fact impair empirical length generalization \citep{jerad2026disentangling}.

\item \textit{Can the theory predict specific numerical bounds?}

Our aim is to predict qualitative rates and the shape of the tradeoff between sample size and input length.
Predicting numerical bounds faces the challenge that constants depending on the precision set $\F$ will be extremely large at full-precision training.
Constants involving it are in general needed (Appendix~\ref{app:ExamplesTightness}).
It is a highly interesting (and challenging) problem for future work to make bounds numerically tighter.

\item \textit{Are the teacher constructions for the three task families minimal?}

For the Alternation Languages, \cite{yang2025kneedeep} show that transformer depth for $L_k$ must be $\geq k$ under a more expressive formalization of transformers (fixed-precision residual stream, unbounded-precision attention). 
Thus, our teacher family with $L_k = \Theta(k), D_k = O(1), H_k = O(1)$ is optimal up to rates.

For the Alphabet Languages, we found models hard to train when head count was $<k$, though we do not have a proof that an $\mathcal{FPT}$ requires $\Omega(k)$ heads. However, a simple information-theoretic argument shows that $D_k \cdot |\F| = \Omega(\log k)$ is unavoidable; hence, the teachers do need to grow with $k$.
For the Induction Head Languages, the same argument applies as $k \rightarrow \infty$.

\item \textit{Why are there missing points in some experimental plots?}

As detailed in Appendix~\ref{sec:experimental-protocol}, across all configurations, each run over 10 random seeds, we strictly retain only the runs that achieve 100\% training accuracy, consistent with the Guess \& Check conditioning assumption. 
\edits{Notably, we observe in Figure~\ref{fig:student_L} that Alternating languages $L_k$ can't reach 100\% training accuracy when the trained student has less than $k-1$ layers, as expected (\cite{yang2025kneedeep}). Similarly, students learning from Alphabet languages $A_k$ with head count $H_s < k$ struggle to fit the training data (Figure~\ref{fig:student_H}).}

\item \textit{What is strict causal masking? Why assume it?}

Strict causal masking means that attention heads cannot attend to the query position, but only to strictly preceding positions.
Strict causal masking increases expressivity compared to non-strict causal masking \citep{li_cotterell2025fpt}. We believe that our techniques and main results would equally apply in the non-strict-masking setup, though then applying to a less expressive model and correspondingly a smaller class of functions.

\item \textit{The theoretical bounds include effects of both teacher size (the transformer representing the target function) and the student size (the transformers in the learner's hypothesis class). Experiments show strong effects of the teacher but weak effects for the student. Why?}

Our theoretical predictions provide worst-case sufficient upper bounds on the test error; practical learning outcomes may be better both because the theoretical analysis may not be tight and because gradient-based learning may generalize better than the theoretical idealizations ERM and G\&C.
Indeed, it is a common observation that key techniques for proving statistical learning bounds show an overly pessimistic dependency on the student size: overparameterized neural networks often generalize very well, far better than is accounted for by capacity-based bounds that guarantee learning for ERM via uniform convergence \citep[e.g.][]{ nagarajan2021uniformconvergenceunableexplain,zhang2017understandingdeeplearningrequires, bartlett2021deep}.
The G\&C analysis as developed by \citet{buzaglo_narrow_teachers_2024}, that we adopt in Section~\ref{sec:statistical-gen-only}, goes some way towards ameliorating this, as its dependence on the student size is substantially milder than for the capacity-based ERM bounds.
Some work has provided generalization bounds that scale very mildly in the student size for simpler neural networks \citep[e.g.][]{allen2019learning,pmlr-v97-arora19a}, though the extension of such results to transformers will likely require new techniques.

\item \textit{The bound $C$ on the height of the difference automata is chosen as the minimum of two bounds (Lemma~\ref{thm:equal-short-then-long}). When is which bound stronger? What do we know about tightness?}

$C_0$ involves the layer count, hidden dimension, and per-head dimension of both student and teacher; it is exponential in the per-head dimension and linear in the others.
$C_1$ involves only the student layer count and the alphabet size; it \edits{is} exponential in either \edits{and becomes polynomial in $L_s$ when the alphabet is fixed}.
Importantly, $C$ is the \emph{minimum} both, i.e., our results always apply to the smaller of both bounds.
Thus, even if our theorems were stated just for $C := C_0$ or $C := C_1$, they would remain valid; giving two versions strengthens our results.

The two bounds have strengths in different hyperparameter regimes. In the regime where models are very deep, $C_0$ is asymptotically smaller, and $C$ remains \emph{linear} in $L_s$ and $L_t$.
In the regime where models are very wide, $C_1$ is asymptotically smaller, as it does not depend on the hidden dimension.
The exception is when the number of heads also increases: then, the per-head dimension $D_{h,s}$ decreases, and $C_0$ may become more favorable.

Notably, only two quantities enter both bounds: $L$ and $|\F|$; both bounds additionally require other quantities, but they can trade off which each other; e.g., as discussed above, taking into account $D$ can turn the depth dependence from exponential (in $C_1$) to linear (in $C_0$).
We show the unavoidability of including the two quantities $L$ and $|\F|$ in Appendix~\ref{app:ExamplesTightness}.

\end{enumerate}

\newpage

\section{Conventions and Background}
\label{sec:conventions-background}

We specify the fixed-precision transformer model used throughout the paper,
following \citet[Section~3 and Appendix~B]{li_cotterell2025fpt}. Our main adaptations are that we retain the
number of attention heads explicitly \edits{and that we drop layer norm as well as the attention score normalization} \editsMiH{(see Appendix~\ref{sec:app:normalization-constraints} for more)}. In addition, we distinguish tokenwise vectorial outputs from their
word-level binary satisfaction framework. 

\subsection{Fixed-precision arithmetic}

\paragraph{Fixed-precision set.}
\label{def:finite-precision-set}
A \textit{fixed-precision set} or \textit{finite-precision set} $\F$ is a finite ordered set, constructed from a subset of the real numbers $\F_{\mathrm{fin}}$ containing $0$ and $1$, and extended with the two infinite limits:
\begin{equation}
 \F=\F_{\mathrm{fin}}\cup\{-\infty,+\infty\},
 \qquad \{0,1\}\subseteq\F_{\mathrm{fin}}\subset\R,
 \qquad p\coloneqq|\F|.
 \label{eq:precision-set}
\end{equation}
The set $\F$ is equipped with deterministic rules for rounding every elementary operation back to $\F$, including overflow and underflow. When we study transformers on a fixed-precision set $\F$, all scalar parameters and intermediate calculation values must belong to $\F$. The precision convention is fixed and independent of the input length.

We use the natural extended-real rules whenever defined, in particular
\[
 \exp(-\infty)=0,\qquad \exp(+\infty)=+\infty,\qquad
 f/(\pm\infty)=0\quad(f\in\F_{\mathrm{fin}}).
\]
More details about the edge cases stemming from the finite-precision arithmetic can be found in \cite{li_cotterell2025fpt}.

\subsection{Transformers}
\label{sec:fpt-definition}

Transformers read words $w=w_1\cdots w_n\in\Sigma^*$ directly; throughout,
$n=|w|$ refers to the length of an input word. Differently to \cite{li_cotterell2025fpt}, no end-of-sequence token is appended and transformer computations
are performed only at the ordinary positions $1,\ldots,n$. \edits{This allows for a broader definition of transformers, with vectorial outputs at every position.} 

We use $i,j$ for positions, $\ell$ for layers, $d$ for the vectorial dimension and $h$ for heads.

\begin{definition}[Fixed-precision transformer]
\label{def:fpt}
Fix a depth $L\geq0$, a width $D\geq1$, and a number of heads $H\geq1$
with $H\mid D$; put $D_h=D/H$. A fixed-precision transformer assigns to
each $w\in\Sigma^*$ of length $n$ an output
\[
 T(w)\in\F^{\nu\times n}.
\]
It consists of a token embedding, $L$ transformer blocks, and a pointwise
output layer. Thus $T(w)_{:,i}\in\F^\nu$ is defined at every ordinary
position $1\leq i\leq n$. We use $\FPT(L,D,H)$ for the class of such
transformers with the corresponding parameter sizes.

\begin{remark*}
All calculations below use the fixed-precision convention of
Definition~\ref{def:finite-precision-set}.
\end{remark*}

\paragraph{Embedding.}
A map $e:\Sigma\to\F^D$ determines the $0$-th layer $\Lambda^{(0)}$, with
\begin{equation}
 \Lambda^{(0)}(w)\in\F^{D\times n},
 \qquad \Lambda^{(0)}_{:,i}(w)=e(w_i)
 \quad(1\leq i\leq n).
 \label{eq:fpt-input}
\end{equation}
There are no positional encodings.

For $0\leq\ell < L$ and $1\leq h\leq H$, we define the attention mechanism
of head $h$ with input $\Lambda^{(\ell)}\in\F^{D\times n}$. We sometimes
omit $\ell$ and $h$ from the notation for readability purposes.

\paragraph{Attention.} An attention mechanism is a map
$\operatorname{A}_h:\F^{D\times n}\to\F^{D_h\times n}$. First, choose
$\edits{W_Q,W_K,W_V}
\in\F^{D_h\times D}$ and set
\begin{equation}
 \edits{Q=W_Q}\Lambda^{(\ell)},\qquad
 \edits{K=W_K}\Lambda^{(\ell)},\qquad
 \edits{V=W_V}\Lambda^{(\ell)}.
 \label{eq:qkv}
\end{equation}
The attention mechanism has a strict causal masking: the visible positions at $i$ are $\{1,\ldots,i-1\}$. For $1\leq j< i$, compute the attention score
\begin{equation}
 S_{i,j}
 := \edits{Q_{:,i}^{\mathsf T}K_{:,j}.}
 \label{eq:fpt-score}
\end{equation}
For a given position $i$, we consider the maximum of queried attention scores at position i $\max_i := \max_{j<i}S_{i,j}$
and define the stabilized attention weights
\[ \alpha_{i,j}
 := \frac{\exp(S_{i,j}-\max_i)}{\sum_{k < i} \exp(S_{i,k}-\max_i)}.
 \label{eq:fpt-softmax-denominator}
\]

The attention map is consequently defined as
\[
\operatorname{A}_h(\Lambda^{(\ell)})_{:,i}
 := \sum_{j<i} \alpha_{i,j}V_{:,j}.
\]

As there is no well-defined attention score for the first index $i=1$, we set
$\operatorname{A}_h(\Lambda^{(\ell)})_{:,1}:=\mathbf{0}\in\F^{D_h}$.

We define by extension the head-agnostic attention mechanism $\operatorname{A}^{(\ell)}$ as the concatenation of the attention heads, with
\[
  \operatorname{A}^{(\ell)}_{:,i}(\Lambda^{(\ell)}) := \left(\operatorname{A}_1(\Lambda^{(\ell)})_{:,i}^{\mathsf{T}},...,\operatorname{A}_H(\Lambda^{(\ell)})_{:,i} ^{\mathsf{T}}\right)^{\mathsf{T}}.
\]

For each intermediate sublayer $\Lambda^{(\ell+0.5)}, 0\leq \ell < L$, we define a position-wise feed-forward mechanism applied to this sublayer. We consider feed-forwards with one hidden layer of dimension \edits{$D_{\mathrm{ff}}$}. We set \edits{$D_{\mathrm{ff}} = D$} for the base case, but it can be changed \edits{without repercussions as long as the statistical complexity guarantees given in Appendix~\ref{sec:app:data-complexity:main} are adapted accordingly.} We consider the $\operatorname{ReLU}$ activation function.

\paragraph{Feed-forward.}
Define $\operatorname{FFN}_\ell: \F^D\to\F^D$, applied independently at
every one of the $n$ positions, as
\[
 \operatorname{FFN}_\ell(x)
 :=\Theta^{F_2}_\ell\operatorname{ReLU}
   (\Theta^{F_1}_\ell x+b^{F_1}_\ell)+b^{F_2}_\ell,
 \qquad \operatorname{ReLU}(t)=\max(0,t),
\]
where $\Theta^{F_1}_\ell \in \edits{\F^{D_{\mathrm{ff}}\times D}}$,
$\Theta^{F_2}_\ell \in \edits{\F^{D\times D_{\mathrm{ff}}}}$, and the biases have matching
dimensions. The parameters are shared across all positions.

\paragraph{Transformer block.} Transformer blocks are \edits{simply} the \edits{composition} of an attention mechanism with a feed-forward, \edits{involving} residual streams \edits{at each intermediate step. Layer norm is omitted; explanations on this architecture choice can be found in Appendix~\ref{sec:app:normalization-constraints}. Consequently, each} of the $L$ transformer block is
\begin{align}
 \Lambda^{(\ell+\frac12)}
 &= \edits{\operatorname{A}^{(\ell)}(\Lambda^{(\ell)})+\Lambda^{(\ell)}},
 \label{eq:fpt-att-residual}\\
 \Lambda^{(\ell+1)}
 &= \edits{\operatorname{FFN}_\ell(\Lambda^{(\ell+\frac12)})
           +\Lambda^{(\ell+\frac12)}}.
 \label{eq:fpt-ff-residual}
\end{align}

\paragraph{Output Layer}
The output layer is the pointwise map $\operatorname{O}:\F^D\to\F^\nu$
given by
\[
 \operatorname{O}(x)
 := \Theta^{\mathrm{out}} x+b^{\mathrm{out}}
\]
where $\Theta^{\mathrm{out}} \in \F^{\nu \times D}$ and $b^{\mathrm{out}} \in \F^\nu$ are the parameters shared across positions.
Thus
\[
 T(w)_{:,i}=\operatorname{O}\bigl(\Lambda^{(L)}_{:,i}(w)\bigr)
 \qquad(1\leq i\leq n).
\]
\end{definition}

\paragraph{Attention arithmetic assumptions.}\label{app:arithmetic-assumptions}
For the transformers considered below, we impose the following conventions,
motivated by the positive-sum argument of
\citet[Appendix~B.2, Lemma~B.5]{li_cotterell2025fpt}.
Attention sums depend only on the multiset of summands, and inserting zeros
does not change them.
Nonnegative summands have a nonnegative sum, with $+\infty$ absorbing.
There is an integer $K_\F\geq1$, independent of input length, such that any
such sum containing more than $K_\F$ copies of a strictly positive finite
value is $+\infty$. We write
\begin{equation}
 r_\F\coloneqq K_\F+1.
 \label{eq:attention-count-cutoff}
\end{equation}
Stabilized exponentials are nonnegative; zero weights contribute zero even
for infinite values. An empty prefix or a softmax denominator equal to $0$
or $+\infty$ gives the zero head output. All remaining arithmetic cases are
resolved deterministically.

\newpage

\section{Logic}
\label{sec:logic}

In this section, we define the past fragments of the linear temporal logic and of a two-variable first-order logic,
introduced by \citet{li_cotterell2025fpt}. These fragments are expressively
equivalent and provide the link between temporal formulas and the
partially-ordered structure studied below, that applies to transformers.

Fix a finite nonempty alphabet $\Sigma$ and a word
$w=w_1\cdots w_n\in\Sigma^*$. The positions $1,\ldots,n$ are the ordinary
token positions. We additionally use $n+1$ as a \emph{virtual logical
endpoint}: it is an unlabelled position after the word, used only to evaluate
word-level formulas.

\subsection{Past temporal logic}
\label{sec:ltl-past}

First we introduce $\LTL$, the past fragment of the full linear temporal logic $\textbf{LTL}[\textbf{N}, \textbf{U}, \textbf{S}, \textbf{P}]$ that provides several operators, such as \emph{next}, \emph{until}, and their past-time counterparts \emph{since} and \emph{past}. In our situation, the temporality refers to the positions of a given word, earlier positions thought as being in the past of later positions. The
fragment $\LTL$ uses only the past operator $\textbf{P}$, read “at some
strictly earlier position”. The restriction to the past operator reflects the
strict masking of the attention mechanism. We use its pointed formulas to
describe tokenwise transformer outputs, and evaluate word-level formulas at
the virtual endpoint.

\begin{definition}[Syntax and semantics of $\LTL$]
\label{def:ltl-past}
The formulas of $\LTL$ are generated by
\begin{equation}
 \varphi ::= \top\mid\pi_a\mid\neg\varphi
             \mid(\varphi\wedge\varphi)\mid\Past\varphi,
 \qquad a\in\Sigma.
 \label{eq:ltl-grammar}
\end{equation}
For any position $1\leq i\leq n+1$, where $n+1$ is the virtual endpoint,
satisfaction is defined inductively by
\begin{align*}
 w,i \models& \top
 \quad\text{always},\\
 w,i\models& \pi_a
 \iff i\leq n\text{ and }w_i=a,\\
 w,i\models& \neg\varphi
 \iff w,i\not\models\varphi,\\
 w,i\models& \varphi\wedge\psi
 \iff w,i\models\varphi\text{ and }w,i\models\psi,\\
 w,i\models& \Past\varphi
 \iff \exists j<i\text{ with }w,j\models\varphi.
\end{align*}
Thus $\Past$ means ``at some strictly earlier position''. The connectives
$\bot,\vee,\Rightarrow$ are Boolean abbreviations. In particular,
$\neg\Past\neg\varphi$ means that $\varphi$ holds at every earlier
position, and is vacuously true at position $1$.
\end{definition}

\begin{definition}
\label{def:ltl-word-satisfaction}
For a logic formula $\Psi\in\LTL$, set
\begin{equation}
 w\models\Psi := w,n+1\models\Psi,
 \qquad \qquad L(\Psi)\coloneqq\{w\in\Sigma^*:w\models\Psi\}.
 \label{eq:ltl-language}
\end{equation}
Every atomic formula $\pi_a$, $a\in\Sigma$, is false at the virtual endpoint. In particular,
the formula $\neg\bigvee_{a\in\Sigma}\pi_a$ holds precisely at this endpoint. We also have $L(\neg\Past\top)=\{\eps\}$.
\end{definition}

\begin{example}
\label{ex:ltl-languages}
For $a,b,c\in\Sigma$,
\begin{align*}
 L(\Past\pi_a)
 &=\Sigma^*a\Sigma^*,\\
 L\bigl(\Past(\pi_c\wedge\Past(\pi_b\wedge\Past\pi_a))\bigr)
 &=\Sigma^*a\Sigma^*b\Sigma^*c\Sigma^*,\\
 L\bigl(\Past(\pi_a\wedge\neg\Past\top)\bigr)
 &=a\Sigma^*.
\end{align*}
The last formula selects an $a$ with no predecessor, hence checking for first symbol. Note that we nevertheless can't express the language $\Sigma^*a$ in $\LTL$.
\end{example}

\begin{lemma}[Prefix invariance]
\label{lem:ltl-prefix-invariance}
For $u,v\in\Sigma^*$, $1\leq i\leq|u|$, and $\varphi\in\LTL$,
\[
 u,i\models\varphi\quad\Longleftrightarrow\quad uv,i\models\varphi.
\]
\end{lemma}
\begin{proof}
Induct on $\varphi$. Atomic predicates depend only on the current symbol,
and $\Past$ quantifies only over positions strictly before $i$.
\end{proof}

\subsection{Past two-variable first-order logic}
\label{sec:pfo-past}
The second logic that we introduce, a first-order logic, is equivalent to $\LTL$. The two-variable
fragment $\mathbf{FO}^2[<]$ uses only the position variables $x,y$, which may
be reused under quantifiers. We then further restrict to its past fragment $\PFO$: 
each new occurrence of a free variable must be upper
bounded by the other, again reflecting strict causal masking. We first give
its pointed formulas, then its sentences.
Naming two positions with the variables $x$ and $y$ is in particular useful for describing an attention query and a key.

The syntax mirrors the construction of $\LTL$: formulas are built from
atomic statements using Boolean connectives and one restricted way of
looking into the past.

\begin{definition}[Syntax and semantics of $\PFO$]
\label{def:pfo-past}
For $a\in\Sigma$ and variables $x,y$, formulas are generated by
\[
 \theta ::= \top \mid \pi_a(x) \mid \neg\theta
             \mid(\theta\wedge\theta)
             \mid \exists y<x:\theta(x,y),
 \label{eq:pfo-past-quantifier}
\]
\edits{where $\theta(x,y)$ is always defined by a boolean combination of two unary formulas; e.g. $\theta(x,y) \coloneqq \psi(x) \wedge \psi(y)$.}
As in Appendix~\ref{sec:ltl-past}, positions range over
$1,\ldots,n+1$, with $n+1$ the virtual endpoint. For a current position $i$
and, when needed, an earlier position $j$, satisfaction is defined inductively by
\begin{align*}
 w,i\models&\ \top
 &&\text{always},\\
 w,i\models&\ \pi_a(x)
 &&\Longleftrightarrow\ i\leq n\text{ and }w_i=a,\\
 w,i\models&\ \neg\theta
 &&\Longleftrightarrow\ w,i\not\models\theta,\\
 w,i\models&\ \theta\wedge\psi
 &&\Longleftrightarrow\ w,i\models\theta\text{ and }w,i\models\psi,\\
 w,i\models&\ \exists y<x:\theta(x,y)
 &&\Longleftrightarrow\ \exists j<i:\ w,i,j\models\theta(x,y).
\end{align*}
A formula with one free variable is a \emph{pointed formula}; a formula
with no free variables is a \emph{sentence}. We use an unbounded existential
quantifier only to close a pointed formula into a sentence, as in
$\exists x:\theta(x)$. Boolean connectives and bounded universal quantifiers
are abbreviations as usual.
\end{definition}

\begin{example}
\label{ex:pfo-subsequence}
The sentence
\[
 \exists x:\bigl(\pi_c(x)\wedge
   \exists y<x:\bigl(\pi_b(y)\wedge\exists x<y:\pi_a(x)\bigr)\bigr)
\]
defines $\Sigma^*a\Sigma^*b\Sigma^*c\Sigma^*$. The innermost quantifier
reuses $x$ for the $a$-position: within its scope, the outer $c$-position
is no longer named. This is the first-order analogue of
$\Past(\pi_c\wedge\Past(\pi_b\wedge\Past\pi_a))$ in $\LTL$.
\end{example}

\begin{remark}
\label{rem:pfo-counting}
For every fixed $k\in\N$ and every one-variable formula $\theta(z)$, we
can express that at least, or exactly, $k$ earlier positions satisfy
$\theta$. Define recursively
\begin{align*}
 \exists^{\geq 0}y<x:\theta(y)&\coloneqq\top,\\
 \exists^{\geq k+1}y<x:\theta(y)&\coloneqq
 \exists y<x:\bigl(\theta(y)\wedge
   \exists^{\geq k}x<y:\theta(x)\bigr),\\
 \exists^{=k}y<x:\theta(y)&\coloneqq
 \bigl(\exists^{\geq k}y<x:\theta(y)\bigr)\wedge
 \neg\bigl(\exists^{\geq k+1}y<x:\theta(y)\bigr).
\end{align*}
The recursive definition reuses $x$ and $y$ at each level, so it remains
in $\PFO$. These are finite formulas for each fixed threshold $k$; they do
not add an unbounded counting operation to the logic.
\end{remark}

\begin{lemma}[Temporal--first-order correspondence]
\label{thm:ltl-pfo-equivalence}
The logics $\LTL$ and $\PFO$ have the same expressive power
\citep[Theorem~2.1]{li_cotterell2025fpt}. More precisely, every $\LTL$
formula is equivalent at each position to a $\PFO$ formula with one free
variable, and every such $\PFO$ formula has an equivalent $\LTL$ formula;
consequently, their sentences define the same class of languages. Under
this correspondence, the basic one-position past test, where $\varphi$ has
only $y$ free, is translated as
\[
 \Past\varphi
 \qquad\longleftrightarrow\qquad
 \exists y<x:\varphi(y).
\]
\end{lemma}

\subsection{Transformers as logic formulas}
\label{sec:transformers-as-logic-formulas}

Under the attention arithmetic assumptions of Appendix~\ref{sec:fpt-definition},
fixed precision lets us describe every coordinate at every ordinary token
position by a pointed $\PFO$ formula. This is the tokenwise version of the
compilation of \citet{li_cotterell2025fpt}; in particular, it does not use an
$\mathtt{EOS}$ token.

\begin{proposition}[Transformer-to-logic]
\label{prop:transformer-to-logic}
Let $T\in\FPT(L,D,H)$ have output dimension $\nu$.  For every output
coordinate $1\leq d\leq\nu$ and $f\in\F$, there is a pointed formula
$\Phi^T_{d,f}(x)\in\PFO$ such that, for every $w\in\Sigma^*$ and every
ordinary position $1\leq i\leq |w|$,
\begin{equation}
 w,i\models\Phi^T_{d,f}(x)
 \quad\Longleftrightarrow\quad
 T(w)_{d,i}=f.
 \label{eq:tokenwise-transformer-formula}
\end{equation}
More generally, the same statement holds for every coordinate and value of
each residual layer $\Lambda^{(\ell)}$.
\end{proposition}

\begin{proof}[Construction]
For every layer $\ell$, coordinate $d$, and $f\in\F$, construct
$\Phi^{(\ell)}_{d,f}(x)$ so that
\begin{equation}
 w,i\models\Phi^{(\ell)}_{d,f}(x)
 \quad\Longleftrightarrow\quad
 \Lambda^{(\ell)}_{d,i}(w)=f.
 \label{eq:layer-value-invariant}
\end{equation}
The embedding case is immediate:
\begin{equation}
 \Phi^{(0)}_{d,f}(x)
 =\bigvee_{\substack{a\in\Sigma\\e(a)_d=f}}\pi_a(x).
 \label{eq:embedding-value-formula}
\end{equation}

Every pointwise map $G:\F^m\to\F^r$ is a finite lookup table.  If its
input coordinates are described by $\varphi_{j,g}(x)$, its $d$-th output has
value $f$ precisely when
\begin{equation}
 \Phi^{G}_{d,f}(x)
 \coloneqq
 \bigvee_{\substack{\mathbf g\in\F^m\\G(\mathbf g)_d=f}}
 \ \bigwedge_{j=1}^{m}\varphi_{j,g_j}(x).
 \label{eq:pointwise-lookup-formula}
\end{equation}
Thus projections, residual connections, feed-forward
maps, and the final output map preserve the construction using Boolean
operations only.

It remains to consider one attention head.  Its query, key, and value vectors
are pointwise maps, so the preceding display gives formulas for every
$\mathbf q,\mathbf k,\mathbf v\in\F^{D_h}$.  Denote the corresponding
formulas by
\[
 \mathrm{Q}^{(\ell,h)}_{\mathbf q}(x),\qquad
 \mathrm{K}^{(\ell,h)}_{\mathbf k}(y),\qquad
 \mathrm{V}^{(\ell,h)}_{\mathbf v}(y);
\]
for example,
$w,i\models\mathrm{Q}^{(\ell,h)}_{\mathbf q}(x)$ if and only if
$Q_{:,i}=\mathbf q$.  The relevant query--key/value relation is the
two-variable formula
\begin{equation}
 \theta_{\mathbf q,\mathbf k,\mathbf v}(x,y)
 \quad\equiv\quad
 \mathrm{Q}^{(\ell,h)}_{\mathbf q}(x)
 \wedge\mathrm{K}^{(\ell,h)}_{\mathbf k}(y)
 \wedge\mathrm{V}^{(\ell,h)}_{\mathbf v}(y).
 \label{eq:attention-unary-type}
\end{equation}
For each current query position $x$, the maximum score, softmax denominator,
and weighted value sum are finite case distinctions over the multiplicities
of earlier positions $y$ satisfying these relations.  In particular, the
required fixed thresholds vary with the query position and are expressed
directly in $\PFO$ by
\begin{equation}
 \exists^{\geq r}y<x:\theta_{\mathbf q,\mathbf k,\mathbf v}(x,y),
 \label{eq:attention-threshold-formula}
\end{equation}
and their exact-count Boolean combinations, as in
Remark~\ref{rem:pfo-counting}.  Let
$\Xi^{(\ell,h)}_{\mathbf q,d,f}(x)$ denote the resulting finite Boolean
combination that selects the cases in which head $h$ outputs value $f$ in
coordinate $d$ for query $\mathbf q$.  Its output formula is
\begin{equation}
 \Phi^{\operatorname{A},(\ell,h)}_{d,f}(x)
 =
 \bigvee_{\mathbf q\in\F^{D_h}}
 \left(
   \mathrm{Q}^{(\ell,h)}_{\mathbf q}(x)
   \wedge\Xi^{(\ell,h)}_{\mathbf q,d,f}(x)
 \right).
 \label{eq:attention-output-formula}
\end{equation}
Fixed precision gives finitely many query cases, and the arithmetic
assumptions of Appendix~\ref{sec:fpt-definition} let us clip multiplicities
at $r_\F$. Thus only finitely many cases are needed.

Consequently, attention produces formulas satisfying
\eqref{eq:layer-value-invariant} for the next sublayer, and the pointwise
steps complete the induction over $\ell$.  Applying the final output lookup
formula~\eqref{eq:pointwise-lookup-formula} yields
\eqref{eq:tokenwise-transformer-formula}.
\end{proof}

All formulas in Proposition~\ref{prop:transformer-to-logic} are evaluated at
ordinary token positions.  By Lemma~\ref{thm:ltl-pfo-equivalence}, each has an
equivalent pointed $\LTL$ formula, \edits{using the $\PFO$} presentation to keep the
current position $x$ and the earlier position $y$ explicit.

\subsection{Partially ordered automata}
\label{sec:transformer-mealy}

Word-level formulas define languages, whereas the pointed formulas
\eqref{eq:tokenwise-transformer-formula} describe an output at each token.
The corresponding automata verifying the formulas are respectively DFAs and Mealy machines.

\begin{definition}[Deterministic Finite Automaton]
\label{def:ordered-automata}
A \emph{DFA} is a tuple $\A=(Q,\Sigma,\delta,q_0,F_A)$, with finite state
set $Q$, initial state $q_0$, accepting set $F_A\subseteq Q$, and transition
function $\delta:Q\times\Sigma\to Q$. By abuse of notation, extend $\delta$
to words by iterating over their letters:
\[
 \delta(q,\eps)=q,\qquad
 \delta(q,ua)=\delta(\delta(q,u),a).
\]
It recognises the language $L(\A)=\{u:\delta(q_0,u)\in F_A\}$.
\end{definition}

For our results, we do not need the following Mealy machine construction, as we work 
\edits{with the difference function which} is expressed by a poDFA.
Nevertheless, Mealy machines \edits{describe individual transformers: they} process
the input left to right and emit an output vector at every position. \edits{We shall see also that the height arguments can directly apply on Mealy machines.
In general, Mealy machines provide} the opportunity to deepen the ordered-automaton analyses of the transformers outside of \edits{the} student--teacher learning framework.

\begin{definition}[Mealy Machine]
\label{def:mealy-machine}
A \emph{Mealy machine} $\mathcal M=(Q,\Sigma,\Gamma,\delta,q_0,\lambda)$ is an automaton in the same way as the DFA \ref{def:ordered-automata},
but has a finite output alphabet $\Gamma$, an output map
$\lambda:Q\times\Sigma\to\Gamma$ and \edits{no} notion of \edits{accepting states}. On input $w=w_1\cdots w_n$, it emits
\[
 \mathcal M(w)_i=\lambda(\delta(q_0,w_{1:i-1}),w_i)
 \qquad(1\leq i\leq n).
\]
\end{definition}

Intuitively, Mealy machines are automata such that their restricted behaviour, with respect to one output symbol, resembles a DFA.

\begin{definition}[Partially ordered automata]
A DFA $\mathcal A$ or a Mealy  machine $\mathcal M$, is said to be \emph{partially ordered} if reachability of its states is antisymmetric,
equivalently if it contains no cycles but self-loops. Its \emph{height}
is the largest number of strict state changes along a run:
\[
 \operatorname{ht}(\mathcal M)
 =\max_{w\in\Sigma^*}
 \bigl|\{i:\delta(q_0,w_{1:i})\neq\delta(q_0,w_{1:i-1})\}\bigr|,
\]
with the same definition for $\A$. It is \emph{reduced} if all states are
reachable and if two distinct states have different future behaviour. 
\end{definition}

\begin{remark*}
Different future behaviours for DFAs means that for any two distinct states representers $w, w'$, there exists a suffix $v \in \Sigma^*$ such that, w.l.o.g. $\delta(w,v) \in F_A$ while $\delta(w',v) \notin F_A$. For Mealy machines, it similarly means that for some suffix concatenation,
the future emitted sequences must disagree.
\end{remark*}

We recall the language-level correspondence established by
\citet[Section~4 and Appendix~C]{li_cotterell2025fpt}.

\begin{proposition}[]
\label{prop:past-language-podfa}
For every $\Psi\in\LTL$, the language
$L(\Psi)=\{w:w,|w|+1\models\Psi\}$ is regular and its reduced DFA is
partially ordered. The same holds for the equivalent word-level
$\PFO$ formulas.
\end{proposition}

\begin{proof}
Li and Cotterell show that the syntactic monoid of an $\LTL$-definable
language is $\mathcal R$-trivial~\citep[Appendix~C]{li_cotterell2025fpt}:
distinct monoid elements have distinct principal right ideals.
The characterisation of \citet{brzozowski_rtrivial} identifies these
regular languages exactly with those whose reduced DFA is partially
ordered. The $\PFO$ statement follows from
Lemma~\ref{thm:ltl-pfo-equivalence}.
\end{proof}

The states of this automaton have the usual right-context interpretation
given by Myhill--Nerode~\citep{myhill1957finite,nerode1958linear}.

\begin{proposition}
\label{prop:past-dfa-right-context}
Let $\A$ be the reduced DFA of a regular language $\mathcal L$. For any
$u,v\in\Sigma^*$,
\[
 \delta(q_0,u)=\delta(q_0,v)
 \quad\Longleftrightarrow\quad
 \forall z\in\Sigma^*:\quad
 (uz\in\mathcal L\ \Longleftrightarrow\ vz\in\mathcal L).
\]
Thus its states are precisely the classes of prefixes with the same
accepting continuations. For $\mathcal L=L(\Psi)$, these classes form
the states of the reduced poDFA above.
\end{proposition}

\begin{proof}
Prefixes reaching the same state have the same acceptance decision on every
continuation, by determinism. Conversely, two distinct states of a reduced
DFA admit a distinguishing continuation. Every state is reachable, so this
identifies the states with the right-context classes.
\end{proof}

We now apply this correspondence to the tokenwise transformer formulas \ref{eq:tokenwise-transformer-formula}.
Their past tests are recognised by poDFAs, and the current letter determines
the output together with the states of these automata.

\begin{proposition}[Transformers as partially-ordered Mealy machines]
\label{prop:past-formulas-mealy}
Every transformer admitting the formulas
\eqref{eq:tokenwise-transformer-formula} is realised by a finite partially
ordered Mealy machine with output alphabet $\Gamma=\F^\nu$:
\[
 \mathcal M_T(w)_i=T(w)_{:,i}\qquad(1\leq i\leq|w|).
\]
In particular, Proposition~\ref{prop:transformer-to-logic} provides these
formulas for the transformers of Appendix~\ref{sec:fpt-definition}.
\end{proposition}

\begin{proof}
Translate the pointed $\PFO$ formulas into equivalent $\LTL$ formulas
$\Psi^T_{d,f}$ by Lemma~\ref{thm:ltl-pfo-equivalence}, and let $\mathcal P$
collect their distinct past subformulas $\chi=\Past\psi$.
For each $\chi$, Proposition~\ref{prop:past-language-podfa} provides a
reduced poDFA $\A_\chi=(Q_\chi,\Sigma,\delta_\chi,q_{0,\chi},F_\chi)$
recognising $L(\chi)$. After a prefix $u$, its acceptance bit gives exactly
the corresponding past test at the next token:
\[
 \delta_\chi(q_{0,\chi},u)\in F_\chi
 \quad\Longleftrightarrow\quad
 u\sigma,|u|+1\models\chi
 \qquad(\sigma\in\Sigma).
\]
Indeed, both sides ask whether $\psi$ holds at some position in $u$;
prefix invariance preserves those truth values.

Take the reachable part of the product of these poDFAs, with
\[
 q_0=(q_{0,\chi})_{\chi\in\mathcal P},\qquad
 \delta(q,\sigma)=\bigl(\delta_\chi(q_\chi,\sigma)\bigr)_{\chi\in\mathcal P}.
\]
For each current letter $\sigma$, evaluate $\Psi^T_{d,f}$ by replacing $\pi_a$
with $\mathbf1\{\sigma=a\}$ and each past test $\chi$ with
$\mathbf1\{q_\chi\in F_\chi\}$, respecting Boolean connectives.
Set $\lambda(q,\sigma)_d$ to the unique $f$ for which this evaluation is true.
Equation~\eqref{eq:tokenwise-transformer-formula} ensures existence and
uniqueness on reachable states. The output uses the state before reading
$\sigma$, so the past remains strict. Finally, every component can only advance
in its reachability order; their finite product is therefore partially
ordered. If $\mathcal P$ is empty, the product has one state.
\end{proof}

This proposition establishes the inherent ordered structure of the transformers, using Mealy machines. We shall see quickly that
the difference function between two transformers translates naturally in the other ordered automaton that we defined earlier: the poDFA.

\begin{remark}
\label{rem:token-predicates-automata}
Each $\Phi^T_{d,f}(x)$ gives a Boolean output on transitions of this machine.
Testing it at the last ordinary token need not define a poDFA language:
$\pi_a(x)$ would give $\Sigma^*a$, whose reduced DFA has a nontrivial cycle
when $|\Sigma|\geq2$. The poDFAs used above instead recognise past tests
at the virtual endpoint. Importantly, a fixed Mealy state may still emit different
vectors on different input letters.
\end{remark}

To identify the states intrinsically, write
\[
 T_u(z)\coloneqq
 \bigl(T(uz)_{:,|u|+j}\bigr)_{j=1}^{|z|}.
\]
This is the sequence emitted on the continuation $z$ after reading $u$;
it excludes the outputs already emitted on $u$.

\begin{proposition}
\label{prop:mealy-right-context}
For a transformer satisfying Proposition~\ref{prop:past-formulas-mealy},
there is a unique reduced Mealy machine, up to isomorphism.
Its states are the classes of
\[
 u\sim_T v
 \quad\Longleftrightarrow\quad
 \forall z\in\Sigma^*:\ T_u(z)=T_v(z),
\]
and its transitions and outputs are
\[
 q_0=[\eps],\qquad
 \delta([u],a)=[ua],\qquad
 \lambda([u],a)=T(ua)_{:,|u|+1}.
\]
This reduced machine is partially ordered. Consequently,
\[
 \delta(q_0,u)=\delta(q_0,v)
 \quad\Longleftrightarrow\quad
 \forall z\in\Sigma^*:\ T_u(z)=T_v(z).
\]
\end{proposition}

\begin{proof}
Causality gives $T_u(az)=T_u(a)\,T_{ua}(z)$, with concatenation of output
sequences. Hence $u\sim_T v$ implies equal outputs on $a$ and
$ua\sim_T va$, so the displayed maps are well defined. Equal states of
the finite machine in Proposition~\ref{prop:past-formulas-mealy} imply
equivalent prefixes; our construction is therefore a finite quotient.
All classes are reachable and distinct classes have a distinguishing
continuation. In any reduced realisation, these same classes correspond
bijectively to states, proving reducedness and uniqueness.

Finally, reduction preserves partial order: repeatedly traversing a
nontrivial quotient cycle would force infinitely many state changes in
the original finite partially ordered machine, a contradiction.
\end{proof}

\subsection{Past subformulas and automaton height}
\label{sec:past-profile-height}
We relate past subformulas to the height of their associated automata.
Word satisfaction is always evaluated at the virtual endpoint
$|w|+1$ of Appendix~\ref{sec:ltl-past}.

\begin{definition}[Subformulas]
\label{def:past-subformulas}
For $\Psi\in\LTL$, let $\Sub(\Psi)$ be the set of its syntactically distinct
subformulas, including $\Psi$ itself. Recursively,
\begin{align*}
 \Sub(\top)&=\{\top\},&\Sub(\pi_a)&=\{\pi_a\},\\
 \Sub(\neg\varphi)&=\{\neg\varphi\}\cup\Sub(\varphi),&
 \Sub(\Past\varphi)&=\{\Past\varphi\}\cup\Sub(\varphi),\\
 \Sub(\varphi\wedge\psi)
 &=\{\varphi\wedge\psi\}\cup\Sub(\varphi)\cup\Sub(\psi).
\end{align*}
The \emph{past subformulas} are those whose outermost operator is $\Past$:
$\PastSub(\Psi)\coloneqq\{\Past\varphi:\Past\varphi\in\Sub(\Psi)\}$.
\end{definition}

\begin{definition}[Past profile]
\label{def:past-profile}
For any formula $\Psi\in\LTL$, we define the past profile of a word $u\in\Sigma^*$ as the vector
\begin{equation}
 \prof_\Psi(u)\in\{0,1\}^{\PastSub(\Psi)},
 \qquad
 \prof_\Psi(u)(\Past\varphi)
 =\mathbf{1}\{u,|u|+1\models\Past\varphi\}.
 \label{eq:past-profile}
\end{equation}
It records which past subformulas hold after reading $u$, at its virtual
endpoint $|u|+1$.
\end{definition}

\begin{example}
For $\Psi=\Past(\pi_b\wedge\Past\pi_a)$, order the profile coordinates as
$(\Past\pi_a,\Psi)$. For distinct $a,b$, the profiles of
$\eps,a,ab$ are respectively $(0,0),(1,0),(1,1)$.
\end{example}

For a language $L\subseteq\SigmaStar$, we note the left residual at a prefix $u$ with respect to a language $L$ as
\[
 u^{-1}L\coloneqq\{v\in\SigmaStar:uv\in L\}.
\]

\begin{lemma}
\label{lem:profile-residual}
Let $u_1,u_2\in\SigmaStar$.  If
\[
  \prof_\Psi(u_1)=\prof_\Psi(u_2),
  \tag{PR}
\]
then,
\[
u_1^{-1}L(\Psi) = u_2^{-1}L(\Psi);
\]
equivalently for every $v\in\SigmaStar$,
\[
  u_1v\in L(\Psi)\quad\Longleftrightarrow\quad u_2v\in L(\Psi).
\]

\end{lemma}

\begin{proof}
Write $l_i=|u_i|$ and $m=|v|$. By structural induction on
$\varphi\in\Sub(\Psi)$, we show that, for $0\leq k\leq m$,
\begin{equation}
 u_1v,l_1+k+1\models\varphi
 \quad\Longleftrightarrow\quad
 u_2v,l_2+k+1\models\varphi.
 \label{eq:shifted-subformula-equivalence}
\end{equation}
First, the atomic case compares the same suffix symbol, or the two virtual
endpoints; hence the equivalence holds. The Boolean cases follow immediately.

Now consider $\varphi=\Past\theta$. For $i\in\{1,2\}$,
\[
 u_iv,l_i+k+1\models\Past\theta
 \iff \prof_\Psi(u_i)(\Past\theta)=1
 \ \lor\
 \exists\ 0\leq j<k:\ u_iv,l_i+j+1\models\theta.
\]
The first case says that $\theta$ already has a witness in $u_i$, which is
stored in the profile. The second says that it has a witness in the first $k$
positions of $v$. By (PR), the first cases agree; by the induction hypothesis,
so do the second. Thus the equivalence holds. Taking $\varphi=\Psi$ and $k=m$
proves the claim.
\end{proof}

\begin{corollary}
\label{cor:profiles-determine-minimal-states}
Let $\Psi \in \LTL$ and let $A_\Psi$ be its minimal DFA. If two prefixes
$u_1,u_2\in\SigmaStar$ have the same past profile, then they reach the same
state of $A_\Psi$.
\end{corollary}

\begin{proof}
Lemma~\ref{lem:profile-residual} gives us that for every suffix $v\in\SigmaStar$,
\[
  u_1v\in L(\Psi)\quad\Longleftrightarrow\quad u_2v\in L(\Psi).
\]
This is precisely the characterisation of
Proposition~\ref{prop:past-dfa-right-context}.
\end{proof}

\begin{definition}
\label{def:profile-order}
For $b,b'\in\{0,1\}^{\PastSub(\Psi)}$, write $b\preceq b'$ when
\[
  \forall\chi\in\PastSub(\Psi),\qquad b(\chi)\leq b'(\chi).
\]
Write $b\prec b'$ when $b\preceq b'$ and $b\neq b'$.
\end{definition}

\begin{lemma}
\label{lem:past-profile-monotonicity}
For all $u,v\in\Sigma^*$,
\[
  \prof_\Psi(u)\preceq\prof_\Psi(uv).
\]
\end{lemma}

\begin{proof}
Let $\Past\theta\in\PastSub(\Psi)$ and suppose that
$\prof_\Psi(u)(\Past\theta)=1$. By definition, there is a position
$m<|u|+1$ such that $u,m\models\theta$. By
Lemma~\ref{lem:ltl-prefix-invariance}, $uv,m\models\theta$; also
$m<|uv|+1$. Hence $uv,|uv|+1\models\Past\theta$, so
$\prof_\Psi(uv)(\Past\theta)=1$. This holds for every coordinate.
\end{proof}

\begin{lemma}
\label{lem:boolean-profile-chain}
Let $S$ be a finite set and let
\[
  b_0\prec b_1\prec\cdots\prec b_r
  \qquad\text{in }\{0,1\}^{S}
\]
under the coordinate-wise order. Then $r\leq |S|$.
\end{lemma}

\begin{proof}
For every $i$, let $S_i\subseteq S$ be the support of $b_i$, namely the set
of coordinates on which $b_i$ is $1$. The strict inequalities give
\[
  S_0\subsetneq S_1\subsetneq\cdots\subsetneq S_r\subseteq S.
\]
Every strict inclusion increases cardinality by at least one. Thus there are
at most $|S|$ strict inclusions, proving $r\leq |S|$.
\end{proof}

\begin{lemma}
\label{lem:strict-residual-path}
Let $u_0,\ldots,u_r\in\SigmaStar$ be a sequence of prefixes such that, for every
$0\leq i<r$, there is a word $x_i\in\SigmaStar$ satisfying
\[
  u_{i+1}=u_ix_i
  \qquad\text{and}\qquad
  u_i^{-1}L(\Psi)\neq u_{i+1}^{-1}L(\Psi).
\]
Then
\[
  r\leq \bigl|\PastSub(\Psi)\bigr|.
\]
\end{lemma}

\begin{proof}
By Lemma~\ref{lem:past-profile-monotonicity},
\[
  \prof_\Psi(u_i)\preceq\prof_\Psi(u_{i+1})
  \qquad (0\leq i<r).
\]
If two consecutive profiles were equal, Lemma~\ref{lem:profile-residual}
would imply $u_i^{-1}L(\Psi)=u_{i+1}^{-1}L(\Psi)$, contrary to assumption.
Hence each displayed relation is strict.
Lemma~\ref{lem:boolean-profile-chain}, applied with $S=\PastSub(\Psi)$,
gives the result.
\end{proof}

We finally combine the previous definitions and results for the following result. In particular, the bound $C_0$ from Eq.~\ref{eq:chain-bound-in-long-short-lemma} is an application of this theorem:

\begin{theorem}[Past subformulas bound automaton height]
\label{lem:minimal-poDFA-height}
For $\Psi\in\LTL$, the reduced DFA $\A_\Psi$ of $L(\Psi)$ satisfies
\[
\operatorname{ht}(\A_\Psi) \leq |\PastSub(\Psi)|.
\]
\end{theorem}
\begin{proof}
The bound follows from Lemma~\ref{lem:strict-residual-path}.
\end{proof}

We can directly convey the above result to the Mealy machine. Intuitively, our transformers from Appendix~\ref{sec:fpt-definition} are the Mealy machines made of the pointed formulas given in \ref{eq:tokenwise-transformer-formula}.

\begin{corollary}
For a finite family $\mathcal G$ of pointed $\LTL$ formulas, the reduced
Mealy machine $\mathcal M_{\mathcal G}$ emitting their truth values at
each token similarly satisfies
\[
 \operatorname{ht}(\mathcal M_{\mathcal G})
 \leq \left|\bigcup_{\varphi\in\mathcal G}\PastSub(\varphi)\right|.
\]
\end{corollary}

\begin{proof}
For the Mealy bound, store the shared past profile over
$\bigcup_{\varphi\in\mathcal G}\PastSub(\varphi)$.
Together with the next letter, it determines the current truth values
and the updated profile: the bit for $\Past\theta$ becomes its old value
OR the current truth of $\theta$. Thus these profiles form a Mealy
realisation. Each bit changes at most once, giving the stated height;
reduction cannot increase it.
\end{proof}

\newpage

\section{Length Generalization}
\label{app:length-generalization-proof}

For two $\mathcal{FPT}$s \edits{given by} $T_s,T_t$ sharing $\Sigma$, $\F$, and output dimension $\nu$,
the \emph{difference function} from Section~\ref{sec:len-gen-only} is
\begin{equation}
\label{app:def:difference-function}
 \Delta(w)\coloneqq
 \mathbf1\{\exists\,1\leq i\leq|w|:\ T_s(w)_{:,i}\neq T_t(w)_{:,i}\}.
\end{equation}

We give two height bounds for its reduced poDFA $\A_\Delta$.
We use the arithmetic assumptions of Appendix~\ref{sec:fpt-definition},
with $p=|\F|$ and $r_\F=K_\F+1$. 

\subsection{Bounding the height with past subformulas}
\label{sec:bounding-the-height}
We count the past subformulas needed to describe one transformer's outputs,
then combine the two families into a formula for the difference function.

\begin{lemma}[Simultaneous output representation]
\label{lem:simultaneous-output-formulas}
Under these arithmetic hypotheses, let
$T\in\FPT(L,D,H)$ have output dimension $\nu$, and put $D_h=D/H$.
There are formulas $\Phi^T_{d,f}\in\LTL$, for $1\leq d\leq\nu$ and
$f\in\F$, such that, for every $w\in\Sigma^*$ and $1\leq i\leq|w|$,
\[
 w,i\models\Phi^T_{d,f}
 \quad\Longleftrightarrow\quad T(w)_{d,i}=f.
\]
Their shared collection of distinct past subformulas satisfies
\begin{equation}
 \mathcal P(T)\coloneqq
 \bigcup_{d=1}^{\nu}\ \bigcup_{f\in\F}\PastSub(\Phi^T_{d,f}),
 \qquad
 |\mathcal P(T)|\leq B(T)\coloneqq r_\F L D\,p^{D/H+1}.
 \label{eq:shared-past-bound}
\end{equation}
\end{lemma}

\begin{proof}
Follow the layerwise construction of
Appendix~\ref{sec:transformers-as-logic-formulas}, now directly in $\LTL$.
At layer $0$, the embedding coordinates are described by
\[
 \Phi^{(0)}_{d,f}=\bigvee_{a\,:\,e(a)_d=f}\pi_a.
\]
Moreover, if a pointwise map $g$ sends input coordinates with values
$(f_1,\ldots,f_r)$ to $g(f_1,\ldots,f_r)$, then its output value $f$ is
described by the Boolean combination
\[
 \bigvee_{(f_1,\ldots,f_r)\,:\,g(f_1,\ldots,f_r)=f}
 \ \bigwedge_{j=1}^{r}\Phi_{j,f_j}.
\]
Thus projections, feed-forward maps, residuals, and the
output map add no past operators. Only attention does.

Fix a layer $\ell$ and a head $h$, suppressing these indices below. For each
key vector $\mathbf k\in\F^{D_h}$, coordinate $d\leq D_h$, and value
$f\in\F$, let $\theta_{\mathbf k,d,f}$ describe the unary condition
$K_{:,i}=\mathbf k$ and $V_{d,i}=f$. These pointwise predicates add no
past operators. Writing $\theta(y)$ for their unary $\PFO$ form, we describe
the clipped counts
\[
 c_{\mathbf k,d,f}(i)
 =\min\!\left\{r_\F,
   \bigl|\{j<i:K_{:,j}=\mathbf k,\ V_{d,j}=f\}\bigr|\right\}
\]
using the threshold tests
\begin{equation}
 \exists^{=r}y<x:\theta(y)\quad(0\leq r<r_\F),
 \qquad
 \exists^{\geq r_\F}y<x:\theta(y).
 \label{eq:shared-threshold-formulas}
\end{equation}
By Remark~\ref{rem:pfo-counting}, each successive at-least threshold adds
one outer $\Past$ to the conjunction of $\theta$ with the previous threshold.
Thresholds $1,\ldots,r_\F$ share their shorter expansions, and exact counts
are Boolean combinations of successive thresholds (threshold $0$ is $\top$).
Hence all tests for one $\theta$ add at most $r_\F$ distinct past subformulas.

Fix an output coordinate $d$. The current query $\mathbf q=Q_{:,i}$ and
the clipped-count vector
\[
 \mathbf c_i=\bigl(c_{\mathbf k,d,f}(i)\bigr)_
 {\mathbf k\in\F^{D_h},\;d\leq D_h,\;f\in\F}
\]
determine the head output. Indeed, these data identify the present keys,
the maximum score, and each stabilized exponential $a_{\mathbf k}$. If some
$a_{\mathbf k}=+\infty$, or if
\[
 c_{\mathbf k,d,f}(i)=r_\F\quad\text{and}\quad a_{\mathbf k}>0,
\]
then the denominator overflows and the head output is zero. Otherwise, every
type with a positive exponential has its exact multiplicity recorded, while
zero terms contribute nothing. Multiset summation then determines the
denominator, the weights, and the weighted value sum. Hence, for a
finite lookup function $A_d$,
\[
 \operatorname{Att}_{d,i}=A_d(\mathbf q,\mathbf c_i).
\]
Expanding this finite lookup into Boolean cases reuses only the threshold
tests above, and so introduces no further past subformulas.

There are
\[
 |\F^{D_h}|\,D_h\,|\F|=D_h p^{D_h+1}
\]
key--coordinate--value types per head. Thus, if $\mathcal P_\ell$ denotes the
past subformulas available after layer $\ell$, then
\[
 |\mathcal P_\ell\setminus\mathcal P_{\ell-1}|
 \leq r_\F H D_h p^{D_h+1}
 =r_\F D p^{D/H+1}.
\]
Summing this bound over the $L$ layers proves~\eqref{eq:shared-past-bound}.
All $\nu$ output coordinates and all $f\in\F$ reuse the same collection, so
the final output map adds no factor.
\end{proof}

\begin{remark}
We count distinct past subformulas, not occurrences or Boolean formula size.
A general $\exists y<x:\theta(x,y)$ need not translate into a single
$\Past$-subformula. Here all counted predicates depend only on $y$;
the current query is handled outside the past operators by Boolean cases.
\end{remark}

The predicates $\Phi^T_{d,f}$ describe the outputs of the Mealy machines
from Appendix~\ref{sec:transformer-mealy}. Comparing these predicates and
checking for a mismatch anywhere in the word gives the difference function.

\begin{proposition}[Difference function as a logic formula]
\label{prop:tokenwise-disagreement}
Let $T_s,T_t$ be student and teacher transformers satisfying the preceding
hypotheses, over the same alphabet and precision set, with the same output
dimension $\nu$. Using the formulas of
Lemma~\ref{lem:simultaneous-output-formulas}, define
\begin{align}
 \psi_\Delta
 &\coloneqq
 \bigvee_{d=1}^{\nu}\ \bigvee_{f\in\F}
 \bigl(\Phi^{T_s}_{d,f}\wedge\neg\Phi^{T_t}_{d,f}\bigr),
 \label{eq:tokenwise-disagreement}\\
 \Psi_\Delta&\coloneqq\Past\psi_\Delta.
 \label{eq:persistent-disagreement}
\end{align}
Then
\[
 w\models\Psi_\Delta\quad\Longleftrightarrow\quad\Delta(w)=1,
\]
and
\begin{equation}
 |\PastSub(\Psi_\Delta)|
 \leq B(T_s)+B(T_t)+1.
 \label{eq:disagreement-past-count}
\end{equation}
Moreover, $L(\Psi_\Delta)\Sigma^*=L(\Psi_\Delta)$.
\end{proposition}

\begin{proof}
At each token, $\psi_\Delta$ holds exactly when an output coordinate differs.
Thus $\Past\psi_\Delta$, evaluated at $|w|+1$, expresses $\Delta(w)=1$
(and is false on $\eps$). Boolean operations add no past subformulas:
\[
 \PastSub(\Psi_\Delta)
 \subseteq\mathcal P(T_s)\cup\mathcal P(T_t)
              \cup\{\Past\psi_\Delta\}.
\]
This proves the count. Causality preserves an existing mismatch under
right concatenation, proving persistence.
\end{proof}

\begin{corollary}
\label{cor:tokenwise-disagreement-height}
The reduced DFA $\A_\Delta$ of $L(\Psi_\Delta)$ is partially ordered and
satisfies
\[
 \operatorname{ht}(\A_\Delta)
 \leq B(T_s)+B(T_t)+1
 =r_\F\bigl(L_sD_s p^{D_s/H_s+1}+L_tD_t p^{D_t/H_t+1}\bigr)+1.
\]
If $\Delta$ is not identically zero, $\A_\Delta$ has a unique accepting
state, which is absorbing.
\end{corollary}

\begin{proof}
Apply Proposition~\ref{prop:past-language-podfa} and
Theorem~\ref{lem:minimal-poDFA-height} to $\Psi_\Delta$, then
use~\eqref{eq:disagreement-past-count}. Every accepted prefix has residual
$\Sigma^*$ by persistence, so all accepting prefixes reach the same
absorbing state by Proposition~\ref{prop:past-dfa-right-context}.
\end{proof}

\begin{remark}[Mealy-machine interpretation]
Let $\mathcal M_T$ be the reduced Mealy machine of $T$.
The shared-profile argument of Appendix~\ref{sec:past-profile-height} gives
$\operatorname{ht}(\mathcal M_T)\leq B(T)$. Running the two machines
in parallel until their outputs differ gives
\begin{equation}
 \operatorname{ht}(\A_\Delta)
 \leq\operatorname{ht}(\mathcal M_{T_s})
     +\operatorname{ht}(\mathcal M_{T_t})+1.
 \label{eq:mealy-disagreement-height}
\end{equation}
Each strict product transition advances at least one component; entering
the absorbing bad state adds at most one. Reduction cannot increase height.
\end{remark}

\subsection{Bounding the height with temporal profiles}
\label{sec:counting-temporal-profiles}

We bound the same difference-function poDFA $\A_\Delta$ using past
nesting depth. This removes the exponential dependence on $D/H$ and
gives a polynomial bound in depth, of degree $s=|\Sigma|$.

\begin{definition}[Past depth and temporal profiles]
\label{def:depth-temporal-profile}
The \emph{past depth} of a formula is defined by
\begin{align*}
 \operatorname{pd}(\top)&=\operatorname{pd}(\pi_a)=0, &
 \operatorname{pd}(\neg\varphi)&=\operatorname{pd}(\varphi),\\
 \operatorname{pd}(\Past\varphi)&=1+\operatorname{pd}(\varphi), &
 \operatorname{pd}(\varphi\wedge\psi)
 &=\max\{\operatorname{pd}(\varphi),\operatorname{pd}(\psi)\}.
\end{align*}
For $w\in\Sigma^*$, write $w_{<i}=w_1\cdots w_{i-1}$ and define
\begin{equation}
 U_0(w)=\star,\qquad
 U_{k+1}(w)=
 \bigl\{(w_i,U_k(w_{<i})):1\leq i\leq |w|\bigr\},
 \label{eq:depth-temporal-profile}
\end{equation}
where $\star$ is a fixed symbol. Unlike $\prof_\Psi$, the profile $U_k$
does not depend on a particular formula: it records which
(letter, lower-depth-profile) pairs have occurred.
\end{definition}

The profile records which letters have appeared in which past contexts.
For fixed $k$, there are finitely many profiles, and higher profiles
determine all lower ones by forgetting their deepest context layer.
For $k\geq1$,
\begin{equation}
 U_k(ua)=U_k(u)\cup\{(a,U_{k-1}(u))\}.
 \label{eq:temporal-profile-update}
\end{equation}
Thus profiles grow by inclusion and update deterministically from the next
letter. Structural induction shows that $(w_i,U_k(w_{<i}))$ determines
every formula of past depth at most $k$ at position $i$.
At the virtual endpoint, all letter atoms are false, so $U_k(w)$ alone
determines truth. Equal profiles therefore give equal residual languages.

\begin{lemma}[Counting strict profile changes]
\label{lem:temporal-profile-changes}
For every $k\geq0$, along the prefixes of any word, $U_k$ changes at most
\begin{equation}
 \binom{k+s}{s}-1
 \label{eq:temporal-profile-change-bound}
\end{equation}
times. Consequently, if $\operatorname{pd}(\Psi)\leq k$, the reduced DFA
of $L(\Psi)$ has height at most $\binom{k+s}{s}-1$.
\end{lemma}

\begin{proof}
For each $a\in\Sigma$, define its next-change rank after $u$ by
\[
 r_a(u)=\min\{j\geq1:U_j(ua)\neq U_j(u)\}.
\]
This is finite because $\Past^{|u|+1}\top$ distinguishes $u$ from $ua$
at the virtual endpoint. In particular,
\begin{equation}
 U_k(ua)\neq U_k(u)
 \quad\Longleftrightarrow\quad r_a(u)\leq k.
 \label{eq:temporal-rank-change}
\end{equation}

Initially, $r_a(\eps)=1$. If $t=r_a(u)$, the ranks update as
\begin{equation}
 r_a(ua)=t+1,\qquad
 r_b(ua)=\min\{r_b(u),t+1\}\quad(b\neq a).
 \label{eq:temporal-rank-update}
\end{equation}
A letter not yet seen has rank $1$. Otherwise, comparing with the prefix
before its last occurrence shows that its rank is one plus the first depth
at which the intervening profiles differ. Appending $a$ changes profiles
from depth $t$ onwards, which gives~\eqref{eq:temporal-rank-update}.

For fixed $k$, let
\[
 \rho_k(u)\coloneqq
 \operatorname{sort}\bigl(\min\{r_a(u),k+1\}:a\in\Sigma\bigr).
\]
If $t\leq k$, Equation~\eqref{eq:temporal-rank-update} replaces one entry
$t$ by $t+1$, leaves all entries below $t$ fixed, and leaves every other
altered entry at least $t+1$. Hence
\[
 U_k(ua)\neq U_k(u)
 \quad\Longrightarrow\quad
 \rho_k(u)<_{\mathrm{lex}}\rho_k(ua).
\]
If $t>k$, both $U_k$ and $\rho_k$ are unchanged. Since there are exactly
\[
 \binom{k+s}{s}
\]
nondecreasing $s$-tuples in $\{1,\ldots,k+1\}^s$, the profile changes at
most $\binom{k+s}{s}-1$ times. The profile automaton recognises $L(\Psi)$;
reduction cannot increase its height.
\end{proof}

This counts profile changes along a run, not the total number of profiles.

\begin{lemma}[Past depth of transformer outputs]
\label{lem:transformer-past-depth}
Under the stated assumptions, a transformer $T\in\FPT(L,D,H)$ with
output dimension $\nu$ admits coordinate/value formulas
$\Phi^T_{d,f}\in\LTL$ satisfying
\[
 w,i\models\Phi^T_{d,f}\iff T(w)_{d,i}=f,
 \qquad \operatorname{pd}(\Phi^T_{d,f})\leq r_\F L,
\]
for every $1\leq i\leq |w|$, $1\leq d\leq\nu$, and $f\in\F$.
\end{lemma}

\begin{proof}
In the construction of Lemma~\ref{lem:simultaneous-output-formulas},
embeddings have depth zero and pointwise maps do not increase depth.
At each layer, the unary key/value predicates inherit the previous
layer's depth. Counting their occurrences up to $r_\F$ adds at most
$r_\F$ nested past operators; attention then uses only Boolean
combinations of these tests and the current query predicates.
Each layer therefore increases depth by at most $r_\F$.
\end{proof}

\begin{corollary}
\label{cor:disagreement-height-temporal-depth}
Put $L_*=\max\{L_s,L_t\}$. The reduced poDFA $\A_\Delta$ of the
difference function satisfies
\begin{equation}
 \operatorname{ht}(\A_\Delta)
 \leq \binom{r_\F L_*+s}{s}.
 \label{eq:disagreement-height-temporal-depth}
\end{equation}
\end{corollary}

\begin{proof}
Set $k=r_\F L_*$. Using the output formulas of
Lemma~\ref{lem:transformer-past-depth} in~\eqref{eq:tokenwise-disagreement}
gives $\operatorname{pd}(\psi_\Delta)\leq k$.
Construct a DFA that stores $U_k(u)$ until the first mismatch.
On reading $a$, $(a,U_k(u))$ determines $\psi_\Delta$ at the new position:
if true, enter an accepting absorbing state; otherwise update $U_k$.
This recognises $L(\Psi_\Delta)$.

The profile grows by inclusion, so the DFA is partially ordered.
Lemma~\ref{lem:temporal-profile-changes} allows at most
$\binom{k+s}{s}-1$ profile changes, and entering the accepting state adds
at most one. Reduction preserves partial order and cannot increase height.
\end{proof}

For fixed $\Sigma$ and $\F$, this bound is polynomial in $L_*$, of degree
$s=|\Sigma|$, and independent of widths, head counts, and output dimension.
It bounds automaton height, not the size of the representing formulas.

\begin{remark}[Mealy-machine interpretation]
The same profile $U_{r_\F L}(u)$, together with the next letter, determines
all outputs of $T$. It therefore also realises $T$ as a Mealy machine, giving
\begin{equation}
 \operatorname{ht}(\mathcal M_T)
 \leq \binom{r_\F L+s}{s}-1.
 \label{eq:mealy-height-temporal-depth}
\end{equation}
Summing the two individual bounds via~\eqref{eq:mealy-disagreement-height}
also bounds $\operatorname{ht}(\A_\Delta)$.
The poDFA proof above is sharper because both transformers share
$U_{r_\F L_*}$ on their common input.
\end{remark}

\begin{remark}
    We remark that \citet{masopust2015complexity} showed a related bound (also in terms of a binomial coefficient) on the height of the minimal DFA representing piecewise testable languages; such languages are expressible by fixed-precision transformers. Our result differs in applying to difference automata for general fixed-precision transformers, and in expressing the bound in terms of the size parameters of those transformers.
\end{remark}

\subsection{Deterministic length generalization}
\label{sec:deterministic-length-generalization-appendix}
Either height bound from Corollary \ref{cor:tokenwise-disagreement-height}-\ref{cor:disagreement-height-temporal-depth}
gives a deterministic length cutoff: any nonzero
difference function between two transformers has its behaviour controlled by height of its poDFA.
Taking the smaller bound yields the following result.

\begin{lemma}[Deterministic length generalization]
\label{lem:deterministic-length-generalization-appendix}
Let $T_s\in\FPT(L_s,D_s,H_s)$ and $T_t\in\FPT(L_t,D_t,H_t)$ satisfy
the preceding hypotheses, with $L_s\geq L_t$ as in
Remark~\ref{remark:learning-assumptions}. Using the arithmetic constant
$r_\F$ from Appendix~\ref{sec:fpt-definition}, set
\begin{equation}\label{eq:deterministic-length-cutoffs-appendix}
  C_0 := r_\F \left[L_sD_s |\F|^{D_{h,s}+1}
                 + L_tD_t |\F|^{D_{h,t}+1}\right]+1, \qquad C_1 :=\binom{r_\F L_s+\S}{\S}.
\end{equation}
For $C:=\min\{C_0,C_1\}$, agreement on all words of length at most $C$
implies agreement on every word:
\[
    T_s|_{\leq C} \equiv T_t|_{\leq C} \quad \implies \quad T_s \equiv T_t.
\]
\end{lemma}

\begin{proof}
Corollaries~\ref{cor:tokenwise-disagreement-height}
and~\ref{cor:disagreement-height-temporal-depth} give
$\operatorname{ht}(\A_\Delta)\leq\min\{C_0,C_1\}=C$.
If $T_s\not\equiv T_t$, let $w$ be a shortest word accepted by
$\A_\Delta$. Its run contains no self-loop as deleting the corresponding
letter would give a shorter accepted word. Every letter therefore
induces a strict transition, so $|w|\leq\operatorname{ht}(\A_\Delta)\leq C$,
contradicting agreement on $\Sigma^{\leq C}$.
\end{proof}

\subsection{Probabilistic length generalization}
\label{sec:probabilistic-LG}

We now show that low error up to a sufficiently large cutoff persists at all
larger lengths.  The argument is automata-theoretic: a random input makes
uniformly positive progress through a partially ordered DFA until it reaches
an absorbing state.

Let $\A=(Q,\Sigma,\delta,q_0,B)$ be a partially ordered DFA, with bad
states $B\subseteq Q$, and let $s=|\Sigma|\geq2$.
Let $C\geq0$ be an integer upper bound on $\operatorname{ht}(\A)$.
For the difference function, take $\A=\A_\Delta$, with $B$ its accepting
states and $C=\min\{C_0,C_1\}$ from~\eqref{eq:deterministic-length-cutoffs-appendix}.
For an
exact length $m$, and for the distribution uniform over all words of length
at most $K$, respectively, write
\[
  p_m\coloneqq
  \Pr_{w\sim\operatorname{Unif}(\Sigma^m)}[\delta(q_0,w)\in B],
  \qquad
  \mathcal D_K\coloneqq\operatorname{Unif}(\Sigma^{\leq K}),
  \qquad
  e_K\coloneqq
  \Pr_{w\sim\mathcal D_K}[\delta(q_0,w)\in B].
\]
Thus $p_m$ is the bad-state probability at one fixed length, whereas $e_K$
is the bad-state probability when every word of length at most $K$ is equally
likely.  The mass that $\mathcal D_K$ assigns to length $K$ is
\[
  \lambda_K
  \coloneqq\Pr_{w\sim\mathcal D_K}[|w|=K]
  =\frac{s^K}{1+s+\cdots+s^K}
  =\frac{(s-1)s^K}{s^{K+1}-1}
  >\frac{s-1}{s}.
\]
Write $Q_{\mathrm{abs}}$ for the absorbing states of $\A$, namely those
satisfying $\delta(q,a)=q$ for every $a\in\Sigma$. All other states are
called transient.

\begin{lemma}[Absorption tail]
\label{app:lem:absorption-tail}
Let $0<\eps<1$.  If
\[
  K\geq s\left(
    C+\log\frac1\eps+
    \sqrt{\log^2\frac1\eps+2C\log\frac1\eps}
  \right),
  \tag{P}
\]
then a uniformly random word $W_K\in\Sigma^K$ ends in a transient state with
probability at most $\eps$.
\end{lemma}

\begin{proof}
If $C=0$, the initial state is absorbing and the claim is immediate.
At every transient state, at least one letter makes a strict transition.
Count a step as successful if it makes a strict transition or the run was
already absorbed. Each step has conditional success probability at least
$1/s$, so sequential coupling shows that the total success count
stochastically dominates $Y\sim\operatorname{Bin}(K,1/s)$.
If the run is still transient after $K$ steps, all successes were strict
transitions and there were fewer than $C$ of them. Hence
\[
  \Pr[\delta(q_0,W_K)\notin Q_{\mathrm{abs}}]\leq\Pr[Y<C].
\]

We use Chernoff's inequality to obtain an explicit sufficient cutoff.

Put $\mu=K/s$ and $\ell_\eps=\log(1/\eps)$.  The displayed condition implies
$\mu>C$.  With $\eta=1-C/\mu\in(0,1)$, the multiplicative Chernoff lower-tail
bound\footnote{See \href{https://en.wikipedia.org/wiki/Chernoff_bound\#Multiplicative_form_(relative_error)}{Wikipedia, ``Chernoff bound: Multiplicative form (relative error)''}.} gives
\[
  \Pr[Y<C]
  \leq \Pr[Y\leq(1-\eta)\mu]
  \leq \exp\left(-\frac{\mu\eta^2}{2}\right)
  =\exp\left(-\frac{(\mu-C)^2}{2\mu}\right).
\]
It remains to make the exponent at least $\ell_\eps$. This is equivalent to
\[
  \mu^2-2(C+\ell_\eps)\mu+C^2\geq0.
\]
The larger root of the left-hand quadratic is
\[
  C+\ell_\eps+\sqrt{\ell_\eps^2+2C\ell_\eps}.
\]
Condition \textup{(P)} is exactly the assertion that $\mu$ is at least this
root. Hence $(\mu-C)^2/(2\mu)\geq\ell_\eps$, and the preceding display is at most
$e^{-\ell_\eps}=\eps$.
\end{proof}

\begin{theorem}[Probabilistic length generalization]
\label{thm:probabilistic-length-generalization}
Let $0<\eps<1$. Suppose that $K$ satisfies \textup{(P)} and that
\[
  \Pr_{w\sim\operatorname{Unif}(\Sigma^{\leq K})}
     [\delta(q_0,w)\in B]\leq\eps.
\]
Then the probability of reaching a bad state remains small both at
every fixed length $m\geq K$ and over all words of length at most
$M\geq K$:
\[
 \Pr_{w\sim\operatorname{Unif}(\Sigma^m)}
    [\delta(q_0,w)\in B]<3\eps,
 \qquad
 \Pr_{w\sim\operatorname{Unif}(\Sigma^{\leq M})}
    [\delta(q_0,w)\in B]<3\eps.
\]
\end{theorem}

\begin{proof}
The bad absorbing states at time $K$ account for at most the bad probability
conditioned on length $K$.  Hence
\[
  \Pr[\delta(q_0,W_K)\in B\cap Q_{\mathrm{abs}}]
  \leq p_K
  \leq\frac{e_K}{\lambda_K}
  <\frac{s}{s-1}\eps
  \leq2\eps.
  \tag{A}
\]
By Lemma~\ref{app:lem:absorption-tail}, the run is still transient at time $K$
with probability at most $\eps$.  If a word of length $m>K$ ends in a
bad state, then at time $K$ its run was either transient or already in a bad
absorbing state.  Therefore, by \textup{(A)},
\[
 \begin{aligned}
  p_m
  &\leq p_K+\Pr[\delta(q_0,W_K)\notin Q_{\mathrm{abs}}]\\
  &<\underbrace{2\eps}_{\text{from }e_K\leq\eps}
    +\underbrace{\eps}_{\text{absorption tail}}
    =3\eps \qquad (m>K).
 \end{aligned}
\]
The same bound holds for $m=K$ by \textup{(A)}.  Finally, $e_M$ is a weighted
average of $e_K$ and the fixed-length errors $p_{K+1},\ldots,p_M$; each of
these terms is strictly below $3\eps$.  Thus $e_M<3\eps$.
\end{proof}

\begin{remark}
Keeping the alphabet dependence in the proof gives
$p_m,e_M<(2+1/(s-1))\eps$.
The upper-bound coefficient is $3$ for binary alphabets and approaches
$2$ as $s$ grows; this does not assert that the bound is attained.
\end{remark}

\subsection{Dependence of Chain Length on L and $\F$}
\label{app:ExamplesTightness}

We construct difference functions with unbounded poDFA height by varying
either precision or depth, with the other architectural sizes fixed.
Throughout, $\Sigma=\{a\}$, and outputs are scalar.
\paragraph{Implementation convention.}
Both constructions use one head and arithmetic satisfying
Appendix~\ref{sec:fpt-definition}, with attention sums rounded after summation.

To select specific positions with attention, we operate the following way.
Choose a negative score $z\in\F$ such that $\exp(z)$ rounds to zero. In particular, setting the convention that $\infty \cdot 0 = 0$,
$z$ can take the value $-\infty$. Selected positions then receive score $0$ while unselected positions receive
score $z$ and value $0$. Thus attention averages the selected values
unless the denominator overflows.

Both constructions keep track of a fixed number of Boolean values,
such as whether a position is selected. We represent each Boolean value using two coordinates: 
$(1,0)$ for true and $(0,1)$ for false. A separate coordinate stores the attention output.

\subsubsection{Necessity of precision dependence}

We create a family of framework where we increase precision while keeping both transformer architectures fixed.
For each sufficiently large integer $k$, use the precision set
\[
 \F_k=\{0,\pm\infty\}
       \cup\{\pm j,\pm 1/j:1\le j\le k\}
       \cup\mathcal C_k.
\]
Here $\mathcal C_k$ contains a bounded number of additional values needed
for the attention scores and the Boolean encoding above, including its
pointwise computations. Thus $p_k\coloneqq|\F_k|=4k+O(1)$ grows linearly
with $k$. Choose the arithmetic so that positive attention sums exceeding
$k$ overflow to $+\infty$. Increasing $k$ therefore allows uniform
attention to average more positions before overflowing.

\begin{proposition}[Precision can force linear height growth]
\label{prop:precision-height-lower-bound}
For the precision structures $\F_k$ above, there are pairs of transformers
with fixed depths $L_s=L_t=2$, fixed widths, and $H_s=H_t=1$, whose
reduced difference poDFAs satisfy
\[
 \operatorname{ht}(\mathcal A_{\Delta_k})=k+2=\Theta(p_k).
\]
Thus the height can grow linearly with precision-set size while the
architectures and alphabet remain fixed.
\end{proposition}

\begin{proof}
Both transformers first apply uniform attention with value $1$ at every
position. The resulting bit is
\[
 b(i)=\mathbf1\{2\le i\le k+1\}.
\]

The teacher retains this bit through its second block. The student's
second block attends only to earlier positions with $b(j)=1$, using
value $b(j)$, and takes the Boolean OR with the current bit. At most
$k$ positions are selected, so this attention never overflows. Consequently,
\[
 T_t(w)_i=\mathbf1\{2\le i\le k+1\},
 \qquad
 T_s(w)_i=\mathbf1\{i\ge2\},
 \qquad
 \Delta_k(w)=\mathbf1\{|w|\ge k+2\}.
\]
The reduced DFA of $\{w:|w|\ge m\}$ counts length up to $m$, then
remains in an accepting sink. Its $m+1$ states are distinct: counts
$r<t\le m$ are distinguished by a suffix of length $m-t$.
Its height is therefore $m$, giving $k+2$ here.
Only Boolean attention outputs occur in this construction, so the
encoding and feed-forward width are independent of $k$. Their pointwise
computations and the score maps require only a bounded number of
additional scalar values, as claimed for $\mathcal C_k$.
\end{proof}

Thus no height bound over \edits{arbitrary} precision structures can ignore its precision set cardinality $|\F|$.

\subsubsection{Necessity of linear depth dependence}\label{sec:app:tightness-depth}

\begin{proposition}
\label{prop:depth-height-lower-bound}
For one fixed precision structure $\F$, there are transformers $T_m$
of depth $L=2m$, with fixed width and one head, whose difference poDFAs
against an always-zero transformer have height $m=L/2$.
\end{proposition}

\begin{proof}
Fix a cutoff $k\ge2$ and \edits{an adequate arithmetic with the same overflow as in the previous proposition}. Its two-layer student computes
$\mathbf1\{i\ge2\}$\edits{. Taking the complement gives} a marker $u_1(i)=\mathbf1\{i=1\}$.

Two further layers move a unique marker forward by one position. Indeed, suppose
\[
u_j(i)=\mathbf1\{i=j\}.
\]
Thus, $u_j$ is $1$ only at position $j$.

The first added layer records whether the marker has appeared at the current
position or earlier:
\[
s_j(i)
=
u_j(i)\vee\mathbf1\{\text{an earlier marker exists}\}
=
\mathbf1\{i\ge j\}.
\]

The second layer attends uniformly to the earlier positions selected by
$s_j$. At position $i>j$, these are the positions
$j,j+1,\ldots,i-1$. But only position $j$ has value $u_j=1$; all other
selected positions have value $0$. Hence the attention output is
\[
r_j(i)=
\begin{cases}
1/(i-j),&1\le i-j\le k,\\
0,&\text{otherwise}.
\end{cases}
\]
In particular, $r_j(i)=1$ holds exactly when $i=j+1$. Setting
\[
u_{j+1}(i)=\mathbf1\{r_j(i)=1\}
\]
therefore creates a unique marker at position $j+1$. This shift uses fixed parameters and can be repeated for every $j$.

Starting from $u_1$, and applying $m-1$ shifts, uses $2m$ layers and gives
\[
T_m(w)_i=\mathbf1\{i=m\},
\qquad
\Delta_m(w)=\mathbf1\{|w|\ge m\}.
\]
\end{proof}

With $\F$, $\Sigma$, $D$, and $H$ fixed, this proposition proves the necessity in, at least, a
 linear dependence on depth in the bound on the height of a difference transformer.

\clearpage

\section{Data Complexity Bound}
\label{sec:app:data-complexity}

This appendix collects the statistical (in-distribution) generalization bound underlying Proposition~\ref{thm:data-complexity}. The core argument is imported from~\citet{rizvi2026narrow}, which studies the same student--teacher setting for single-output transformers; in order to keep the present paper self-contained, we restate the main result in our notation and sketch its proof (Appendix~\ref{sec:app:data-complexity:main}), pointing the reader to~\citet{rizvi2026narrow} for full details. We then describe the two modifications needed for our setting, \edits{stemming from the absence of attention output projection matrices in our transformers (Appendix~\ref{sec:app:data-complexity:proj-matrices})}  %
and \edits{the fact our transformers produce vectorial outputs} at every position (Appendix~\ref{sec:app:data-complexity:per-position}). \edits{Furthermore}, we contrast the resulting G\&C bound with the analogous bound for ERM (Appendix~\ref{sec:app:data-complexity:erm}) \edits{and discuss the statistical bounds on transformer architectures that include attention score rescaling or layer norm (Appendix~\ref{sec:app:normalization-constraints}).}

Throughout, we write $D_{h,s} := D_s/H_s$ and $D_{h,t} := D_t/H_t$ for the per-head dimensions, and $D_{\mathrm{ff},s}, D_{\mathrm{ff},t}$ for the hidden widths of the feed-forward blocks. All parameters take values in the finite set $\F$, so that a transformer in $\FPT(L, D, H)$ has finitely many configurations and G\&C samples uniformly among them.

\subsection{Main statement and proof sketch}\label{sec:app:data-complexity:main}

We say that a student $T_s$ is \emph{teacher-equivalent} (TE) to $T_t$ on $\mathcal{D}_K$ if $\Pr_{w \sim \mathcal{D}_K}[T_s(w) \neq T_t(w)] = 0$. Let $\tilde p$ denote the probability that a uniformly random $T_s \in \FPT(L_s, D_s, H_s)$ is TE. The following is the general form of the G\&C bound, \edits{as given by ~\citet{rizvi2026narrow}, where the considered transformer class $\widetilde \FPT$ uses output projection matrices after attention, and assumes only one output value per word instead of our vectorial outputs at every position}.

\begin{theorem}[Statistical generalization for G\&C; adapted from Theorem~B.1 and Corollary~B.1 of~\citet{rizvi2026narrow}]\label{thm:app:gc-general}
Fix $\eps \in (0,1)$ and $\delta \in (0, \tfrac{1}{5})$. Let $T_t \in \widetilde \FPT(L_t, D_t, H_t)$ be a teacher and $\mathcal{H} = \widetilde \FPT(L_s, D_s, H_s)$ a hypothesis class with
\[
L_s \geq L_t, \qquad D_s \geq D_t, \qquad H_s \geq H_t, \qquad D_{h,s} \geq D_{h,t}, \qquad D_{\mathrm{ff},s} \geq D_{\mathrm{ff},t}.
\]
Then
\[
\tilde p \;\geq\; |\F|^{-B},
\]
where
\begin{align*}
B \;:=\;& \edits{|\Sigma| D_s} \;+\; L_t\Big(   H_t\big(D_t D_{h,t} + 2 D_t D_{h,s}\big) \\
        &+ \ \ \ \ \ (H_s - H_t)\, D_t D_{h,s}
        + D_s D_t
        + D_s D_{\mathrm{ff},t} + D_t D_{\mathrm{ff},s} + D_s + D_{\mathrm{ff},s}
\Big) \\
&+\; (L_s - L_t)\big(H_s D_{h,s} + D_{\mathrm{ff},s} + 1\big) D_t .
\end{align*}
Consequently, if $T_s$ is obtained by the G\&C learner on an i.i.d.\ sample $S$ of $N$ strings from $\mathcal{D}_K$ with
\[
N \;\geq\; \frac{1}{\eps}\big(B \log|\F| + 3\log(2/\delta)\big),
\]
then, with probability at least $1 - \delta$ over $S$ and the draw of $T_s$,
\[
\Pr_{w \sim \mathcal{D}_K}\big[T_s(w) \neq T_t(w)\big] < \eps .
\]
One then has $B \sim \mathcal{O}\big(L_t D_t^2 + D_s D_t \edits{(L_s + L_t) + |\Sigma| D_s}\big)$ under the standard conventions $D = H D_h$ and $D_{\mathrm{ff}} = D$.
\end{theorem}

\edits{The output matrices term $L_t D_s D_t$ is removed from $B$ in Appendix~\ref{sec:app:data-complexity:proj-matrices}} and the readout term required for per-position outputs is added in Appendix~\ref{sec:app:data-complexity:per-position}; Proposition~\ref{thm:data-complexity} is the combination of Theorem~\ref{thm:app:gc-general} with these two modifications and with the ERM bound of Appendix~\ref{sec:app:data-complexity:erm}.

\begin{proof}[Proof sketch]
The proof has two ingredients.

\paragraph{From $\tilde p$ to a sample-complexity bound.}
The main lemma of~\citet{buzaglo_narrow_teachers_2024} states that, for any finite hypothesis class and any $0/1$ loss, if the probability of sampling a TE hypothesis under the prior is $\tilde p < \tfrac12$, then G\&C generalizes to error below $\eps$ with probability at least $1-\delta$ as soon as $N \geq (-\log\tilde p + 3\log(2/\delta))/\eps$. Hence it suffices to lower bound $\tilde p$, and $\tilde p \geq |\F|^{-B}$ yields the stated bound.

\paragraph{Lower bounding $\tilde p$ by constraint counting.}
We exhibit one explicit set of $B$ parameter values such that any student agreeing with them on those $B$ coordinates (and arbitrary elsewhere) computes exactly the function of $T_t$. Since parameters are drawn uniformly from $\F$, this event has probability exactly $|\F|^{-B}$, and $\tilde p$ is at least as large. The construction proceeds layer by layer and maintains the invariant that, at every position, the residual stream of the student has the form $[\lambda^\star, \mathbf{0}]$, where $\lambda^\star \in \F^{D_t}$ is the residual stream of the teacher and the remaining $D_s - D_t$ coordinates are zero.
\begin{itemize}
    \item \emph{Embedding.} \edits{The first embedding coordinates of every token are given by the teacher's embedding, and the} bottom $D_s - D_t$ coordinates are set to zero; this contributes \edits{$|\Sigma| D_s$} constraints.
    \item \emph{Teacher layers ($\ell \leq L_t$).} For each of the $H_t$ teacher heads, the top-left $D_t \times D_{h,t}$ blocks of the student's $\mathbf{W}_Q, \mathbf{W}_K, \mathbf{W}_V$ are set to the teacher's matrices, and the blocks of $\mathbf{W}_Q, \mathbf{W}_V$ acting on the extra head coordinates are set to zero; since the query is zero on these coordinates, $\mathbf{W}_K$ may be left free there. For the remaining $H_s - H_t$ heads only the top block of $\mathbf{W}_V$ is zeroed, so that they contribute nothing to the teacher coordinates. The output projection and the feed-forward block are constrained analogously: their teacher blocks are copied and the blocks that would leak into (or read from) the extra coordinates are zeroed, including the biases. This preserves the invariant and contributes the $L_t(\cdot)$ term of $B$; see Lemma~B.2 of~\citet{rizvi2026narrow} for the block matrices.
    \item \emph{Weak-identity layers ($L_t < \ell \leq L_s$).} The extra $L_s - L_t$ layers need only preserve the top $D_t$ coordinates, not the full residual. Setting the first $D_t$ columns of the attention output projection, of the second feed-forward matrix, and of the corresponding bias to zero suffices; the bottom coordinates may then take arbitrary values without affecting the teacher block. This contributes $(H_s D_{h,s} + D_{\mathrm{ff},s} + 1) D_t$ constraints per layer, which is the last term of $B$ (Theorem~B.1 of~\citet{rizvi2026narrow}).
\end{itemize}
Summing the three contributions gives $B$, and the asymptotic form follows by substituting $D = H D_h$ and $D_{\mathrm{ff}} = D$.
\end{proof}

\subsection{Absence of output projection matrices}\label{sec:app:data-complexity:proj-matrices}

Unlike the architecture in~\citet{rizvi2026narrow}, neither our teacher nor
our student has an attention output projection. In the first $L_t$ layers,
there is therefore no output matrix to copy or pad with zeros. This removes
$D_sD_t$ constraints per layer. 

For each of the remaining $L_s-L_t$ weak-identity layers, the source construction zeros
the part of the output projection that writes to the teacher coordinates.
We instead zero the rows of the value matrices $\mathbf{W}_V$ that produce these coordinates.
There are $D_t$ such rows in total, each with $D_s$ entries. The residual connection then
preserves the $D_t$ first dimensions, and the count of fixed parameters on this layers is exactly as before.

We hence have for this distinction, a reduced constraint count
\[
B_{\mathrm{no\text{-}proj}}=B-L_tD_sD_t.
\]

\subsection{Per-position outputs}\label{sec:app:data-complexity:per-position}

Theorem~\ref{thm:app:gc-general} is stated for a single output per string, as in~\citet{rizvi2026narrow}. transformers in $\FPT$ instead return a vector of dimension $\nu$ at every position through the pointwise affine output layer $\operatorname{O}$ (Definition~\ref{def:main-FPT}). We write $\mathbf{U} \in \F^{\nu \times D_s}$ and $\mathbf{u} \in \F^{\nu}$ for the student's output weights and bias, and $\mathbf{U}^\star \in \F^{\nu \times D_t}$, $\mathbf{u}^\star \in \F^{\nu}$ for the teacher's. Teacher-equivalence is understood in the strict per-string sense: $T_s$ is TE if, with probability one over $w \sim \mathcal{D}_K$, its output matches the teacher's at every position. The corresponding empirical loss on a sample $(w^{(n)})_{n \leq N}$ is
\[
\frac{1}{N} \sum_{n=1}^{N} \mathbb{I}\Big[\exists\, i \leq |w^{(n)}| \;:\; T_s(w^{(n)})_i \neq T_t(w^{(n)})_i\Big],
\]
which is $\{0,1\}$-valued per string and therefore fits the G\&C lemma of~\citet{buzaglo_narrow_teachers_2024} without change.

The constraint-counting argument carries over verbatim: the construction of Appendix~\ref{sec:app:data-complexity:main} preserves the teacher's residual stream in the top $D_t$ coordinates at every position, not only at the last one. The only additional constraints come from the output layer. Writing the student's final residual stream at position $i$ as $[\lambda_i^\star, \tilde\lambda_i]$, matching the teacher's output at every position requires the first $D_t$ columns of $\mathbf{U}$ to equal $\mathbf{U}^\star$, the remaining $D_s - D_t$ columns to be zero (so that $\tilde\lambda_i$ does not leak into the output), and $\mathbf{u} = \mathbf{u}^\star$. This fixes every entry of the output layer, adding exactly $(D_s + 1)\,\nu$ constraints, so that the constraint count becomes
\[
B_{\mathrm{seq}} \;=\; \edits{B_{\mathrm{no\text{-}proj}}} \;+\; (D_s + 1)\,\nu .
\]
Since $\nu$ is fixed by the task, this term is subleading whenever $\nu \leq D_t$, and the asymptotics of Theorem~\ref{thm:app:gc-general} are unchanged.

\subsection{Comparison to ERM}\label{sec:app:data-complexity:erm}

An analogous sample-complexity bound holds under Empirical Risk Minimization, obtained by a standard uniform-convergence argument over the finite hypothesis class $\FPT(L_s, D_s, H_s, \F)$ (Appendix~B.1 of~\citet{buzaglo_narrow_teachers_2024}). Counting parameters layer by layer under $D_s = H_s D_{h,s}$ and $D_{\mathrm{ff},s} = D_s$ (\edits{$\S D_s$ for the embedding layer,} $3 H_s D_s D_{h,s} = 3D_s^2$ for the query, key and value matrices, $2D_s^2$ for the two feed-forward matrices, and $2D_s$ for the biases), the total parameter count of the student, and hence the ERM complexity term, is
\[
B_{ERM} \;=\; \edits{\S D_s +} L_s\big(\edits{5} D_s^2 + 2 D_s\big) + (D_s + 1)\,\nu \;\sim\; \mathcal{O}\big(\edits{\S D_s \;+\;  }L_s D_s^2 \big),
\]
so that ERM generalizes to error below $\eps$ with probability at least $1 - \delta$ whenever
\[
N \;\geq\; \frac{1}{\eps}\big(B_{ERM} \log|\F| + \log(1/\delta)\big).
\]
Contrasting this with \edits{$B \approx \mathcal{O}(L_s D_s D_t)$} from Theorem~\ref{thm:app:gc-general}, ERM scales with the \emph{square} of the student width, whereas G\&C scales with the \emph{product} of the student and teacher widths; when the teacher is genuinely narrow ($D_t \ll D_s$), G\&C therefore saves a factor \edits{on the order of} $D_s/D_t$ over ERM. This narrow-teacher advantage is intrinsic to the volume-based argument underlying G\&C, and is not captured by capacity-based bounds \citep{edelman2022inductive, trauger2024sequence, yang2026tight}.

\subsection{Normalization within transformers in fixed precision}\label{sec:app:normalization-constraints}

Our transformer model is inspired by \citet{li_cotterell2025fpt}, but omits the attention-score factor $1/\sqrt{D_h}$ and post layer normalization~\citep{ba2016layernorm}. Layer normalization and attention-score scaling can be seamlessly included in the ERM analysis of Appendix~\ref{sec:app:data-complexity:erm}. The argument counts the parameters of a finite hypothesis class and does not require a teacher--student embedding. Under the same realizability assumption, it is enough to include the learned normalization parameters in this count. As their number is linear in $L_sD_s$, the leading order of the ERM bound $\mathcal{O}\big(L_s D_s^2 + \S D_s \big)$ is unchanged.

Accounting for this factor and the normalization in the constructive Guess \& Check argument of~\citet{rizvi2026narrow} would require extra technical conditions on the fixed-precision arithmetic; by omitting them, we keep the assumptions on $\F$ simple. The same counting method applies with normalization whenever the embedding preserves the teacher's computation exactly. But since both normalizations depend on the model dimensions, the embedding of the teacher in a larger student requires additional conditions on the fixed-precision arithmetic, as described below.

\paragraph{Attention-score scaling.}
The embedding pads the teacher's queries and keys with zeros to preserve their dot product. With attention-score scaling, the teacher and student divide this product by $\sqrt{D_{h,t}}$ and $\sqrt{D_{h,s}}$, respectively. In infinite-precision arithmetic, these factors can be matched by rescaling the copied query projection matrix:
\[
 W_Q^\star \longmapsto c W_Q^\star,
 \qquad c=\sqrt{D_{h,s}/D_{h,t}}.
\]
To use this rescaling in fixed precision, every entry of $cW_Q^\star$ must belong to $\F$. Representability of $c$ is needed, but is not enough, as rounding within the projections, dot products, and score division can break the equality. Additional  assumptions are hence required, ensuring that this rescaling preserves the rounded scores and controls the overflow and underflow. 

\paragraph{Layer normalization.}
For a vector $x$, layer normalization takes the form
\[
 \operatorname{LN}(x)
 =\gamma\odot\frac{x-\mu(x)\mathbf{1}}{\sqrt{v(x)+\epsilon_{\mathrm{LN}}}}+\beta,
\]
where $\mu(x)$ and $v(x)$ are the mean and variance across coordinates, $\gamma$ and $\beta$ are learned vectors, and $\epsilon_{\mathrm{LN}}$ is a stabilization constant. Absorbing this operation into the embedding requires accounting for the mean and variance of the added coordinates. Writing zero padding as $\bar x=[x;0]\in\mathbb{R}^{D_s}$ for $x\in\mathbb{R}^{D_t}$ and $a=D_t/D_s$, we have in infinite-precision arithmetic,
\[
 \mu(\bar x)=a\mu(x),
 \qquad
 v(\bar x)=a v(x)+a(1-a)\mu(x)^2.
\]
The change thus depends on the input $x$, and cannot generally be corrected by fixed $\gamma$ and $\beta$. More elaborate padding schemes would need to preserve the teacher's normalized coordinates through every sublayer. Any compensating parameters must also be representable in $\F$, and the correction must survive rounding for every residual vector the teacher can produce. Extra student layers pose a further problem: setting their attention and feed-forward to zero still leaves layer normalization acting on the residual. Thus, including layer normalization in the construction requires conditions on both the encoding and the arithmetic.
Given the technical complexity, we prefer to formalize transformers without layer normalization for our results.

\begin{remark}
	These constraints concern our particular construction (Appendix~\ref{sec:app:data-complexity:main}), rather than G\&C itself. For ERM, the statistical bound applies directly to the normalized architecture, and Theorem~\ref{thm:main-theorem} and Theorem~\ref{thm:joint-bound-michaels-attempt-main-paper} hold with the updated ERM statistical bound.
\end{remark}

\newpage

\section{Main theorem and Tradeoff}

\subsection{Main theorem}\label{sec:app:full-pipeline}

The main result given by Theorem~\ref{thm:main-theorem} simply combines two results: sufficiently many samples give
small error on $\mathcal D_K$, and a sufficiently large cutoff $K$ extends
that guarantee to every $\mathcal D_M$ with $M\ge K$.

\begin{theorem}[]\label{appendix:thm:main-theorem}

Fix $\eps\in(0,1)$, $\delta\in(0,\tfrac15)$, and $|\Sigma|\ge2$.
Let the student $T_s \in \mathcal{FPT}(L_s, D_s, H_s)$ and teacher
$T_t \in \mathcal{FPT}(L_t, D_t, H_t)$ satisfy the learning framework of
Remark~\ref{remark:learning-assumptions}, and the arithmetic assumptions of
Appendix~\ref{sec:fpt-definition}.
Let $B$ be an admissible G\&C complexity term from Proposition~\ref{thm:data-complexity} and set
$C=\min\{C_0,C_1\}$ as in~\eqref{eq:chain-bound-in-long-short-lemma}.

    If $T_s$ is given by the Guess \& Check procedure, agreeing with  $T_t$ on $N$ i.i.d. samples from $\mathcal{D}_K$, with
    \begin{equation}
        N \geq \frac{1}{\eps}\big(B \log|\F| + 3 \log(2/\delta)\big)
    \end{equation}

    and 
        \begin{equation}
		    K\ge \S\left( C+\log(1/\eps) + \sqrt{\log^2(1/\eps)+2C\log(1/\eps)} \right),
    \end{equation}
    then, with probability at least $1-\delta$ over the training sample
    and the G\&C draw, we have
    \begin{equation}
        \forall M \geq K, \qquad \operatorname{Pr}_{\mathcal{D}_{M}}\left[T_s(w) \neq T_t(w)\right] \edits{< 3} \eps.
    \end{equation}
\end{theorem}

\begin{proof}
By Proposition~\ref{thm:data-complexity}, the condition on $N$ gives,
with probability at least $1-\delta$,
\[
 e_K:=\Pr_{w\sim\mathcal D_K}[T_s(w)\neq T_t(w)]<\eps.
\]
Fix any student in this event. Its difference function's poDFA has height
at most $C$ by Corollaries~\ref{cor:tokenwise-disagreement-height}
and~\ref{cor:disagreement-height-temporal-depth}.
The condition on $K$ is exactly the cutoff required by
Theorem~\ref{thm:probabilistic-length-generalization}, which gives
$e_M<3\eps$ for every $M\ge K$.
This implication is deterministic once the student is fixed, so no
additional failure probability or union bound over $M$ is needed.
\end{proof}

\paragraph{Hidden dependencies.}
Under the \edits{$D_{\mathrm{ff}} = D$} feed-forward width convention, a sufficient sample budget has order
\[
 \mathcal O\!\left(
 \frac{
 [L_tD_t^2+D_sD_t\edits{(L_s+L_t)}+|\Sigma|D_s+D_s\nu]\log|\F|
 +\log(1/\delta)}
 {\eps}\right).
\]
This retains the embedding and output-layer terms; the implicit constants
are those of the data-complexity bound.
Additionally, for $\ell=\log(1/\eps)$,
$\sqrt{\ell^2+2C\ell}\le C+\ell$, so the sufficient length cutoff from Theorem~\ref{thm:probabilistic-length-generalization}  is
$\mathcal O(|\Sigma|(C+\log(1/\eps)))$.

For ERM, the same proof applies after replacing the sample condition by
the ERM condition in Proposition~\ref{thm:data-complexity}; the length
cutoff is unchanged.

\subsection{Multiplicative upper bound}\label{sec:multiplicative-upper-bound}

We prove Theorem~4.2 from Section~4.3. For this, we importantly show that, whichever absorbing state the automaton eventually reaches, it has probability at least \(a_C(K)\) of reaching that state within \(K\) steps. This lets us control eventual disagreement
using the disagreement already visible at length $K$.

\begin{theorem}[Joint generalization bound in sample size and length]
\label{thm:app:multiplicative-upper-bound}
Fix $s=|\Sigma|\geq2$, $\delta\in(0,\tfrac15)$, and integers $N\geq1$ and
$K\geq0$. Work under the learning and arithmetic assumptions of
Remark~\ref{remark:learning-assumptions} and Appendix~\ref{sec:fpt-definition}.
Let $T_t$ be the teacher and let $T_s$ be obtained by G\&C from a labeled
sample $S\sim\mathcal D_K^N$, where
$\mathcal D_K=\operatorname{Unif}(\Sigma^{\leq K})$.
Let $B$ be an admissible G\&C complexity term from
Proposition~\ref{thm:data-complexity}, let $C=\min\{C_0,C_1\}$ be the uniform height bound
in~\eqref{eq:deterministic-length-cutoffs-appendix}, and set
\[
 a_C(K):=\Pr[\operatorname{Bin}(K,1/s)\geq C],\qquad
 \rho:=\frac{s}{s-1},\qquad
 B_\delta:=B\log|\F|+3\log(2/\delta).
\]
If $a_C(K)>0$, then, with probability at least $1-\delta$ over $S$ and
the G\&C draw, for every integer $M\geq K$,
\begin{equation}
\label{eq:app:multiplicative-upper-bound}
 \Pr_{w\sim\mathcal D_M}[T_s(w)\neq T_t(w)]
 \leq\frac{\rho B_\delta}{N\,a_C(K)}.
\end{equation}
For ERM, the same statement holds with an admissible ERM complexity
term $B=B_{ERM}$ from Proposition~\ref{thm:data-complexity}; the probability
is then only over $S$.
\end{theorem}

We first establish the bound on the absorption time of an infinite run of letters. 

\begin{lemma}
\label{lem:app:conditional-absorption}
Let $\A$ be a reduced poDFA of height at most $C$, driven by an infinite stream
of i.i.d. uniform letters from $\Sigma$. Write $Q_n$ for its state after
$n$ letters, $T$ for its first hitting time of its absorbing state, and
$Z=Q_T$ the absorbing state eventually reached. Then $T<\infty$ and for every absorbing state
$z$,
\[
 \Pr(T\leq K\mid Z=z)\geq a_C(K).
\]
\end{lemma}

\begin{proof}
At each transient state, the probability of leaving is at least $1/s$,
so absorption must occurs within $C$ strict transitions.
A self-loop leaves the state unchanged, and future letters are independent;
thus conditioning on $Z=z$ does not change the self-loop probability.
Even under this conditioning, each step before absorption therefore makes
progress with probability at least $1/s$. By coupling with independent
Bernoulli$(1/s)$ trials, absorption is no slower than obtaining $C$
successes, giving
\[
 \Pr(T\leq K\mid Z=z)\geq\Pr[\operatorname{Bin}(K,1/s)\geq C]=a_C(K).
\]
\end{proof}

\begin{proof}[Proof of the theorem]
For a fixed student, write
\[
 p_m:=\Pr_{w\sim\operatorname{Unif}(\Sigma^{=m})}[T_s(w)\neq T_t(w)],
 \qquad
 e_m:=\Pr_{w\sim\mathcal D_m}[T_s(w)\neq T_t(w)].
\]
We have by convention that disagreement means a mismatch at any output position, so the difference poDFA $\A_\Delta$ has at most a unique absorbing
accepting state $\bot$, with its height bounded by $C$.

Run $\A_\Delta$ on an i.i.d. uniform stream, with $T,Z$ as in the lemma,
and set $q_\infty:=\Pr(Z=\bot)$ the probability that the ultimate absorbing state is the disagreement one. Since $\bot$ is absorbing,
\[
 p_K=\Pr(T\leq K,\ Z=\bot)
 =\edits{\Pr(T\leq K\mid Z=\bot) \cdot q_\infty
 \geq a_C(K) \cdot q_\infty.}
\]
Moreover, persistence of disagreement makes $p_m$ nondecreasing, so for every $M$,
\[
 e_M=\frac{\sum_{m=0}^{M}s^m p_m}{\sum_{m=0}^{M}s^m}
 \leq \frac{\sum_{m=0}^{M}s^m p_M}{\sum_{m=0}^{M}s^m}
 = p_M\leq q_\infty.
\]
To connect the exact-length quantity $p_K$ to the training distribution,
observe that
\[
 \lambda_K:=\Pr_{w\sim\mathcal D_K}(|w|=K)
 =\frac{s^K}{\sum_{m=0}^{K}s^m}
 =\frac{(s-1)s^K}{s^{K+1}-1}>\frac{s-1}{s}=\rho^{-1}.
\]
Therefore $e_K\geq\lambda_K p_K\geq\lambda_K a_C(K)q_\infty$, which implies
\begin{equation}
\label{eq:app:multiplicative-transfer}
 e_M\leq q_\infty\leq\frac{e_K}{\lambda_K a_C(K)}
 \leq\frac{\rho e_K}{a_C(K)}.
\end{equation}

Finally, Proposition~\ref{thm:data-complexity} gives
$e_K\leq B_\delta/N$ with probability at least $1-\delta$, as we can apply the proposition
 with $\eps=B_\delta/N$.
Substituting into~\eqref{eq:app:multiplicative-transfer} proves the result
simultaneously for every $M\geq K$.
The ERM version is obtained identically, thanks to $\log(1/\delta)\leq3\log(2/\delta)$.
\end{proof}

The factor $a_C(K)$ measures how much of the eventual disagreement is
already visible at length $K$ and the factor $\rho$ accounts for drawing
training words uniformly from all words of length at most $K$.

\subsection{Tradeoff between Length and Data Complexity}\label{sec:app:tradeoff}
\label{sec:app:token-budget-allocation}

We give the calculations behind the tradeoff discussed in Section~\ref{sec:tradeoff} and its
implications for the allocation of a given token-budget. Keep the notation of
Theorem~\ref{thm:app:multiplicative-upper-bound}, with $s=|\Sigma|\geq2$,
$C\geq1$, and $\rho=s/(s-1)$. Fix a target error $\eps\in(0,1)$ and
confidence $1-\delta$. The theorem gives the sufficient sample size
\begin{equation}\label{eq:prefix-target-error-frontier}
 N_{\rm req}(K):=\frac{N_0}{a_C(K)},
 \qquad N_0:=\frac{\rho B_\delta}{\eps},
 \qquad K\geq C.
\end{equation}
Here $N_0$ is the limit sample size as $K\to\infty$, that would still satisfy the target error $\eps$. Halving the target error doubles
the $N_0$'s and so it does for $N_{\rm req}(K)$, at every length, without changing its dependence on $K$.

\paragraph{Three length regimes.}
The following table summarizes the sample sizes certified by the bound and already discussed in Section~\ref{sec:tradeoff}.
\begin{proposition}[Sample size as a function of length]
\label{prop:simple-length-sample-rates}
\begin{equation}\label{eq:prefix-exponential-to-constant}
\begin{array}{c|c}
 \text{training length} & \text{sufficient number of samples} \\ \hline
 K<C & \text{no finite $N$ is certified} \\[2pt]
 K=C & N_{\rm req}(C)=s^C N_0 \\[2pt]
 K\geq sC & N_0\leq N_{\rm req}(K)\leq2N_0
\end{array}
\end{equation}
Thus increasing length from $C$ to $sC$ reduces the sufficient sample
size by a factor of at least $s^C/2$.
\end{proposition}

\begin{proof}
For $K<C$, $a_C(K)=0$, while $a_C(C)=s^{-C}$.
At $K=sC$, the binomial has integer mean $C$, which is also a median
\citep[Theorem~1]{kaas1970median}, so $a_C(sC)\geq1/2$.
The last row follows because $a_C(K)$ increases to one.
\end{proof}

\paragraph{How to invest a fixed token budget.}
Write $\operatorname{Budget}=NK$ for the token budget, an upper bound
on the total number of tokens when sampling from $\mathcal D_K$.
We keep $K$ integral and ignore rounding of $N$. Substituting
$N=\operatorname{Budget}/K$ into the theorem gives
\begin{equation}\label{eq:app:tradeoff-budget-bound}
 \Pr_{w\sim\mathcal D_M}[T_s(w)\neq T_t(w)]
 \leq\frac{\rho B_\delta}{\operatorname{Budget}}\,
       \frac{K}{a_C(K)}
 \qquad(M\geq K),
\end{equation}
with probability at least $1-\delta$. Thus the best allocation must maximize
$a_C(K)/K$.

\begin{corollary}[Optimal token allocation]
\label{cor:app:tradeoff-token-allocation}\label{cor:optimal-token-allocation}
Suppose $\operatorname{Budget}\geq sC$. The smallest value of the upper
bound in~\eqref{eq:app:tradeoff-budget-bound} is
\[
 \eps^\star(\operatorname{Budget})
 =\frac{\rho B_\delta}{\operatorname{Budget}}
   \min_{C\leq K\leq\operatorname{Budget}}\frac{K}{a_C(K)},
\]
attained by
\[
 K^*\in\arg\max_{C\leq K\leq\operatorname{Budget}}\frac{a_C(K)}{K},
 \qquad N^*=\frac{\operatorname{Budget}}{K^*}.
\]
Moreover,
\begin{equation}\label{eq:app:tradeoff-budget-optimum}
 \frac{\rho B_\delta sC}{\operatorname{Budget}}
 \leq\eps^\star(\operatorname{Budget})
 \leq\frac{2\rho B_\delta sC}{\operatorname{Budget}}.
\end{equation}
In particular, choosing $K=sC$ gives an error bound within a factor
two of the optimum.
\end{corollary}

\begin{proof}
Only $K/a_C(K)$ depends on the allocation. For
$Y\sim\operatorname{Bin}(K,1/s)$, Markov's inequality gives
\[
 a_C(K)=\Pr(Y\geq C)\leq\frac{\mathbb E[Y]}{C}=\frac{K}{sC},
 \qquad\text{hence}\qquad \frac{K}{a_C(K)}\geq sC.
\]
Conversely, $a_C(sC)\geq1/2$, so choosing $K=sC$ gives
$K/a_C(K)\leq2sC$.
\end{proof}

There is no closed form formula to find the optimum, but the exact optimizer can be found by checking only
$C\leq K\leq\min\{\operatorname{Budget},2sC\}$:
every $K>2sC$ has $K/a_C(K)\geq K>2sC$ and cannot improve on $K=sC$.
At $K=C$, the same ratio is $Cs^C$. For $C\geq2$, choosing $K=sC$
therefore improves the error bound at a fixed token budget by a factor
of at least $s^{C-1}/2$.

\paragraph{An explicit choice of $K$.}\label{app:explicit-choice-tradeoff}
For $C\geq2$, define
\begin{equation}\label{eq:app:tradeoff-explicit-length}
 \widehat K
 :=\left\lceil s\left(
 C+\log C+\sqrt{\log^2 C+2C\log C}
 \right)\right\rceil.
\end{equation}
The Chernoff calculation in the proof of
Lemma~\ref{app:lem:absorption-tail}, with tail probability $1/C$,
gives $a_C(\widehat K)\geq1-1/C$. If the budget allows
$\widehat K\leq\operatorname{Budget}$, let
$\widehat\eps(\operatorname{Budget})$ be the bound obtained with
$K=\widehat K$ and $N=\operatorname{Budget}/\widehat K$. Then
\[
 1\leq
 \frac{\widehat\eps(\operatorname{Budget})}
      {\eps^\star(\operatorname{Budget})}
 \leq\frac{\widehat K}{sC(1-1/C)}
 =1+\mathcal O\!\left(\sqrt{\frac{\log C}{C}}\right).
\]
Hence $\widehat K=sC+\mathcal O(s\sqrt{C\log C})$ is asymptotically
optimal as $C\to\infty$. The simpler choice $K=sC$ gives the uniform
factor-two guarantee above. This \edits{supports} the choice of the bound on $K$ in the joint statistical and length generalization main result given by Theorem~\ref{thm:main-theorem}.

\paragraph{Least token budget for a target error.}
The same calculation can be read in reverse. By
\eqref{eq:prefix-target-error-frontier}, the least token budget certified
for error $\eps$ is
\[
 \operatorname{Budget}_{\min}(\eps)
 =N_0 \cdot \min_{K\geq C}\frac{K}{a_C(K)},
 \qquad
 N_0sC\leq\operatorname{Budget}_{\min}(\eps)\leq2N_0sC.
\]
Choose a minimizing length $K$ and take $N=N_0/a_C(K)$.
Thus both problems---minimizing error at a fixed budget and minimising
budget at a fixed error---reduce to the same ratio $K/a_C(K)$ optimization.

\subsection{Adaptive main theorem: uniform training length}
\label{app:main-thm-uniform-training-sampling}
We adapt Theorem~\ref{thm:main-theorem} to the distribution that first
chooses a length uniformly, then draws a word of that length:
\[
 \widetilde{\mathcal D}_K:\quad
 k\sim\operatorname{Unif}\{0,\ldots,K\},\qquad
 w\sim\operatorname{Unif}(\Sigma^k).
\]
The sample-size condition is unchanged. For the length transfer, we use
the fact that at least half of the sampled lengths are at least $K/2$.

\begin{corollary}[Main theorem for uniform training lengths]
\label{cor:probabilistic-length-generalization-uniform-by-length}
Fix $\eps\in(0,\tfrac12)$, $\delta\in(0,\tfrac15)$, and $s=|\Sigma|\geq2$.
Use the learning and arithmetic assumptions of
Theorem~\ref{appendix:thm:main-theorem}, with the same complexity term $B$
and height bound $C$. Let $T_s$ be obtained by G\&C from a labeled sample
$S\sim\widetilde{\mathcal D}_K^N$. If
\[
 N\geq\frac{B\log|\F|+3\log(2/\delta)}{\eps}
\]
and
\[
 \left\lceil\frac K2\right\rceil
 \geq s\left(
 C+\log\frac1{2\eps}
 +\sqrt{\log^2\frac1{2\eps}+2C\log\frac1{2\eps}}
 \right),
\]
then, with probability at least $1-\delta$ over $S$ and the G\&C draw,
\[
 \forall M\geq K,\qquad
 \Pr_{w\sim\widetilde{\mathcal D}_M}[T_s(w)\neq T_t(w)]\leq4\eps.
\]
\end{corollary}

\begin{proof}
The statistical argument of Proposition~\ref{thm:data-complexity} applies
to any input distribution, so the condition on $N$ gives error at most
$\eps$ on $\widetilde{\mathcal D}_K$ with probability at least $1-\delta$.
Fix a student in this event and write
\[
 p_m:=\Pr_{w\sim\operatorname{Unif}(\Sigma^m)}[T_s(w)\neq T_t(w)],
 \qquad
 \widetilde e_m:=\frac1{m+1}\sum_{k=0}^m p_k.
\]
Since $I=\{\lceil K/2\rceil,\ldots,K\}$ contains at least half the lengths,
\[
 \frac1{|I|}\sum_{k\in I}p_k
 \leq\frac{K+1}{|I|}\widetilde e_K\leq2\eps.
\]
Hence some $\widetilde K\in I$ satisfies $p_{\widetilde K}\leq2\eps$.
The difference poDFA has height at most $C$. By the condition on $K$
and Lemma~\ref{app:lem:absorption-tail}, it is still transient after
$\widetilde K$ letters with probability at most $2\eps$.
Any disagreement at a later length either already occurred by
$\widetilde K$, or the run was still transient then. Thus $p_m\leq4\eps$
for every $m>\widetilde K$. Averaging gives, for every $M\geq K$,
\[
 \widetilde e_M
 \leq\frac{K+1}{M+1}\eps+\frac{M-K}{M+1}\,4\eps
 \leq4\eps.
\]
This holds for all $M$ on the same statistical event.
\end{proof}

As usual, for ERM, use its sample-size condition from
Proposition~\ref{thm:data-complexity}. Here, the see that the bound given on $K$ is roughly two times larger. However the rates do not change, and
as in Theorem~\ref{thm:main-theorem}, the sufficient length remains
$\mathcal O(s(C+\log(1/\eps)))$.

\begin{remark}
The multiplicative tradeoff bound from Theorem~\ref{thm:joint-bound-michaels-attempt-main-paper} also extends to
$\widetilde{\mathcal D}_K$. At least half of its mass lies on lengths
between $\lceil K/2\rceil$ and $K$. Thus, the same averaging argument
finds a length in this interval with error at most twice the average
error. Repeating the proof of Theorem~\ref{thm:app:multiplicative-upper-bound}
therefore gives a bound of order
\[
 \frac{2B_\delta}{N\,a_C(\lceil K/2\rceil)}
\]
on every later error. The factor $2$ replaces the factor $\rho$ as
it represents the mass of the upper half of the length range.
\end{remark}

\newpage

\section{Extensions}

\subsection{Local Attention}
\label{appendix:local-attention}

We keep fixed precision and strict causal masking. Each head is either
global or local; a local head at position $i$ attends only to
$\max(1,i-W),\ldots,i-1$, for a common window size $W\geq1$.
Let $T_s,T_t$ have depths $L_s,L_t$ and the same output dimension. Set
\[
 s=|\Sigma|\geq2,\qquad R=W\max(L_s,L_t).
\]
We consider disagreement at ordinary token positions,
\[
 \Delta(w)=\mathbf{1}\{\exists i\leq |w|:\ T_s(w)_i\neq T_t(w)_i\}.
\]
By causality, $\Delta(u)=1$ implies $\Delta(uv)=1$. We do not include
disagreement evaluated only at EOS.

\begin{definition}[Cache--poDFA cascade]
\label{def:local-cascade}
Let $\operatorname{suffix}_R(u)$ retain the last $\min(R,|u|)$ symbols of $u$.
A cache--poDFA cascade is a DFA with states $(c,q)\in\Sigma^{\leq R}\times Q$,
initial state $(\eps,q_0)$, and transitions
\begin{equation}
 \delta((c,q),a)=
 \bigl(\operatorname{suffix}_R(ca),f(q,c,a)\bigr),
 \qquad q\preceq f(q,c,a),
 \label{eq:local-cascade-transition}
\end{equation}
where $(Q,\preceq)$ is a finite partially ordered set. Its accepting states
are those with $b(q)=1$, for a flag $b:Q\to\{0,1\}$ satisfying
$b(q)\leq b(f(q,c,a))$. The ordered component is a poDFA over
$\Sigma^{\leq R}\times\Sigma$: it reads the cache together with the next letter.
Let $C_{\mathrm{loc}}\geq1$ be an integer bound on the number of strict
changes of $q$ along a run, including a change of the disagreement flag.
\end{definition}

\begin{proposition}[Representation of disagreement]
\label{prop:local-cascade-representation}
The language $\{w:\Delta(w)=1\}$ is recognized by a cache--poDFA cascade
of radius $R$. In particular, the full DFA need not be partially ordered.
\end{proposition}

\begin{proof}
Write $Y^j\varphi$ to mean that $\varphi$ held $j$ positions earlier,
with value false if that position does not exist. The
finite-precision compilation of~\citet{li2026characterizing} gives a
pointed formula $\psi_\Delta\in\mathrm{LTL}[\Past,Y]$ for disagreement at
the current position. A local head uses the current representation and
the previous $W$ representations, adding shifts $Y^j$ with $j\leq W$.
Global heads introduce only Boolean
operations and $\Past$, as in Appendix~\ref{sec:transformers-as-logic-formulas}.
Thus the total shift along any formula branch is at most $R$.

Move every shift inward using
\[
 Y\Past\varphi\equiv\Past Y\varphi,\qquad
 Y(\varphi\wedge\theta)\equiv Y\varphi\wedge Y\theta,\qquad
 Y\neg\varphi\equiv Y\top\wedge\neg Y\varphi.
\]
The resulting formula, still denoted $\psi_\Delta$, uses local atoms
$Y^j\pi_a,Y^j\top$ with $0\leq j\leq R$, Boolean operations, and $\Past$.
The cache $c$ and current letter $a$ determine all local atoms;
$|c|$ also records which earlier positions exist.

Suppose $t$ symbols have been read, and the next symbol is $a=w_{t+1}$.
For each past subformula $\chi=\Past\theta$, store
\[
 z_\chi=\mathbf{1}\{\exists j\leq t:\ w,j\models\theta\}.
\]
This is the truth of $\chi$ before reading $a$.
The cache is $c=\operatorname{suffix}_R(w_1\cdots w_t)$.

For any normalized subformula $\varphi$, define
$\operatorname{val}_\varphi(z,c,a)$ as its truth value at the current position.
With $m=|c|$ and positions in $ca$ numbered from $1$ to $m+1$,
compute it recursively as follows:
\begin{align*}
 \operatorname{val}_{Y^j\pi_d}(z,c,a)
 &=
 \begin{cases}
  \mathbf{1}\{(ca)_{m+1-j}=d\},&j\leq m,\\
  0,&j>m,
 \end{cases}\\
 \operatorname{val}_{Y^j\top}(z,c,a)
 &=\mathbf{1}\{j\leq m\},\\
 \operatorname{val}_{\Past\eta}(z,c,a)
 &=z_{\Past\eta},\\
 \operatorname{val}_{\neg\varphi}(z,c,a)
 &=1-\operatorname{val}_\varphi(z,c,a),\\
 \operatorname{val}_{\varphi\wedge\eta}(z,c,a)
 &=\operatorname{val}_\varphi(z,c,a)\wedge
   \operatorname{val}_\eta(z,c,a).
\end{align*}
Here $d\in\Sigma$, $0\leq j\leq R$, and $Y^0$ is the identity.

The state update is
\[
 z'_\chi=z_\chi\vee\operatorname{val}_\theta(z,c,a),
 \qquad
 b'=b\vee\operatorname{val}_{\psi_\Delta}(z,c,a).
\]
These updates record whether $\theta$, or a disagreement, has occurred.
Both use the old profile $z$, so past operators exclude the current position.

Starting from $z=b=0$, induction gives the stated meaning of $z$ and
$b=\Delta(w_1\cdots w_t)$. Each bit changes only from $0$ to $1$,
so $q=(z,b)$ is partially ordered.
Let $G$ count the distinct past subformulas before moving the shifts inward.
Each gives at most $R+1$ shifted versions; adding the disagreement flag gives
\begin{equation}
 C_{\mathrm{loc}}=(R+1)G+1.
 \label{eq:local-height}
\end{equation}
\end{proof}

\begin{remark}[Scope of the height bound]
This construction extends the \edits{bound on the count of the past subformulas from}
Appendix~\ref{sec:bounding-the-height} \edits{which stems the threshold $C_0$ in Eq.~\ref{eq:chain-bound-in-long-short-lemma}. Each shifted past test changes from false to true at most once, so counting
these tests bounds the ordered component's height. In particular, setting the cache window $W$ to 0 retrieves the global attention bound. However, we cannot} extend the temporal-profile bound of Appendix~\ref{sec:counting-temporal-profiles} \edits{with this method. Indeed,} temporal profiles over
$\Sigma$ do not record the recent suffix needed by local attention.
Adding the cache changes the alphabet, so the original profile bound
no longer applies.
\end{remark}

\begin{definition}[Settled states and absorption]
\label{def:local-settled}
A reachable state $(c,q)$ is \emph{settled} if either $b(q)=1$
(disagreement has occurred), or every continuation keeps $b=0$
(disagreement is impossible). Otherwise it is \emph{transient}.
Merge the first kind into a bad sink and the second into a safe sink.
Both sets are closed under transitions, so this preserves the language.
The merged states are absorbing; let $\tau$ be the first time the run
reaches either of them.
\end{definition}

Both transformers share the same suffix cache, which depends only on the input.
The cache may keep changing after the outcome is settled.
If the ordered component can never advance, the state is settled.

\begin{lemma}[Progress in a fresh block]
\label{lem:local-block-progress}
From every transient state, some word of length at most $B=R+1$
strictly advances the ordered component. Consequently, for independent
uniform input symbols,
\begin{equation}
 \Pr[\text{strict transition in the poDFA in the next }B\text{ steps}
       \mid\text{history}]\geq \alpha,\qquad
 \alpha=\S^{-B}.
 \label{eq:local-block-progress}
\end{equation}
\end{lemma}

\begin{proof}
A transient state admits a continuation causing disagreement.
Take a word $v$ ending at the first strict advance of the ordered component.
If $|v|\leq B$, use $v$. Otherwise use its final $B$ symbols. Either this
shorter word advances the ordered component within $R$ steps, or those
steps restore the cache just before the last symbol of $v$, while $q$
stays fixed. The last symbol then causes the same advance.
This word has length $r\leq B$ and occurs at the start of the next block
with probability $s^{-r}\geq s^{-B}$, given any history.
\end{proof}

\begin{theorem}[Absorption tail with local attention]
\label{thm:local-absorption-tail}
For an i.i.d.\ uniform input stream, every integer $K\geq0$ satisfies
\begin{equation}
 \Pr[\tau>K]\leq
 \Pr\!\left[
   \operatorname{Bin}\!\left(\left\lfloor\frac K B\right\rfloor,\alpha\right)
   <C_{\mathrm{loc}}
 \right],
 \qquad B=R+1,\quad\alpha=s^{-B}.
 \label{eq:local-absorption-tail}
\end{equation}
In particular, for $0<\eps<1$, write
$\ell_\eps=\log(1/\eps)$. The
right-hand side is at most $\eps$ whenever
\begin{equation}
 K\geq B\left\lceil
   s^B\left(C_{\mathrm{loc}}+\ell_\eps+
       \sqrt{\ell_\eps^2+2C_{\mathrm{loc}}\ell_\eps}\right)
 \right\rceil.
 \label{eq:local-cutoff}
\end{equation}
This gives the sufficient scale
$K=O\bigl((R+1)s^{R+1}(C_{\mathrm{loc}}+\log(1/\eps))\bigr)$.
\end{theorem}

\begin{proof}
Split the stream into blocks of length $B$. Count a block as successful
if it contains an ordered advance or a settled state.
By Lemma~\ref{lem:local-block-progress}, its success probability is at
least $\alpha$, given all previous blocks. Coupling with independent
Bernoulli$(\alpha)$ trials shows that the number of successes in $n$
blocks stochastically dominates $\operatorname{Bin}(n,\alpha)$.

If $\tau>nB$, there are fewer than $C_{\mathrm{loc}}$ successful blocks:
each contains a strict advance, and $C_{\mathrm{loc}}$ advances would
leave the state settled. Taking $n=\lfloor K/B\rfloor$ proves
Equation~\ref{eq:local-absorption-tail}. For
$\mu=\lfloor K/B\rfloor\alpha>C_{\mathrm{loc}}$, the Chernoff bound gives
\[
 \Pr[\operatorname{Bin}(\lfloor K/B\rfloor,\alpha)<C_{\mathrm{loc}}]
 \leq
 \exp\!\left(-\frac{(\mu-C_{\mathrm{loc}})^2}{2\mu}\right).
\]
Condition~\eqref{eq:local-cutoff} makes this at most $\eps$.
\end{proof}

\begin{corollary}[Length generalization]
\label{cor:local-length-generalization}
Let $p_k=\Pr_{w\sim\operatorname{Unif}(\Sigma^k)}[\Delta(w)=1]$ and
$e_K=\Pr_{w\sim\operatorname{Unif}(\Sigma^{\leq K})}[\Delta(w)=1]$.
For every $m\geq K$,
\[
 p_m\leq p_K+\Pr[\tau>K].
\]
Thus, if $K$ satisfies~\eqref{eq:local-cutoff}, then $p_K\leq\eps$
implies $p_m\leq2\eps$ for all $m\geq K$. Likewise,
$e_K\leq\eps$ implies $p_m<3\eps$ and $e_M<3\eps$
for all $m,M\geq K$.
\end{corollary}

\begin{proof}
Couple different lengths as prefixes of one infinite stream. If a
disagreement occurs after time $K$ but not by time $K$, the state at
time $K$ must be transient. This proves the first inequality.
For bounded-length sampling,
\[
 p_K\leq e_K/\lambda_K,\qquad
 \lambda_K=\frac{s^K}{1+s+\cdots+s^K}>\frac{s-1}{s}\geq\frac12.
\]
The tail bound gives $p_m<3\eps$; averaging the errors at lengths
above $K$ together with $e_K$ gives the same bound for $e_M$.
\end{proof}

A uniform bound on $C_{\mathrm{loc}}$ lets us combine this corollary with
statistical generalization, keeping the same confidence.
Local attention replaces single-step progress by blocks of length $R+1$,
with progress probability at least $s^{-(R+1)}$.

\subsection{Generalization with Nonuniform Distribution over Vocabulary}
\label{sec:nonuniform-distribution}

Our main length generalization bound is stated for a uniform distribution over the vocabulary $\Sigma$ for simplicity. In fact, the proof yields a generalization to a nonuniform distribution $P$ over $\Sigma$. 
In fact, even dependencies between successive symbols are allowed; in this case we can consider distributions $P$ over semi-infinite strings $X \in \Sigma^\omega$\footnote{I.e., $X$ has the form $X_1 X_2 X_3 \dots$.}.
We only need to assume that every poDFA over $\Sigma$ reaches an absorbing state in finite time almost surely on $X \sim P$; this is satisfied, e.g., if for each $i > 0$ and each $\sigma \in \Sigma$:
\begin{equation}
    P(X_i=\sigma | X_{1\dots i-1} = \xi) \geq \epsilon_P
\end{equation}
for every $\xi \in \Sigma^{i-1}$, for some $\epsilon_P > 0$.

\begin{definition}
Let $\mathcal{A}$ be a poDFA, and $P$ be a distribution over semi-infinite strings in $\Sigma^\omega$. 
Define the following random variables over semi-infinite strings $X \sim P$:
\begin{enumerate}
\item $T \in \mathbb{Z}$: time of reaching an absorbing state %

\item $Z$: the absorbing state eventually reached
\end{enumerate}
Then define
\begin{equation}
    a_{\mathcal{A}}(K) := \min_{z \in Q_{absorb}} \operatorname{Pr}[T \leq K | Z=z]
\end{equation}
where $Q_{absorb}$ denotes the set of absorbing states of $\mathcal{A}$; we restrict to states where $\operatorname{Pr}[Z=z] > 0$ .

\end{definition}

To keep the same length distribution as $\mathcal D_K$, set $s=|\Sigma|\geq2$
and define $\mathcal D_K^P$ by drawing $X\sim P$ and an independent length $L$ with
\[
 \Pr(L=\ell)=\frac{s^\ell}{\sum_{j=0}^K s^j},
 \qquad 0\leq\ell\leq K,
 \qquad w=X_{1:L}.
\]
When $P$ is the i.i.d.\ uniform source, $\mathcal D_K^P=\mathcal D_K$.

Then the proof of our joint generalization bound (Theorem~\ref{thm:joint-bound-michaels-attempt-main-paper}) carries through for $\mathcal D_K^P$, replacing the quantity $a_C(K)$ used in that theorem with a quantity defined from $a_{\mathcal{A}}(K)$:
\begin{theorem}[Joint Generalization Bound in Sample Size and  Length, Version of Theorem~\ref{thm:joint-bound-michaels-attempt-main-paper} for General Input Distribution]\label{thm:joint-bound-michaels-attempt-arbitrary-P}
Use the same hypotheses as Theorem~\ref{thm:app:multiplicative-upper-bound},
with teacher $T_t$ and hypothesis class $\mathcal H$, and define
\begin{equation}\label{eq:a-k-for-general-hyp-class}
    a(K) := \inf_{T_s \in \mathcal{H}} a_{\mathcal{A}(\Delta(T_t, T_s))}(K)
\end{equation}
where $\mathcal{A}(\Delta(T_t, T_s))$ represents the difference automaton for $T_t$ and $T_s$.\footnote{As in Corollary~\ref{cor:tokenwise-disagreement-height} and Proposition~\ref{prop:tokenwise-disagreement}.}

Let $T_s$ be trained on $N$ i.i.d.\ samples from $\mathcal D_K^P$,
with $a(K)>0$.
Then, for every $M\geq K$, the error rate on $\mathcal D_M^P$ satisfies

    \begin{equation}\label{eq:multiplicative-bound-general-P}
 \Pr_{w\sim\mathcal D_M^P}[T_s(w)\neq T_t(w)]
 \leq \frac{\rho B_\delta}{N\cdot a(K)},
\end{equation}

with probability at least $1-\delta$, where $\rho:=\frac{s}{s-1}$ and $B_\delta:=B\log|\F|+3\log(2/\delta)$.
\end{theorem}

The statement above assumes G\&C; as usual, it also holds for ERM
with $B=B_{ERM}$ from Proposition~\ref{thm:data-complexity}.

\begin{proof}
For a fixed student, write
$e_m=\Pr_{w\sim\mathcal D_m^P}[T_s(w)\neq T_t(w)]$, and let $p_K$ be
the error probability on the exact prefix $X_{1:K}$.
Let $q$ be the probability of ever disagreeing on $X\sim P$.
Since disagreement is absorbing, the definition of $a(K)$ gives
$p_K\geq a(K)q$, while $e_M\leq q$ for every $M$.
The mass at length $K$ is
\[
 \lambda_K=\frac{s^K}{\sum_{j=0}^K s^j}>\frac{s-1}{s}=\rho^{-1}.
\]
Thus $e_K\geq\lambda_Kp_K\geq\lambda_Ka(K)q$, and
\[
 e_M\leq q\leq\frac{e_K}{\lambda_Ka(K)}
 \leq\frac{\rho e_K}{a(K)}.
\]
The statistical argument of Proposition~\ref{thm:data-complexity} applies
to $\mathcal D_K^P$ and gives $e_K\leq B_\delta/N$ with probability at
least $1-\delta$ (trivially if $B_\delta/N\geq1$).
Substitution proves the claim, with the same ERM adaptation as in
Theorem~\ref{thm:app:multiplicative-upper-bound}.
\end{proof}

\subsection{Lower Bound on Tradeoff}

We show a \emph{lower bound} also exhibiting a \emph{multiplicative interaction} between $N$ and $a(K)$.
Our lower bound on the tradeoff matches the more general version allowing nonuniform input distributions (Appendix~\ref{sec:nonuniform-distribution}).

\begin{theorem}\label{thm:lower-bound-on-tradeoff}
Let $C, d \in \mathbb{N}$.
    There is a family $\mathcal{F}_{C,d}$ of target functions $f$ representable by fixed-precision transformers, such that, \edits{for} every $0 < \epsilon \leq \frac{1}{4}$, there is an input distribution $P$ \edits{satisfying}:

Consider a learner with hypothesis class $\mathcal{H} := \mathcal{F}_{C,d}$ producing a hypothesis $T_s \in \mathcal{H}$ from $N$ i.i.d. samples at length $K$ (each sample the length-$K$ prefix of a sample from $P$, labeled with the output provided by the target $f$).

Then, define the risk $R_{N,K}$ by averaging the test loss on very long test strings:
    \begin{equation*}
     \lim\sup_{{K_{test}} \rightarrow \infty}\operatorname{Pr}_{w \sim P}(T_s(w|_{1\dots K_{test}}) \neq f(w|_{1\dots K_{test}}))
     \end{equation*}
across targets $f \in \mathcal{F}_{C,d}$ and size-$N$ i.i.d. datasets.

Then, for all $N$ and all $K \geq C+2$
    \begin{equation*}
     R_{N,K} \leq \epsilon 
   \ \ \ \   \Rightarrow \ \ \ \ 
     \epsilon \geq \frac{\log_2 |\mathcal{F}_{C,d}|}{8 N a(K)} 
    \end{equation*}
    where $a(K)$ is as defined in Eq.~\ref{eq:a-k-for-general-hyp-class}.

\end{theorem}

Note here that, in the setting of the function classes $\mathcal{F}_{C,d}$ built here, this bound closely resembles Theorem~\ref{thm:joint-bound-michaels-attempt-arbitrary-P}: the numerator is the negative log probability that a uniform sample from $\mathcal{H}$ matches any given ground truth $f$ -- which is essentially the statistical complexity term $B$ arising in the G\&C bound (which is the negative log probability that a random sample from the hypothesis class matches the target, as per discussion in Appendix~\ref{sec:app:data-complexity}) -- whereas the  denominator is proportional to the product of $N$ and $a(K)$, as in Theorem~\ref{thm:joint-bound-michaels-attempt-arbitrary-P}.
Thus, any bound applying at the generality of (i) general hypothesis classes $\mathcal{H} \subseteq \FPT(L, D, H)$, (ii) general distributions $P$ over $\Sigma^\omega$ needs to have essentially this form, as we have  both upper and lower bounds on $\epsilon$ of this form up to proportionality factors.

\begin{proof}

We define $P$ over $\Sigma^\omega$ so that successive symbols are always sampled independently; we thus specify the probability at the level of individual symbols (i.e., specify $P$ as a probability over $\Sigma$).
Define
\begin{equation}
    \Sigma := \{0,1\} \times [d] \times \{0,1\} %
\end{equation}
with elements $(g,j,p)$, with
\begin{equation}
    P((g,j,p)) = P_1(g) P_2(j)  P_3(p) %
\end{equation}

where
\begin{align*}
    P_1(g = 1) = \eta \ \ \ \ \ \ \ \ \ \ \ \ \ \     P_3(p = 1) = r
\end{align*}
and $P_2$ is uniform.

We parameterize the family by (i) the parameter $d \in \mathbb{N}$, (ii) the parameter $C \in \mathbb{N}$, (iii) a bit vector $v \in \{0,1\}^d$, and (iv) a tilt  $\eta > 0$, (v) a tilt $r \in (0,1)$.
Here, (i)--(iii) jointly define $f$; (iv) -- (v) define the distribution $P$.
We later choose $\eta$ based on $\epsilon$, as $\eta := 4\epsilon$.

Given a target $v = (v_1, \dots, v_d) \in \{0,1\}^d$, we first define a poDFA as follows.
The state set is 
\begin{equation}
    \mathcal{Q} := \{q_0, \bot\} \cup \{q_{j,i} : 0 \leq i \leq C; 1 \leq j \leq d\}. %
\end{equation}
\begin{enumerate}
\item Start in $q_0$.
\item Consider the first token, of the form $(g,j,p)$
\begin{enumerate}
\item If $g \neq 1 $, go to $\bot$. 
\item Else, $j$ indexes a chain to go into. That is, go to $q_{j,0}$.
\end{enumerate}

\item Once we are in such a state
\begin{equation}
    \delta(q_{j,i}, (g,j',p)) := \begin{cases}
        q_{j,i+1} & \text{if }i < C, p = 1 \\
        q_{j,i} & \text{else}
    \end{cases}
\end{equation}
That is, each branch is $q_{j,0} \rightarrow q_{j,1} \rightarrow \dots \rightarrow q_{j,C}$ where only $p=1$ advances the chain

Also, set $\delta(\bot, \sigma) \equiv \bot$ for every $\sigma \in \Sigma$.

\item Every state $q_0$, $\bot$, and $q_{j,i}$ with $i<C$ outputs $0$. Endpoint $q_{j,C}$ outputs an indicator for $v_j$; that is for every $\sigma \in \Sigma$, 
\begin{equation}
    \mathfrak{o}(q_{j,C}, \sigma) = v_j
\end{equation}
\end{enumerate}

This is represented by a family of $\mathcal{FPT}$s with $L=\Theta(C)$, \edits{$D=\Theta(d)$}, $H=1$.

Now for any two distinct elements of $\mathcal{F}_{C,d}$,
we note that the difference automaton $\mathcal{A}_\Delta$
has a very similar structure: it consists of chains of
length $C$ again, and the symbols have the same effect of advancing in a chain as they have in the poDFAs given above.

Define 
\begin{equation}
    a_C(t) := \operatorname{Pr}\left[\operatorname{Bin}(t,r) \geq C\right]
\end{equation}

This in fact matches the definition of $a(K)$ from Eq.~\ref{eq:a-k-for-general-hyp-class}.

A single example reveals $v_j$ with probability
\begin{equation}
    q_j(K) = \frac{\eta a_C(K-2)}{d}
\end{equation}
After $N$ independent examples, coordinate $j$ remains ``unrevealed'' with prob
\begin{equation}
    \left(1-\frac{\eta a_C(K-2)}{d}\right)^N
\end{equation}
Consider test length $m$.
For a given coordinate $j$, the chance of it becoming relevant on a test sample here is
\begin{equation}
    \frac{\eta a_C(m-2)}{d}
\end{equation}
The test error rate is:
\begin{equation}
    \text{TestErrorRate} \geq \frac{\eta a_C(m-2)}{2} \left(1-\frac{\eta a_C(K-2)}{d}\right)^N
\end{equation}

Setting $m \rightarrow \infty$:
\begin{equation}
    \text{TestErrorRate} \geq \frac{\eta}{2} \left(1- \frac{\eta a_C(K-2)}{d}\right)^N
\end{equation}
Take $\eta = 4\epsilon$:

\begin{equation}
    \text{TestErrorRate} \geq 2\epsilon \left(1- \frac{4\epsilon a_C(K-2)}{d}\right)^N
\end{equation}
If $\text{TestErrorRate} \leq \epsilon$, then, dividing both sides by $\epsilon$, and applying Bernoulli's Inequality
\begin{align*}
    \frac{1}{2} \geq  \left(1- \frac{4\epsilon a_C(K-2)}{d}\right)^N %
    \geq  1 - \frac{4N\epsilon a_C(K-2)}{d} 
\end{align*}
Hence,
\begin{align*}
    \frac{4 N\epsilon a_C(K-2)}{d}  \geq \frac{1}{2}
\end{align*}
Rearranging, and by monotonicity of $a_C$,
\begin{equation}
     \epsilon \geq \frac{d}{8 N a_C(K)}
\end{equation}
Now, by construction, $d = \log_2 |\mathcal{F}_{C,d}|$. %
This concludes the result.
\end{proof}

\newpage

\section{Experiments}
\label{app:experiments}

We use three controlled synthetic task families to test the qualitative
predictions of our theory. We choose these families for variety in teacher
depth, width, number of heads, alphabet size, and associated automaton
height. The induction-head task also tests local positional information
provided by a local attention head. \edits{We use standard FP32 arithmetic and test 
whether the theoretical predictions carry over to real-world numerical implementations}.
First, the logarithmic--linear tradeoff between
training length $K$ and sample size $N$ suggests that an additive increase in
$K$ can offset a multiplicative reduction in $N$: increasing training length
can therefore be more effective than adding examples. Second, \edits{larger teacher and} larger student
depth $L$, width $D$, or alphabet size $|\Sigma|$ can increase the sufficient
requirements on $N$ and/or $K$, suggesting possible performance degradation
at fixed training resources. These experiments also have an exploratory
role: the two bounds on the automaton height $C$ in
Lemma~\ref{thm:deterministic-length-generalization} yield different dependencies
on these quantities, and worst-case guarantees need not describe their
empirical effects.

\subsection{Task families}
\label{sec:task-families}

\subsubsection{\texorpdfstring{$L_k$}{Lk}: Alternating languages }

Let the alphabet be $\Sigma_L=\{a,b\}$.  The language $L_k$ consists of the
words that begin with $a$ and that contain at least $k$ alternating constant blocks of a's or b's.
Equivalently, writing the unique block decomposition of a nonempty word as
$w=r_1^+\cdots r_m^+$, with $r_i^+ = r_ir_i^*$ and where the consecutive $r_i, r_{i+1}\in\{a,b\}$ are distinct.
\begin{equation}
 L_k:=\left\{w = r_1^+\cdots r_m^+\in\{a,b\}^*:
 r_1=a\ \text{and}\ m\geq k\right\}.
 \label{eq:experimental-Lk}
\end{equation}

Thus $L_3$ accepts $\mathtt{aaabbbbbba}$ and $\mathtt{ababa}$, but rejects
$\mathtt{aaaabb}$ and $\mathtt{bbbaabb}$. These tasks are inspired by the
alternating-block depth hierarchy of~\citet{yang2025kneedeep}; here we ask
for at least $k$ blocks rather than exactly $k$. \edits{\cite{yang2025kneedeep} prove that, in a fixed-precision paradigm with infinite precision inside attention, the minimal number of  transformer layers required for to express the alternating languages $L_k$ is $k-1$.} We use this family to study increasing sequential depth, and hence the dependence on teacher depth $L_t$, while keeping the alphabet fixed.

\paragraph{Specifics.}\label{app:exp-alternating-lang-specifics}
The residual stream gives access to the current
token, which is included when counting blocks. The target at position $i$
is $1$ exactly when the prefix $w_1\cdots w_i$ belongs to $L_k$; the EOS
target is the membership label of the full word. Evaluation includes every
ordinary position and EOS. Each new letter differs from the previous one with probability $0.05$, rather
than $0.5$, and words start with $a$ with probability $0.95$. This makes
alternations rarer and avoids early target saturation. The distribution still
fits our \edits{uniform sampling} framework: one can \edits{equivalently consider} an alphabet of $20$ symbols, \edits{extending $\{a,b\}$ with $18$ padding symbols}. Additionally, Appendix~\ref{sec:nonuniform-distribution}
discusses the non-uniform extension for the the sample--length tradeoff multiplicative bound in
Equation~\ref{eq:multiplicative-bound}.

\paragraph{Example.} For $L_3$, the target becomes $1$ when the third
block begins:
\[
\begin{array}{r|cccccccccc}
w & \mathtt{a} & \mathtt{a} & \mathtt{a} & \mathtt{b} & \mathtt{b} & \mathtt{b} & \mathtt{a} & \mathtt{a} & \mathtt{a} & \mathtt{EOS} \\
y & 0 & 0 & 0 & 0 & 0 & 0 & 1 & 1 & 1 & 1
\end{array}
\]

\subsubsection{\texorpdfstring{$A_k$}{Ak}: Alphabet languages }

Here $\Sigma_k=\{s_1,\ldots,s_k\}$.  These Alphabet languages ask
whether every letter has occurred at least once in the word:
\begin{equation}
 A_k:=\left\{w\in\Sigma_k^*:
 \forall j\in[k],\ \exists i\in[|w|]\ \text{such that}\ w_i=s_j\right\}.
 \label{eq:experimental-Ak}
\end{equation}
For example, $A_5$ accepts precisely the words over
$\{a,b,c,d,e\}$ in which no letter is missing. We assume a reference teacher
construction, where one attention head monitors each letter, giving $H_t=k$.
With this hypothesis, increasing $k$ therefore varies $H_t$, the corresponding residual width
$D_t$, and $|\Sigma_k|$ together, rather than isolating head count.

\begin{remark*}
Unlike the Alternating languages, we \edits{have no proof} that representing $A_k$ requires $k$ teacher heads. We nevertheless use $H_t=k$ as a heuristic, supported by two main reasons. First, we empirically observed that learning the Alphabet languages with less heads than alphabet symbols is difficult. Secondly, under a sufficiently skewed distribution, if one head monitors two letters $a, b$, its attention mechanism may fail to detect a sporadic presence of letter $b$ among a word full of $a$'s: indeed, most of the attention mass would be set on $a$'s, while fixed-precision rounding and saturation would potentially leave $b$ undetected.
\end{remark*}

\paragraph{Specifics.} Letters are sampled independently and uniformly.
Each ordinary position is labelled by the membership of its prefix in
$A_k$, and EOS by the membership of the full word. Evaluation includes
all these positions.

\paragraph{Example.} For $A_3$ over $\{a,b,c\}$, the target becomes $1$
once every letter has appeared:
\[
\begin{array}{r|ccccccc}
w & \mathtt{a} & \mathtt{b} & \mathtt{c} & \mathtt{a} & \mathtt{b} & \mathtt{c} & \mathtt{EOS} \\
y & 0 & 0 & 1 & 1 & 1 & 1 & 1
\end{array}
\]

\subsubsection{\texorpdfstring{$\mathit{IH}_k$}{IHk}: Induction Head tasks}

The Induction Head tasks have a vector-valued prediction at every ordinary token
position; EOS is not supervised or evaluated. Here again, $\Sigma_k=\{s_1,\ldots,s_k\}$.  For each
$i\in[|w|]$, the target output is a vector
$\mathit{IH}_k(w)_i\in\{0,1\}^k$, with one dimension per letter in $\Sigma_k$, whose $d$-th coordinate records
whether the bigram $\langle w_i,s_d\rangle$ has appeared strictly earlier in
the word:
\begin{equation}
 \bigl(\mathit{IH}_k(w)_i\bigr)_d
 :=\mathbf{1}\!\left\{\exists j\in\{1,\ldots,i-2\}:\
       (w_j,w_{j+1})=(w_i,s_d)\right\},
 \qquad d\in[k].
 \label{eq:experimental-IHk}
\end{equation}
For instance, at the third position of $\mathtt{abac}$, the current letter
is $a$ and the earlier bigram $\langle a,b\rangle$ has occurred, so the
$b$-coordinate is $1$. Thus $\mathit{IH}_k$ is the product of $k$ binary
classification tasks, one for each possible successor $s_d$. Distinct
letters must have distinct embeddings, so fixed precision requires
$|\F|^{D_t}\geq k$. This family therefore varies the alphabet size and
the teacher's width requirement, as well as the output dimension $k$.

\paragraph{Specifics.}
The induction-head task uses local positional information, unlike the
preceding two families. A range-one local attention head can attend only
to position $i-1$ at position $i>1$, providing the information needed to
form bigrams without a separate positional encoding.
Appendix~\ref{appendix:local-attention} treats this extension through a
finite suffix cache coupled to a partially ordered component. We use
only its past-subformula-counting bound, corresponding to $C_0$ in
Lemma~\ref{thm:deterministic-length-generalization}, with the additional
cache and block factors. The temporal-profile bound $C_1$ is not
established for this extension.

\paragraph{Example.} For $\mathit{IH}_3$ over $\{a,b,c\}$, write target
coordinates in the order $(a,b,c)$:
\[
\begin{array}{r|cccccc}
w & \mathtt{a} & \mathtt{b} & \mathtt{a} & \mathtt{c} & \mathtt{a} & \mathtt{b} \\
y & (0,0,0) & (0,0,0) & (0,1,0) & (0,0,0) & (0,1,1) & (1,0,0)
\end{array}
\]

\subsection{Protocol and evaluation}
\label{sec:experimental-protocol}

\paragraph{Binary loss and accuracy.}\label{app:bce-and-accuracy}
Exact equality of FP32 outputs is not a useful experimental criterion, so
we treat each target coordinate as a binary class. For a word $w$, let
$\mathcal S(w)$ be its supervised positions, let $y_{i,d}\in\{0,1\}$ be
the $d$th target coordinate at position $i$, and let $z_{i,d}$ be the
corresponding model logit. We minimize
\begin{align}
 \mathcal L(\theta)
 &=\frac{1}{\sum_{n=1}^N|\mathcal S(w^{(n)})|}
   \sum_{n=1}^N\sum_{i\in\mathcal S(w^{(n)})}
   \frac{1}{\nu}\sum_{d=1}^{\nu}
   \ell_{\mathrm{BCE}}\bigl(z^{(n)}_{i,d},y^{(n)}_{i,d}\bigr),
 \label{eq:experimental-bce}\\
 \ell_{\mathrm{BCE}}(z,y)
 &=-y\log\sigma(z)-(1-y)\log\bigl(1-\sigma(z)\bigr),
\end{align}
with $\sigma(z)\coloneqq(1+e^{-z})^{-1}$.
Note that the output dimension $\nu$ is only $1$ for $L_k$ and $A_k$, and $\nu=|\Sigma_k|$ for
$\mathit{IH}_k$. Thus the first two tasks use one BCE term at every
ordinary position and EOS, whereas $\mathit{IH}_k$ averages BCE across
all $|\Sigma_k|$ target coordinates at each ordinary position.

We classify a coordinate as $\widehat y_{i,d}=\mathbf 1\{z_{i,d}\geq
0\}$ and a word is correct if $\widehat y_{i,d}=y_{i,d}$ for every
$i\in\mathcal S(w)$ and $d\in[\nu]$. Accuracy is the fraction of correct
words, and perfect accuracy means that every word is correct. 

\begin{remark*}
In the fixed-precision abstraction, the binary classifications given above, based on the logit regions $1\{z_{i,d} < 0\}$ and $\{z_{i,d}\geq
0\}$, are representable by boolean predicates in the same way as done in Proposition~\ref{prop:tokenwise-disagreement}. Equivalently, the binary classification framework naturally stems from an alternative definition of the difference function \ref{app:def:difference-function}, checking on two logit regions rather than on exact equality.
\end{remark*}

\paragraph{Sampling and supervision.}
Unless stated otherwise, all ordinary input symbols are drawn independently
and uniformly from the alphabet. Each sampled word has length exactly
$K$, measured in ordinary symbols. Agreement requires an exact match at
every supervised position, and the current letter can contribute through
the residual stream.

For $L_k$ and $A_k$, we additionally append and supervise an EOS token.
These are language-recognition tasks, so the terminal readout of
\citet{li_cotterell2025fpt} is natural, alongside our prefix-wise outputs.
EOS is an experimental convention, not an assumption of our base framework.
Requiring a correct EOS output adds a success condition, making evaluation
at least as demanding as checking ordinary positions alone; we do not
claim a measured effect on training difficulty. The persistent difference
function in our bounds concerns ordinary positions, while EOS accuracy is
an additional check. 

\paragraph{Optimisation and repetitions.}
In all cases, optimisation uses Adam \cite{kingma2014adam} with \edits{a batch size of 32 and a} 
learning rate searched $\in  \{$2e-5, 5e-5, 1e-4$\}$, choosing the learning rate that induces the most seeds to fit perfectly the training data. 
Training lasts for at most $2{,}000$ epochs and stops early
when the epoch training loss is below 1e-5. Unless the quantity $N$ is explicitly given, the training sets contains $N=256$ i.i.d. samples. All configurations in all the experiments were run $10$ times over $10$ different seeds, \edits{each seed affecting the initialization, the training dataset and the evaluation dataset}. Finally, a model is selected in our analyses if and only if it achieves perfect accuracy on its training set. This matches the consistency requirement of Guess-and-Check, but not its distribution over hypotheses.

\paragraph{Choice of training length.}
\label{app:experiments-adjusted-length}
\edits{For simplicity, generated datasets are made of samples of length exactly $K$, instead of words of different lengths smaller than $K$. All the results of this paper hold up to a slight constant adaptation, stemming immediately from the deletion of the weight factor $\lambda_K$ in Section~\ref{sec:probabilistic-LG}.} Unless $K$ is swept, it is chosen using task-specific coverage criteria:
approximately $80\%$ of $L_k$ words starting with $a$ contain the required
blocks, $80\%$ of $A_k$ words contain every letter, and $50\%$ of
$\mathit{IH}_k$ words have exhibited every bigram before the final
ordinary position. These criteria concern the target task. They were also partially selected empirically
to make the experiments interesting (neither trivial nor impossible).
 The lower threshold for $\mathit{IH}_k$
reflects the useful coordinate-wise supervision available before all
bigrams have occurred. The $A_k$ and $\mathit{IH}_k$ choices are guided by
coupon-collector estimates.

\paragraph{Default student architectures.}
The trained transformers \edits{architecture is an enhanced version of the definition given in Appendix~\ref{sec:fpt-definition}, in order to closing in with real-world transformers. We use affine query, key, and value projections and apply an affine output projection after head concatenation. Attention scores are divided by $\sqrt{D_{s,h}}$ as is customary. Furthermore, we apply Post LayerNorm ~\citep{ba2016layernorm} twice in each layer, once after the attention and once after the feed-forward. Local-attention heads are also added for the induction head tasks $\mathit{IH}_k$.} The ReLU feed-forward layers have one hidden layer of width $2D_s$. Outputs are evaluated by coordinate-wise binary
classification rather than exact equality of the logits. Computations
use PyTorch's default \textbf{FP32} precision, which naturally reflects the
fixed-precision setting. \edits{We give the exact student parameter sizes of the experiments in the next section.} %

\edits{
\begin{remark*}
To prove that every $\LTL$ formula can be simulated by a fixed-precision transformer despite vanishing attention weights, \citet[Section~3.3]{li_cotterell2025fpt} provides a two-layer patch construction to implement the $\mathbf{P}$ast operator resilient to any input length. The first block identifies witnesses before attention weights vanish, and the second aggregates this set without saturation. Consequently, experimental and theoretical transformer depths for our tasks may differ by a constant; for example, the alphabet languages $A_k$ might theoretically require two layers to handle vanishment of attention weights. Since FP32 saturation arises only at lengths far beyond our experiments and might not follow the exact arithmetic assumption stated in Appendix~\ref{app:arithmetic-assumptions}, we ignore this distinction.
\end{remark*}
}

In the end, the evaluated models were selected \edits{based on} representational complexity and tractability, ensuring they were sufficiently expressive to yield meaningful insights while remaining learnable.

\paragraph{Evaluation.}
We evaluate at absolute test lengths, such as $K_{\mathrm{test}}=500$,
or relative lengths, such as $K_{\mathrm{test}}=5K$. These lengths count
ordinary symbols, excluding EOS. Each model is evaluated on $2{,}000$
independent test words of the specified length. The supervised positions
are the ordinary positions and EOS for $L_k$ and $A_k$, and the ordinary
positions only for $\mathit{IH}_k$; accuracy is as defined above.

To \edits{match} the selection \edits{prior} of Guess-and-Check, as noticed above, we retain for plotting only
models with perfect training accuracy. Consequently, a reported point can aggregate fewer than the $10$ attempted
seeds \edits{and} a blank region in a plot \edits{signals} that no seed achieved perfect training
accuracy for that configuration. \edits{In particular, some configurations may represent only one of the ten seeds if the training was difficult on this configuration.}

\subsection{Results}

We first examine whether increasing the teacher's parameter size shifts the
amount of data or training length needed for high test accuracy.  
Teacher sizes refer to realizations of the target functions. The \edits{three} families probe different, sometimes coupled resources:
depth $L_t$ for $L_k$; head count $H_t$, width $D_t$, and alphabet size
for $A_k$; and width requirements, alphabet size, and output dimension
for $\mathit{IH}_k$.

\edits{We then study how varying the student parameter sizes can impact the performance, and how it compares with theory.}

\edits{Finally, we demonstrate how the required training length $K$ and sample size $N$ trade-off for the three task families.}

\subsubsection{Impact of the teacher's size}
\label{sec:teacher-complexity-results}

The two height bounds in
Lemma~\ref{thm:deterministic-length-generalization} give different
dependencies: $C_0$ grows with depth and width, but decreases with head
count at fixed width (other quantities held fixed); $C_1$ depends on depth
and alphabet size, not width or head count. \edits{The statistical bounds on $N$ depend on width and depth. In addition, a larger alphabet lowers the worst-case probability
of making progress in the automaton, impacting $K$, while also showing up in bounds on $N$}. These effects can increase the
sufficient requirements on $N$ and $K$. \edits{The experiments therefore give us insights on whether the empirically demanded quantities $N$ and $K$ pursue the theoretical sufficient requirements.} For $\mathit{IH}_k$, only the local extension of
the $C_0$ argument is used.

In this subsection studying the impact of teacher's size, student backbone $(L_s,D_s,H_s)$
is held fixed across tasks; input embeddings and output dimensions
still follow the task alphabet and targets. Moreover, among each family, all students have the same parameter sizes, despite having different targets. This permits to \edits{restrain the variations induced by different student sizes. For every family of tasks, we set the student structure to be large enough so that all the targetted teachers are realizable.}
\begin{itemize}
 \item For the alternating languages, the studied tasks are $L_4, L_5, L_6, L_7$. Every student has exactly $L_s=6$ layers, embedding dimension $D_s=256$, and one attention head $H_s =1$. When varying $K$, they all train on $256$ samples. When varying $N$, the training length $K$ is set, respectively for  $L_4, L_5, L_6, L_7$, to $86, 111, 135, 158$ (c.f. Choice of training lengh~\ref{app:experiments-adjusted-length}).
 \item For the alphabet languages, the tasks are $A_{10}, A_{12}, A_{15}, A_{20}$. Every student has one layer $L_s=1$ and $H_s=30$ attention heads. The global embedding dimension is fixed to $D_s=360$. When varying $K$, they all train on $256$ samples. When varying $N$, the training length $K$ is set, respectively for  $A_{10}, A_{12}, A_{15}, A_{20}$, to $37, 47, 62, 89$.
 \item For the induction head tasks, the targets are $\mathit{IH}_5, \mathit{IH}_6, \mathit{IH}_7, \mathit{IH}_8$. Every student has two layers $L_s=2$ and embedding
 dimension $D_s=256$ and two heads per layer: one local head of range
 $1$ and one global head. When varying $K$, they all train on $256$ samples. When varying $N$, the training length $K$ is set, respectively for $\mathit{IH}_5, \mathit{IH}_6, \mathit{IH}_7, \mathit{IH}_8$, to $92, 144, 211, 292$.
\end{itemize}

\begin{remark*}
    \edits{While we set $L_s = 6$ for alternating languages to correspond with the minimal realizable number of layers of the biggest task $L_7$, we do not proceed accordingly for the alphabet languages $A_k$. Indeed, we set $H_s$ to $30$ and not $20$, as learnability was observed to be difficult when the number of heads is exactly the size of the alphabet. This choice does not alter the informative nature of the experiments or their relationship to the theoretical framework.}
\end{remark*}

\paragraph{Required data complexity.} We first vary $N$ and compare task
sizes within each family. The student backbone is fixed, while $K$ is
adjusted using the coverage criteria above. Larger tasks therefore receive
longer, potentially more informative examples, which could partly offset
their increased difficulty. Nevertheless, the observed sample requirement
generally increases with task size. For $A_k$ and $\mathit{IH}_k$, this
comparison does not separate teacher size from alphabet size or, for
$\mathit{IH}_k$, the number of output coordinates that must all be correct.

\begin{figure}[H]
    \centering
    \begin{subfigure}[b]{\textwidth}
        \centering
        \includegraphics[width=\linewidth]{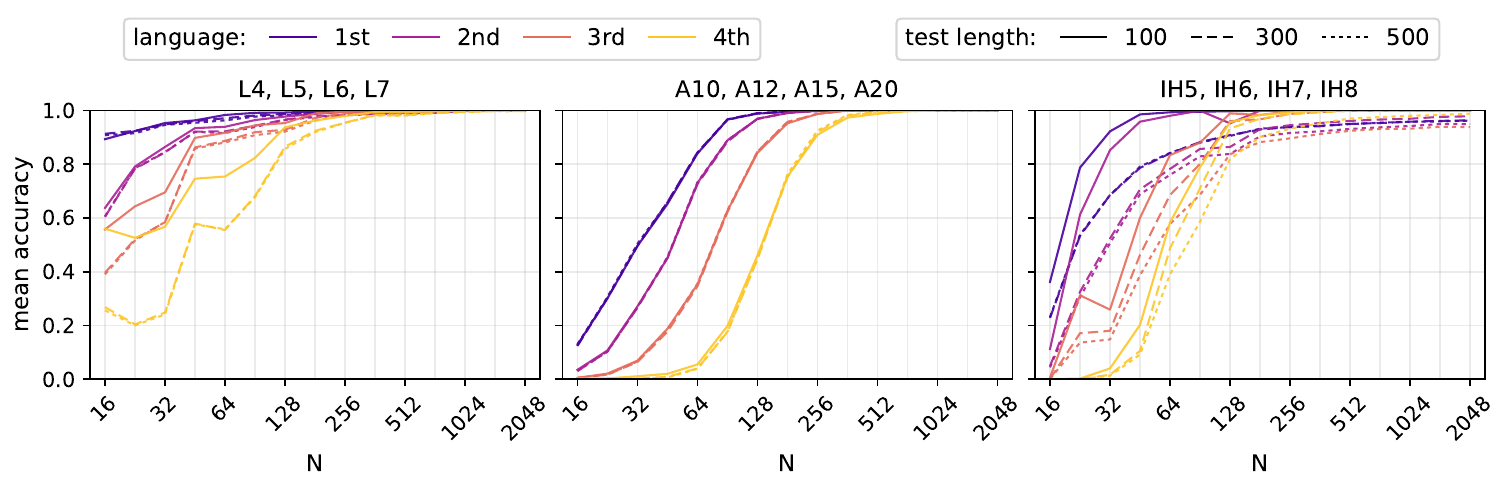}
        \vspace{-15pt}
        \caption{\small Evaluation on absolute lengths}
        \label{fig:aggregate_N_absolute}
    \end{subfigure}
    \vspace{0.5cm} %
    \begin{subfigure}[b]{\textwidth}
        \centering
        \includegraphics[width=\linewidth]{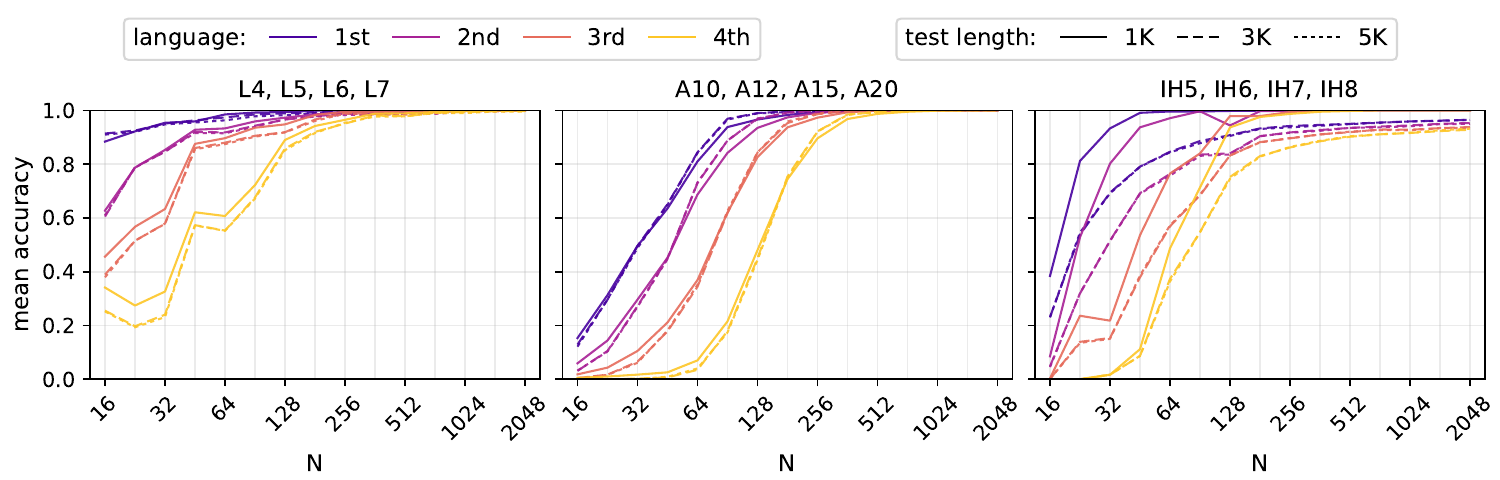}
        \vspace{-15pt}
        \caption{\small Evaluation on relative lengths}
        \label{fig:aggregate_N_relative}
    \end{subfigure}
        \vspace{-1.1cm} %
    \caption{\small \textbf{Impact of teacher's size on data complexity}. $L_k$ probes teacher depth $L_t$; $A_k$ jointly varies $H_t$, $D_t$, and $|\Sigma|$; $\mathit{IH}_k$ varies alphabet size, width requirements, and output dimension. Student backbones are fixed within each family, while $K$ follows the task-specific coverage criterion.}
    \label{fig:aggreage_N}
\end{figure}

\paragraph{Required training length.} We similarly sweep $K$, comparing
task sizes at each common training length, with $N$ and the student
backbone fixed. Larger tasks generally require longer training words for
high test accuracy. We observe that this is qualitatively consistent with the increased
sufficient length scales.

\begin{figure}[H]
    \centering
    \begin{subfigure}[b]{\textwidth}
        \centering
        \includegraphics[width=\linewidth]{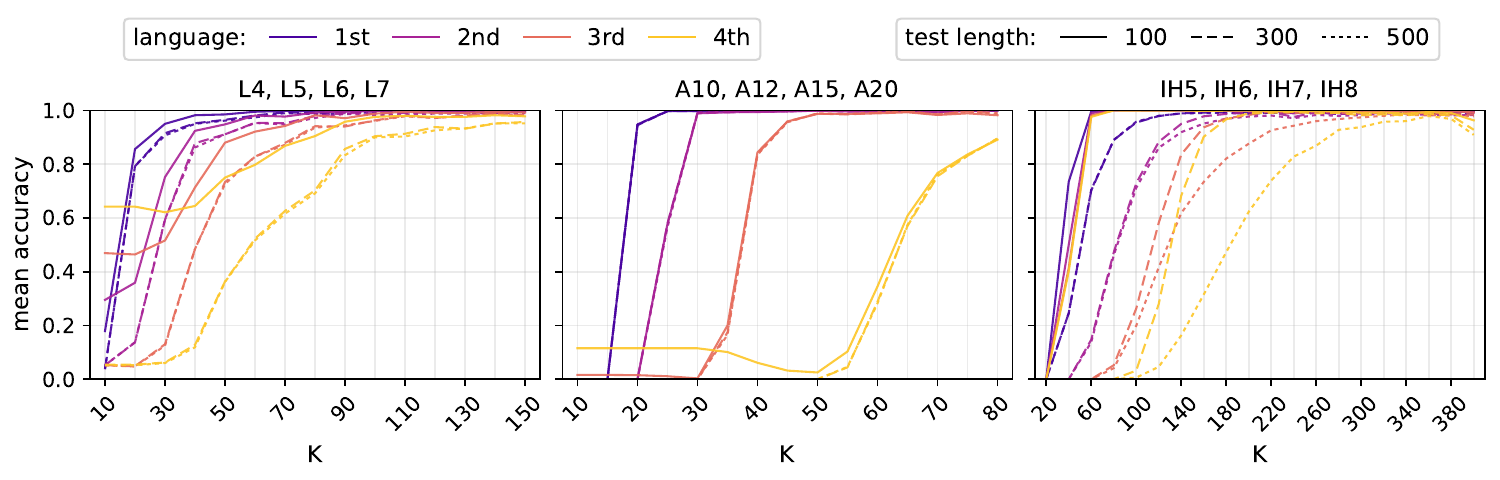}
        \vspace{-15pt}
        \caption{\small Evaluation on absolute lengths}
        \label{fig:aggregate_K_absolute}
    \end{subfigure}
    \vspace{0.5cm} %
    \begin{subfigure}[b]{\textwidth}
        \centering
        \includegraphics[width=\linewidth]{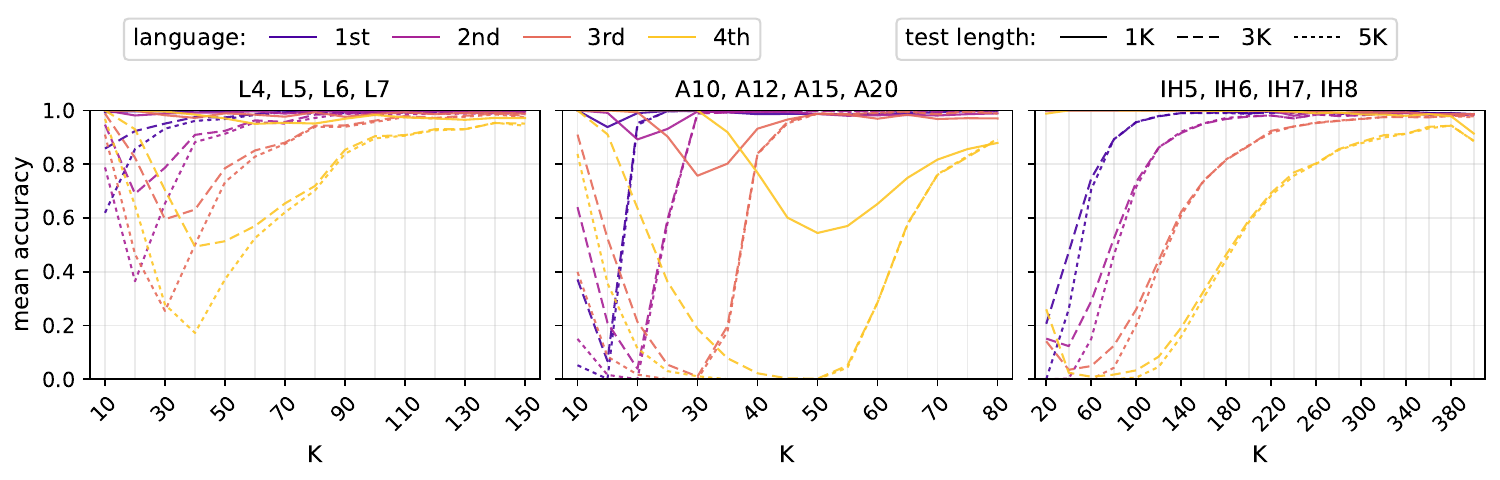}
        \vspace{-15pt}
        \caption{\small Evaluation on relative lengths}
        \label{fig:aggregate_K_relative}
    \end{subfigure}
        \vspace{-1.1cm} %
    \caption{\small \textbf{Impact of teacher's size on length complexity}. $L_k$ probes teacher depth $L_t$; $A_k$ jointly varies $H_t$, $D_t$, and $|\Sigma|$; $\mathit{IH}_k$ varies alphabet size, width requirements, and output dimension. Tasks are compared at shared values of $K$, with $N$ and student backbones fixed. \edits{For relative evaluations, early accuracy is sometimes high for tasks that are trivial on short words: words of length $3K$ or $5K$ might still be short and have a trivial target.}}
    \label{fig:aggreage_K}
\end{figure}

\subsubsection{Impact of the student's size}
\label{app:experiments-impact-student}

We next vary the size of the student, $L_s, D_s$ and $ H_s$, while holding the teacher and all other
independent experimental parameters fixed. \edits{Conversely to the previous results studying the impact of the teacher sizes, we do not fix the student architecture within families. Instead for each task, we set parameters 'close' to the hypothesized minimal size for realizability. Hence, apart from explicit variations of architectural parameters, the trained students use the following default architectures.}
\begin{itemize}
 \item For the alternating languages $L_k$, the tasks are $L_4, L_5, L_6, L_7$. By default, the students have $L_s=k-1$ layers, embedding dimension
 $D_s=256$, and one attention head $H_s =1$. $N = 256$ and $K$ is set, respectively for  $L_4, L_5, L_6, L_7$, to $86, 111, 135, 158$ (c.f. Choice of training lengh~\ref{app:experiments-adjusted-length}).
 \item For the alphabet languages $A_k$, the tasks are $A_{10}, A_{12}, A_{15}, A_{20}$. By default, the students have one layer $L_s=1$ and $H_s=k$ attention heads. The global embedding
 dimension is fixed to $D_s=360$. $N = 256$ and $K$ is set, respectively for  $A_{10}, A_{12}, A_{15}, A_{20}$, to $37, 47, 62, 89$.
 \item For the induction head $\mathit{IH}_k$, the tasks are $\mathit{IH}_5, \mathit{IH}_6, \mathit{IH}_7, \mathit{IH}_8$. By default, the students have two layers $L_s=2$ and embedding
 dimension $D_s=256$, with two heads per layer: one local head of range
 $1$ and one global head. $N = 256$ and $K$ is set, respectively for $\mathit{IH}_5, \mathit{IH}_6, \mathit{IH}_7, \mathit{IH}_8$, to $92, 144, 211, 292$.
\end{itemize} 

Overall, the effect \edits{of the student size} is not
consistently monotone. But the bounds describe worst-case sufficient resources and added capacity can also improve learning.
This distinction is relevant in overparameterized models, as discussed in Appendix~\ref{sec:app:faq}.

\paragraph{Student depth.}
Increasing the number of student layers $L_s$ sometimes improves learning
and length generalization. For $A_k$, however, it substantially degrades
length generalization. This observation is consistent with the depth dependence of
both height bounds and the resulting sufficient requirements on $N$ and
$K$.

\paragraph{Student embedding dimension.}
Varying the student embedding dimension $D_s$ has little visible effect on
length generalization and, in some settings, improves learning. Thus, increasing
$D_s$ does not appear to impose a systematic data or length-complexity cost in
these experiments. At fixed $H_s$, $C_0$ increases with $D_s$, whereas
$C_1$ is independent of width. The observations therefore do not support
a simple monotone width penalty nor identify which bound is
informative. For $\mathit{IH}_k$, the comparison remains with the local
extension of $C_0$, not $C_1$.

\paragraph{Student head count.}
\edits{Although varying $H_s$ may violate the condition $D_{h,s} \geq D_{h,t}$ from Remark~\ref{remark:learning-assumptions}, the $ERM$ learner depicted in Appendix~\ref{sec:app:data-complexity:erm} does not require this condition as long as the target remains realizable. These experiments are therefore suitable for testing the theory’s qualitative predictions. We observe that} increasing the number of student attention heads $H_s$ appears to slightly
harm learning in the $L_k$ and $\mathit{IH}_k$ families, but substantially
improves it in the $A_k$ family, including beyond the $k$ heads used in the
reference construction. At fixed $D_s$, increasing $H_s$ reduces
$D_{h,s}=D_s/H_s$ and hence the student contribution
$r_\F L_sD_s|\F|^{D_s/H_s+1}$ to $C_0$; $C_1$ is unchanged.
The improvement for $A_k$ is notably compatible with this direction.
For $\mathit{IH}_k$, \edits{the sweep varies the total head count, always keeping equal numbers of local and global heads}.

\begin{figure}[H]
    \centering
    \begin{subfigure}[b]{\textwidth}
        \centering
        \includegraphics[width=\linewidth]{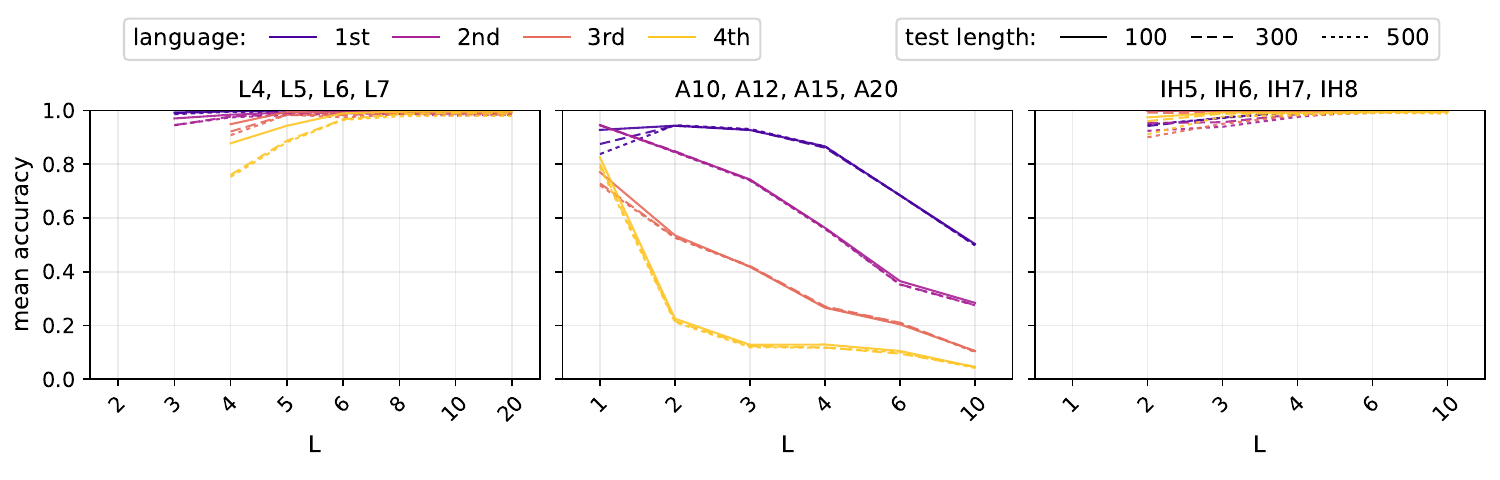}
        \vspace{-15pt}
        \caption{\small Evaluation on absolute lengths}
        \label{fig:student_L_absolute}
    \end{subfigure}
    \vspace{0.5cm} %
    \begin{subfigure}[b]{\textwidth}
        \centering
        \includegraphics[width=\linewidth]{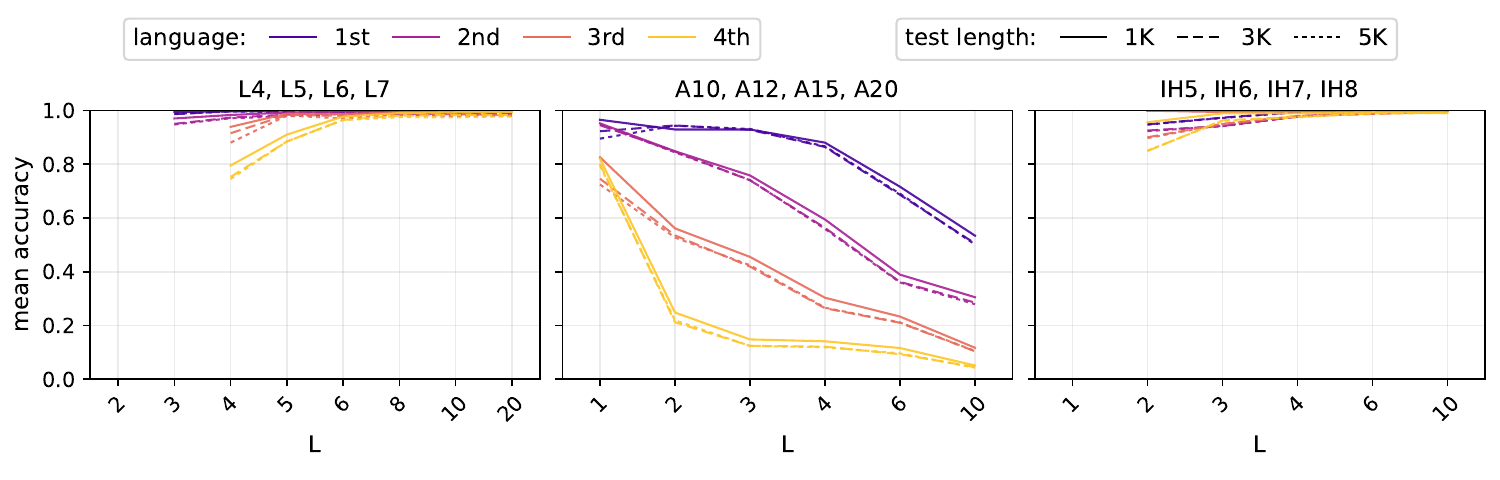}
        \vspace{-15pt}
        \caption{\small Evaluation on relative lengths}
        \label{fig:student_L_relative}
    \end{subfigure}
        \vspace{-1.1cm} %
    \caption{\small \textbf{Impact of the student's depth on length generalization}. We vary $L_s$ while holding $D_s$ and $H_s$ fixed.}
    \label{fig:student_L}
\end{figure}

\begin{figure}[H]
    \centering
    \begin{subfigure}[b]{\textwidth}
        \centering
        \includegraphics[width=\linewidth]{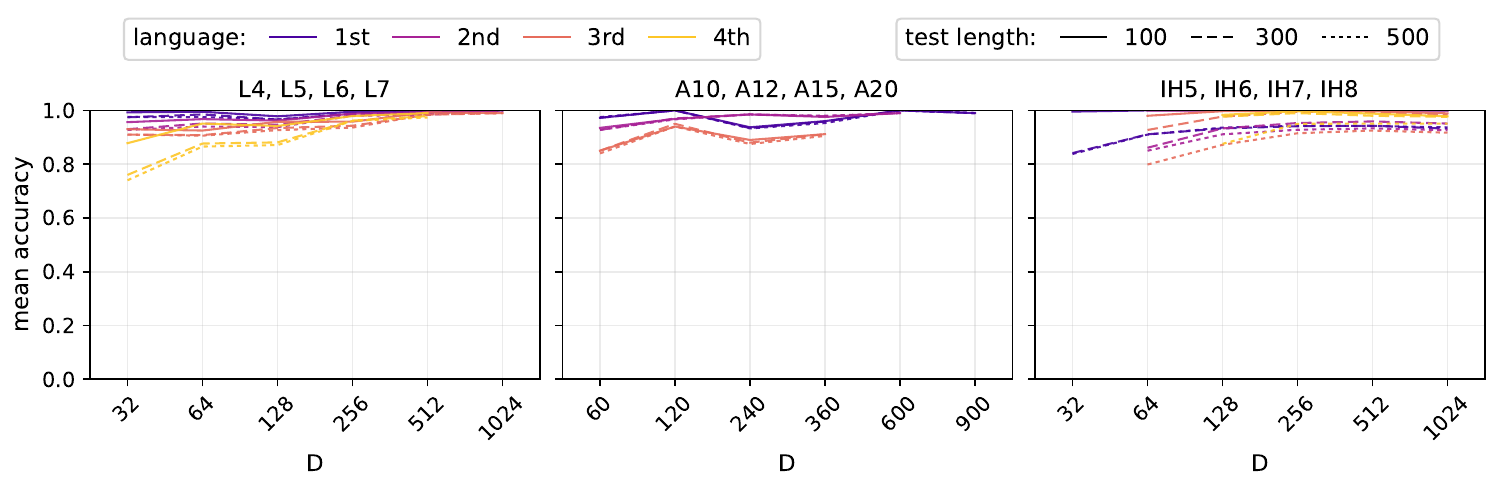}
        \vspace{-15pt}
        \caption{\small Evaluation on absolute lengths}
        \label{fig:student_D_absolute}
    \end{subfigure}
    \vspace{0.5cm} %
    \begin{subfigure}[b]{\textwidth}
        \centering
        \includegraphics[width=\linewidth]{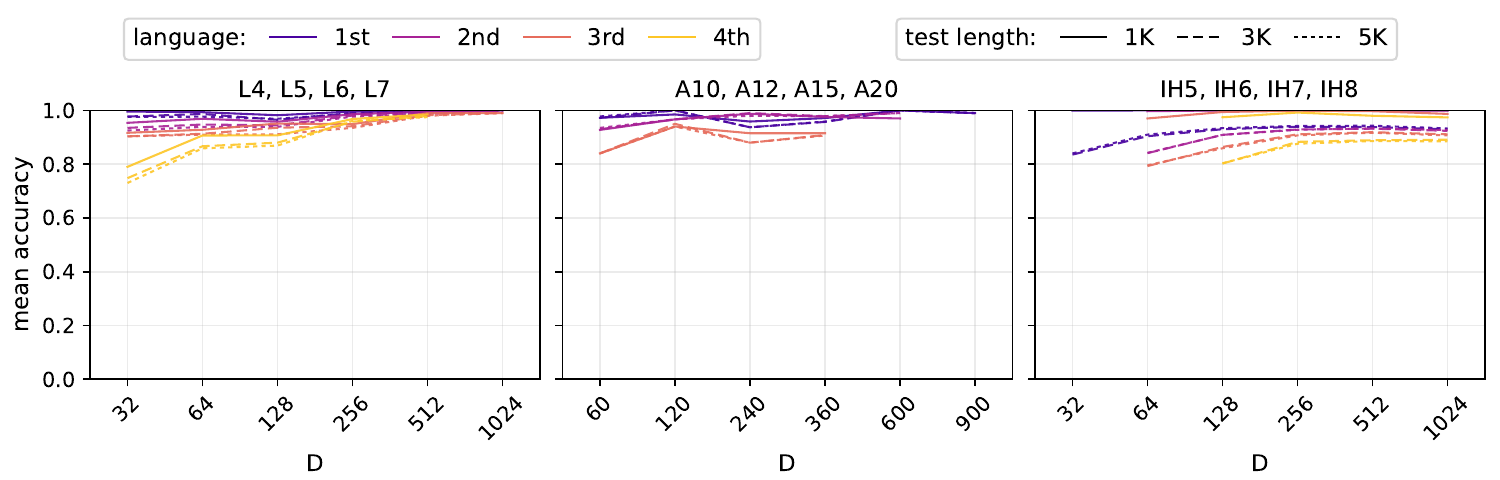}
        \vspace{-15pt}
        \caption{\small Evaluation on relative lengths}
        \label{fig:student_D_relative}
    \end{subfigure}
        \vspace{-1.1cm} %
    \caption{\small \textbf{Impact of the student's embedding dimension on length generalization}. We vary $D_s$ while holding $L_s$ and $H_s$ fixed.}
    \label{fig:student_D}
\end{figure}

\begin{figure}[H]
    \centering
    \begin{subfigure}[b]{\textwidth}
        \centering
        \includegraphics[width=\linewidth]{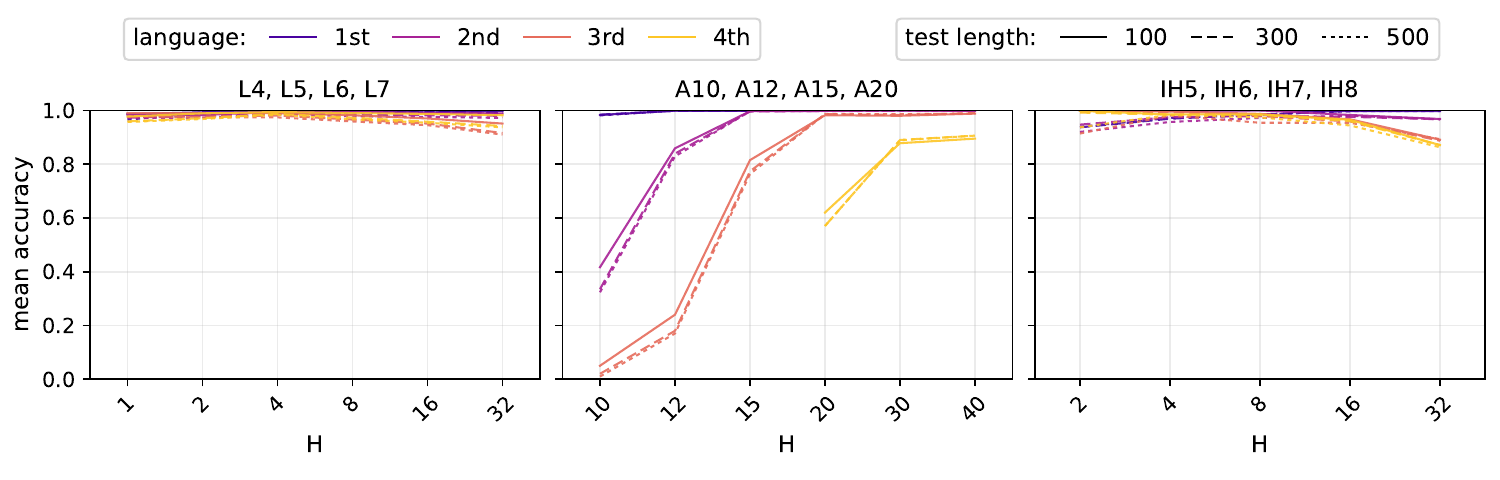}
        \vspace{-15pt}
        \caption{\small Evaluation on absolute lengths}
        \label{fig:student_H_absolute}
    \end{subfigure}
    \vspace{0.5cm} %
    \begin{subfigure}[b]{\textwidth}
        \centering
        \includegraphics[width=\linewidth]{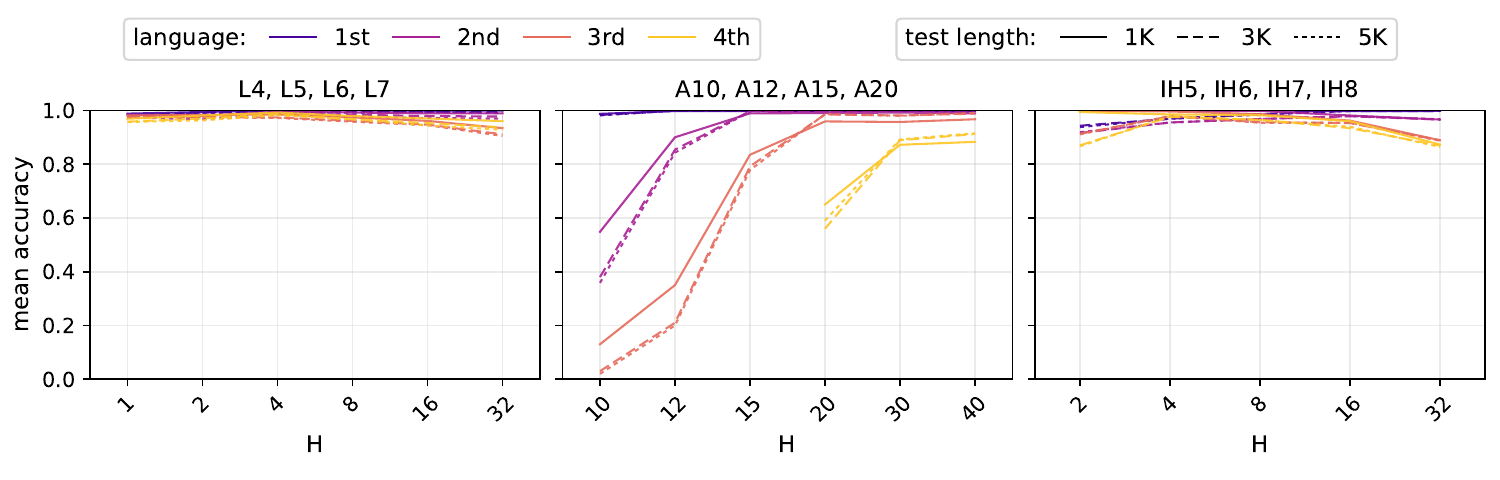}
        \vspace{-15pt}
        \caption{\small Evaluation on relative lengths}
        \label{fig:student_H_relative}
    \end{subfigure}
        \vspace{-1.1cm} %
    \caption{\small \textbf{Impact of the student's head count on length generalization}. We vary $H_s$ while holding $L_s$ and $D_s$ fixed.}
    \label{fig:student_H}
\end{figure}

\subsubsection{Tradeoff between training length and data complexity}
\label{app:experiments-tradeoff}

For these tradeoff experiments, student models are set with the following parameter sizes.
\begin{itemize}
    \item The alternating task $L_6$ has exactly $5$ layers, $D_s = 256$ and $H_s = 1$.
    \item The alphabet task $A_{15}$ has exactly $1$ layer, $D_s = 360$ and $H_s = 20$. The number of heads is $20$ and not $15$ no ease the learning procedure.
    \item The induction head task $\mathit{IH}_7$ has exactly $2$ layers, $D_s = 256$ and it has two attention heads, one local head of range $1$ and one global attention head.
\end{itemize}

We observe a tradeoff between training length $K$ and sample size $N$:
longer examples can substantially reduce the number needed for length
generalization. Both axes use linear scales. The observed frontier is
qualitatively consistent with the logarithmic--linear tradeoff illustrated
in Figure~\ref{fig:multiplicative-tradeoff}.
For a fixed accuracy threshold, the configurations that achieve the
threshold with the smallest token budget typically use a larger $K$ and a
smaller $N$. Conversely, for a fixed token budget, the configurations with the
best accuracy likewise favor a larger $K$ over a much larger $N$. This
resembles the allocation suggested by
Corollary~\ref{cor:optimal-token-allocation}. For $\mathit{IH}_k$, the
relevant theoretical comparison uses the blockwise local-attention
extension of the $C_0$ argument.
The plotted budgets $NK$ count ordinary tokens; $L_k$ and $A_k$ additionally
use one supervised EOS position per word.

\begin{figure}[H]
    \centering
    \begin{subfigure}[b]{\textwidth}
        \centering
        \includegraphics[width=\linewidth]{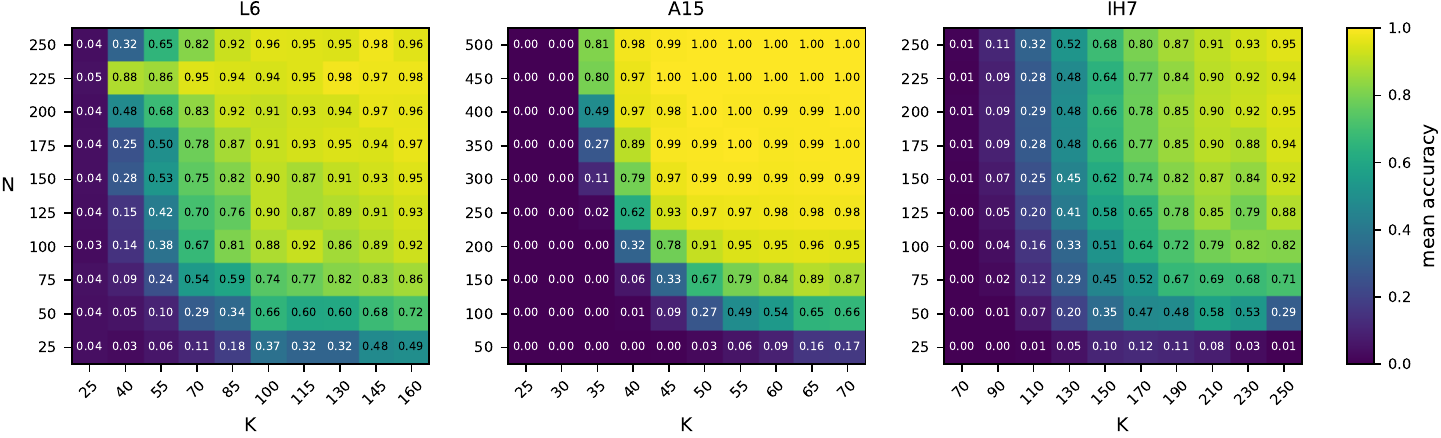}
        \vspace{-15pt}
        \caption{\small Evaluation at absolute test length $500$}
        \label{fig:aggregate_KxN_absolute_500}
    \end{subfigure}
    \vspace{0.5cm} %
    \begin{subfigure}[b]{\textwidth}
        \centering
        \includegraphics[width=\linewidth]{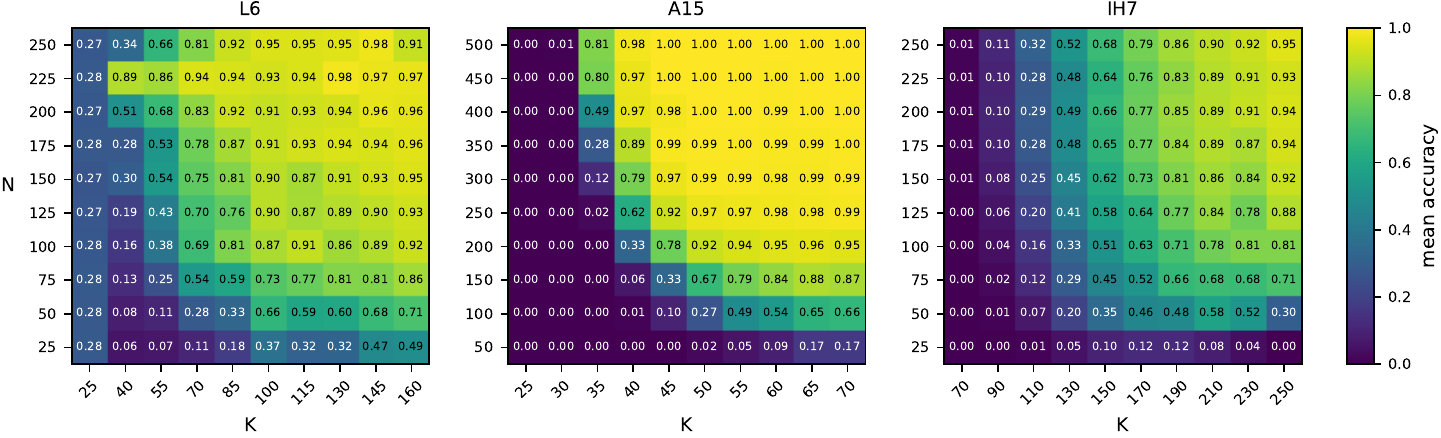}
        \vspace{-15pt}
        \caption{\small Evaluation at relative test length $5K$}
        \label{fig:aggregate_KxN_relative_5K}
    \end{subfigure}
    \vspace{0.5cm} %
    \begin{subfigure}[b]{\textwidth}
        \centering
        \includegraphics[width=\linewidth]{xp_plots/Aggregate_KxN_relative_accuracy_5K.pdf}
        \vspace{-15pt}
        \caption{\small Here is Figure~\ref{fig:aggregate_KxN_5K} for completeness. For a 0.75 minimum accuracy threshold, highlighted in red are the 3 configurations that have the smallest token-budget.}
        \label{fig:aggregate_KxN_relative_accuracy_5K}
    \end{subfigure}
    \vspace{0.5cm} %
    \begin{subfigure}[b]{\textwidth}
        \centering
        \includegraphics[width=\linewidth]{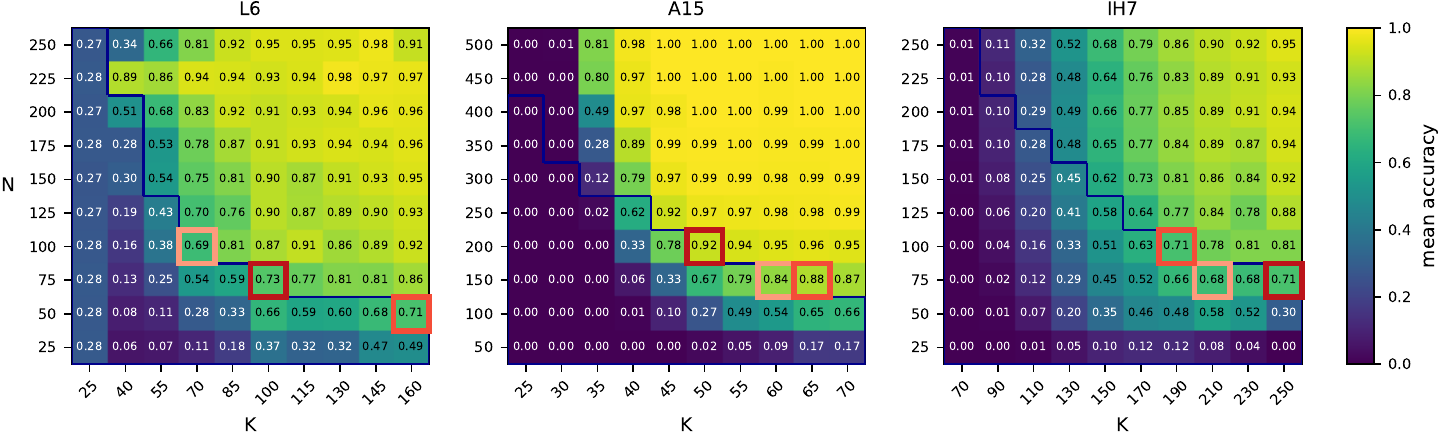}
        \vspace{-15pt}
        \caption{\small For a maximum token-budget threshold, respectively of $8000, 10000, 20000$ tokens, highlighted in red are the 3 configurations that achieve the best accuracy}
        \label{fig:aggregate_KxN_relative_budget_5K}
    \end{subfigure}
        \vspace{-1.1cm} %
    \caption{\small \textbf{Tradeoff between training length and data complexity}. One task of each family: $L_6, A_{15}$ and $\mathit{IH}_7$ to study the tradeoff between the two quantities.}
    \label{fig:aggregate_KxN_tradeoff}
\end{figure}

We also include a control experiment in which the axis labelled $N$
specifies a budget of supervised token positions rather than a number
of words. Gains from increasing $K$ at a fixed budget suggest a
benefit from longer contexts beyond simply supplying more supervised
positions.
\edits{This is in agreement with theoretical predictions from Eq.~\ref{eq:multiplicative-bound}.}

\begin{figure}[H]
    \centering
    \includegraphics[width=\linewidth]{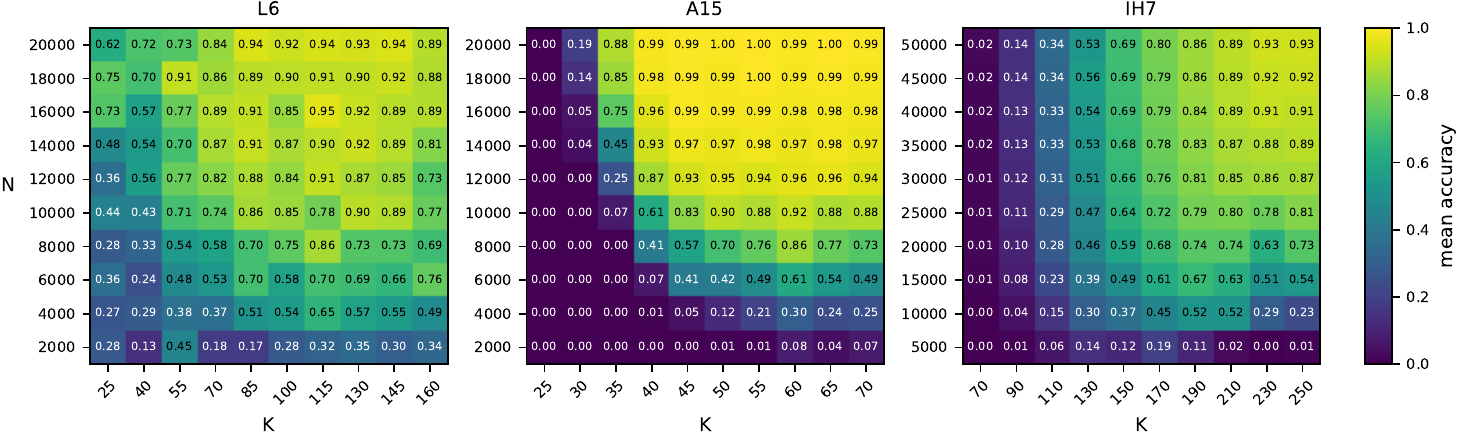}
    \vspace{-15pt}
    \caption{\small Control experiment for $L_6$, $A_{15}$, and $\mathit{IH}_7$\edits{, with evaluation at relative test length 5$K$}. The axis labelled $N$ gives the \edits{total token budget of the dataset} rather than the number of words.}
    \label{fig:aggregate_KxN_relative_control}
\end{figure}

We have seen above the aggregation of the Length--Sample tradeoff of three distinct tasks $L_6$, $A_{15}$, and $\mathit{IH}_7$. We complete the collection with a compendium, of many combinations with increasing test lengths for each of three tasks.

\begin{figure}[H]
    \centering
    \begin{subfigure}[b]{\textwidth}
        \centering
        \includegraphics[width=\linewidth]{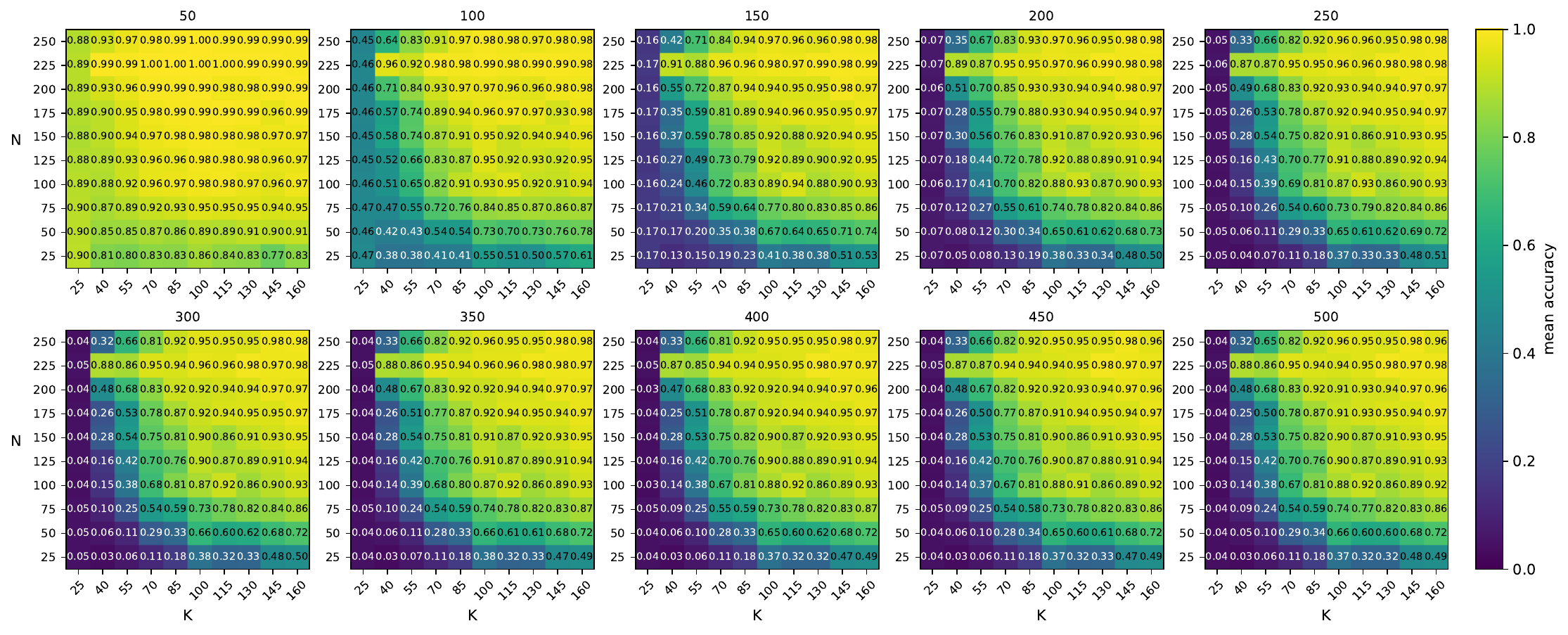}
        \vspace{-15pt}
        \caption{\small Alternating language $L_6$}
    \end{subfigure}
    \vspace{0.5cm} %
    \begin{subfigure}[b]{\textwidth}
        \centering
        \includegraphics[width=\linewidth]{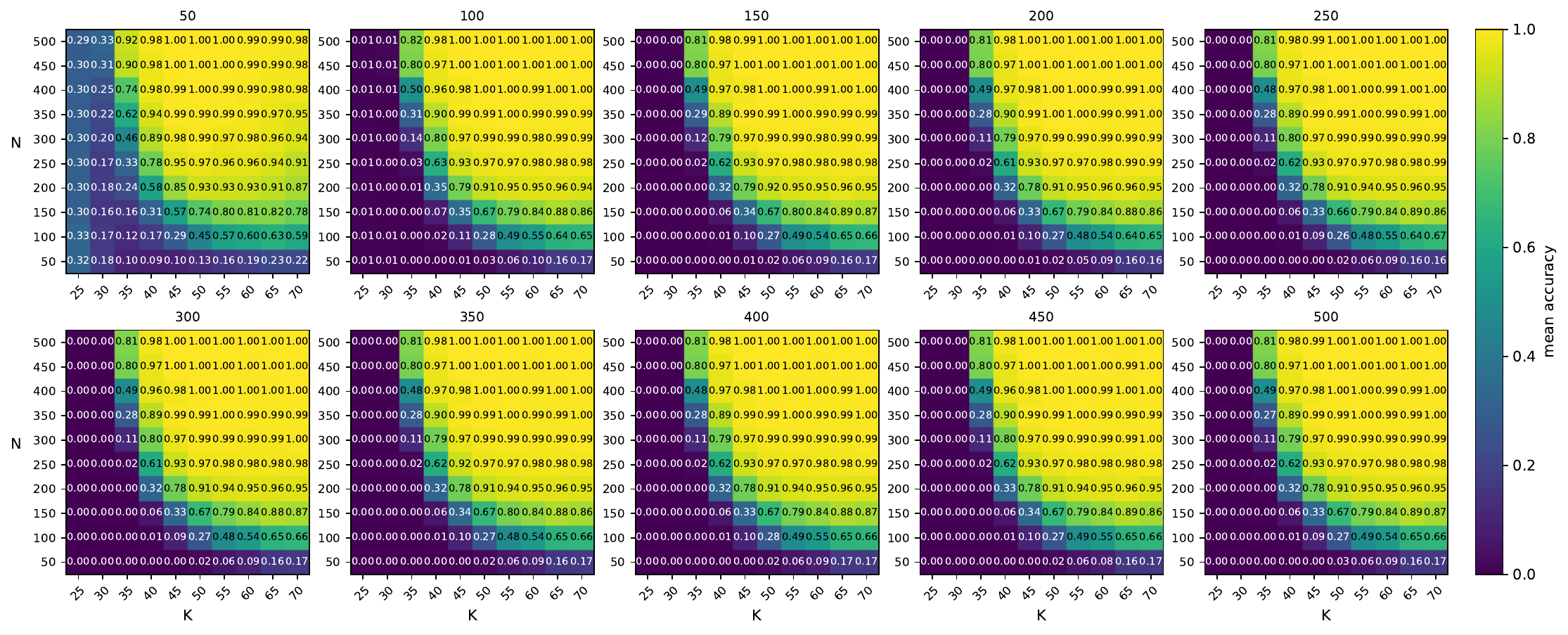}
        \vspace{-15pt}
        \caption{\small Alphabet language $A_{15}$}
    \end{subfigure}
    \begin{subfigure}[b]{\textwidth}
        \vspace{-0.5cm} %
        \centering
        \includegraphics[width=\linewidth]{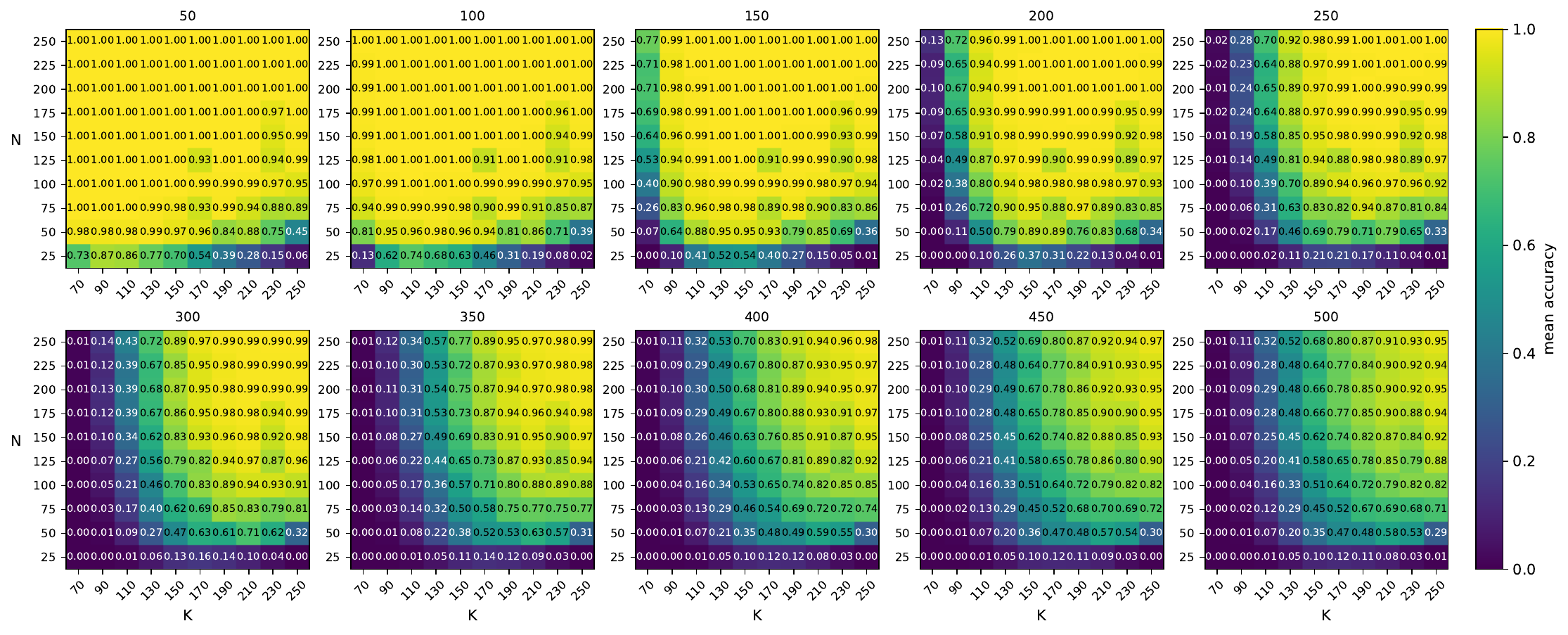}
        \vspace{-15pt}
        \caption{\small Induction Head task $\mathit{IH}_7$}
    \end{subfigure}
      \vspace{-0.5cm} %
    \caption{\small Tradeoff between training length and data complexity: evaluation on absolute lengths, from length $50$ to length $500$ with additive incrementation of $50$.}
\end{figure}

\begin{figure}[H]
    \centering
    \begin{subfigure}[b]{\textwidth}
        \centering
        \includegraphics[width=\linewidth]{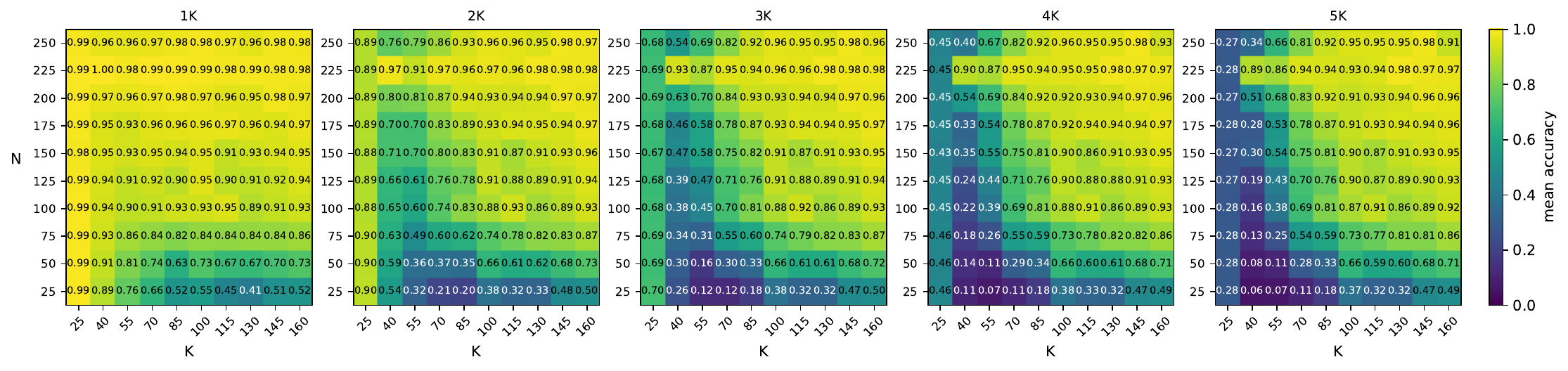}
        \vspace{-15pt}
        \caption{\small Alternating language $L_6$}
    \end{subfigure}
    \vspace{0.5cm} %
    \begin{subfigure}[b]{\textwidth}
        \centering
        \includegraphics[width=\linewidth]{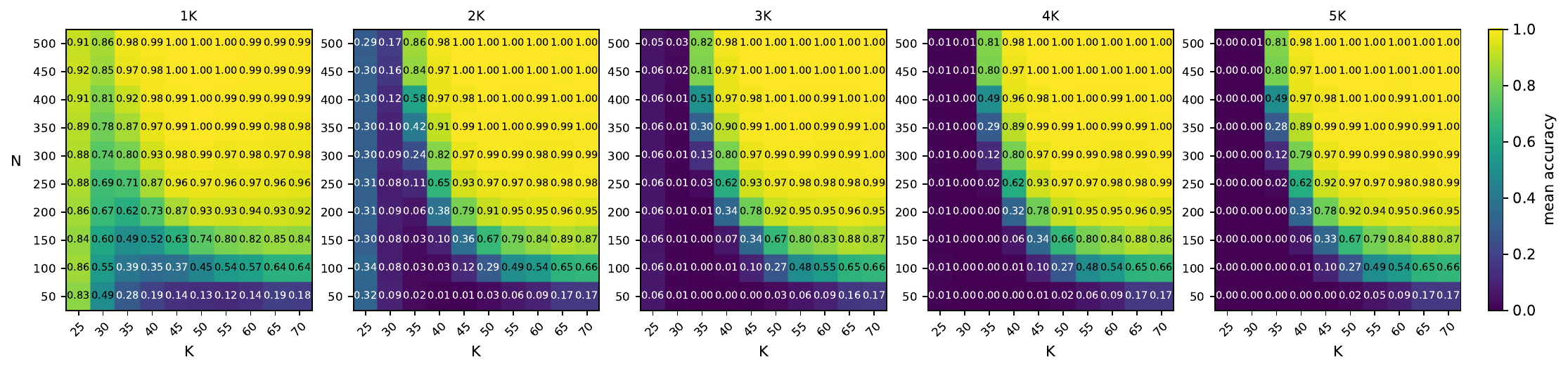}
        \vspace{-15pt}
        \caption{\small Alphabet language $A_{15}$}
    \end{subfigure}
    \begin{subfigure}[b]{\textwidth}
        \vspace{-0.5cm} %
        \centering
        \includegraphics[width=\linewidth]{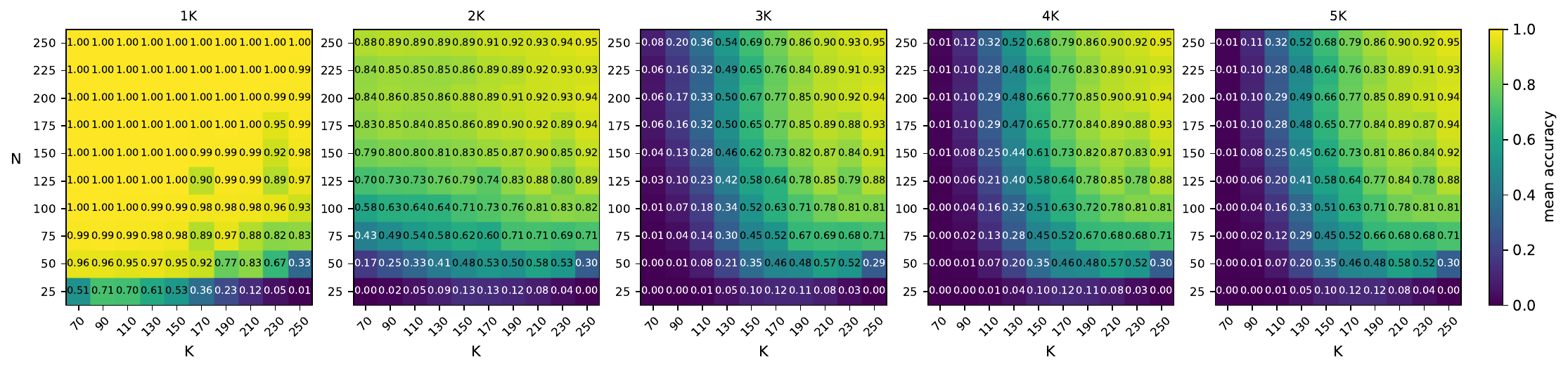}
        \vspace{-15pt}
        \caption{\small Induction Head task $\mathit{IH}_7$}
    \end{subfigure}
      \vspace{-0.5cm} %
    \caption{\small Tradeoff between training length and data complexity: evaluation on relative lengths $1K, 2K, 3K, 4K, 5K$, with $K$ the training length of the involved cells configuration.}
\end{figure}

\begin{figure}[H]
    \centering
    \begin{subfigure}[b]{\textwidth}
        \centering
        \includegraphics[width=\linewidth]{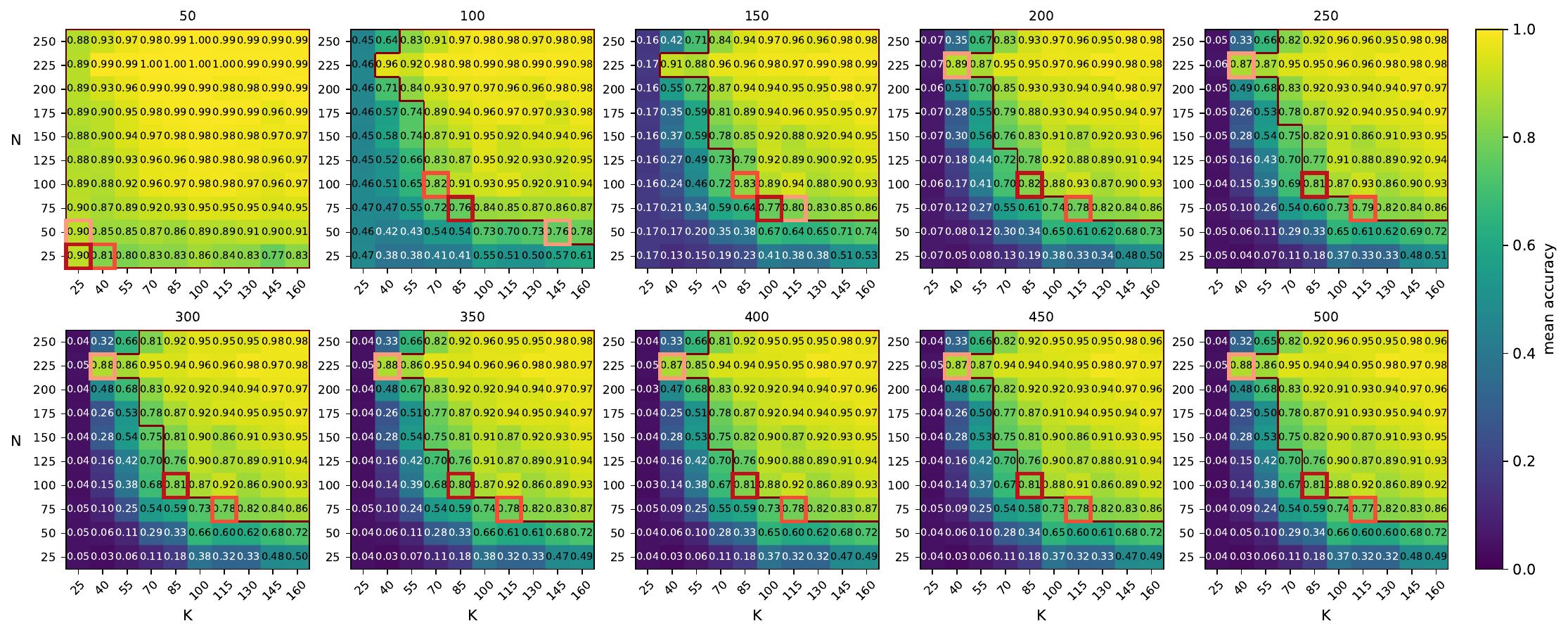}
        \vspace{-15pt}
        \caption{\small Alternating language $L_6$}
    \end{subfigure}
    \vspace{0.5cm} %
    \begin{subfigure}[b]{\textwidth}
        \centering
        \includegraphics[width=\linewidth]{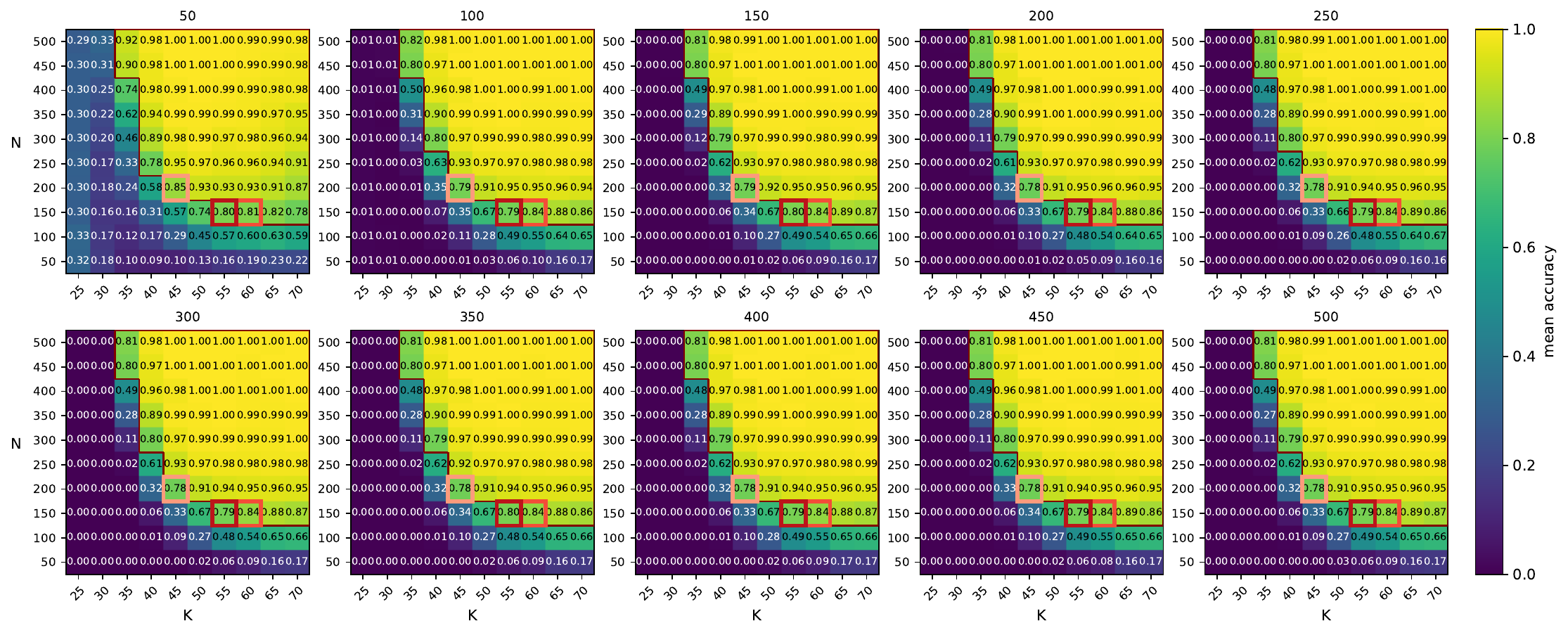}
        \vspace{-15pt}
        \caption{\small Alphabet language $A_{15}$}
    \end{subfigure}
    \begin{subfigure}[b]{\textwidth}
        \vspace{-0.5cm} %
        \centering
        \includegraphics[width=\linewidth]{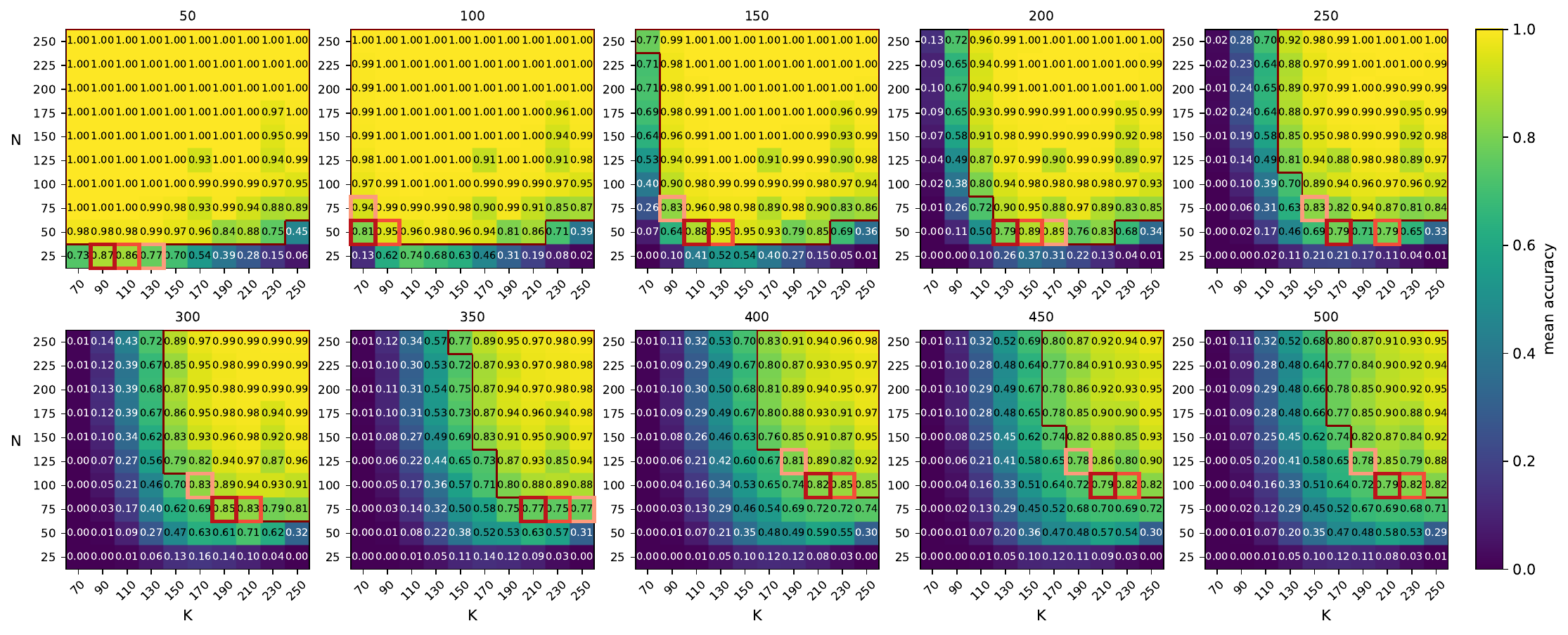}
        \vspace{-15pt}
        \caption{\small Induction Head task $\mathit{IH}_7$}
    \end{subfigure}
      \vspace{-0.5cm} %
    \caption{\small Tradeoff between training length and data complexity: evaluation on absolute lengths, from length $50$ to length $500$ with additive incrementation of $50$. Red contouring for accuracy $\geq 0.75$; highlighted in red are the 3 configurations under this constraint that have the smallest token-budget.}
\end{figure}

\begin{figure}[H]
    \centering
    \begin{subfigure}[b]{\textwidth}
        \centering
        \includegraphics[width=\linewidth]{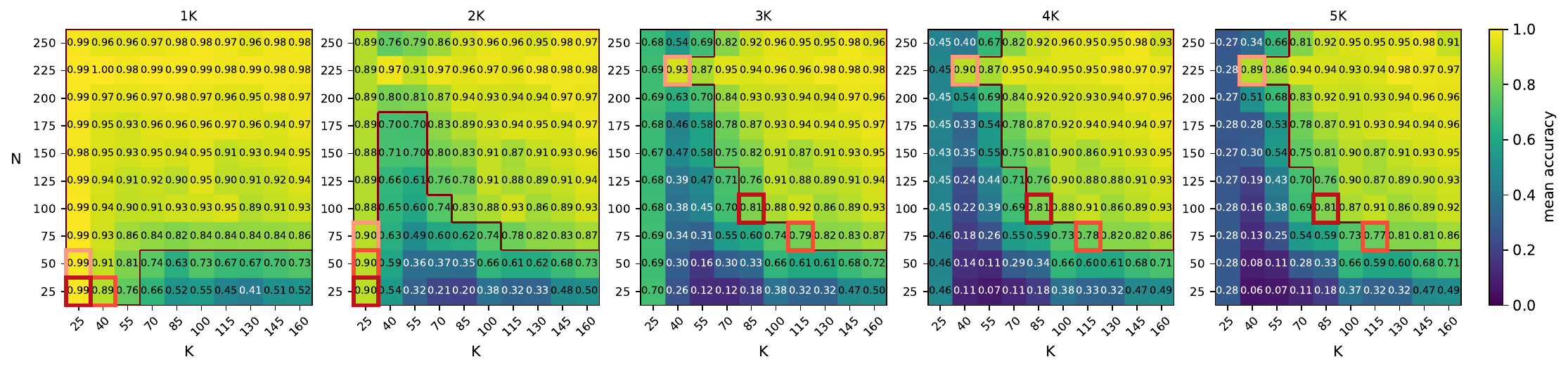}
        \vspace{-15pt}
        \caption{\small Alternating language $L_6$}
    \end{subfigure}
    \vspace{0.5cm} %
    \begin{subfigure}[b]{\textwidth}
        \centering
        \includegraphics[width=\linewidth]{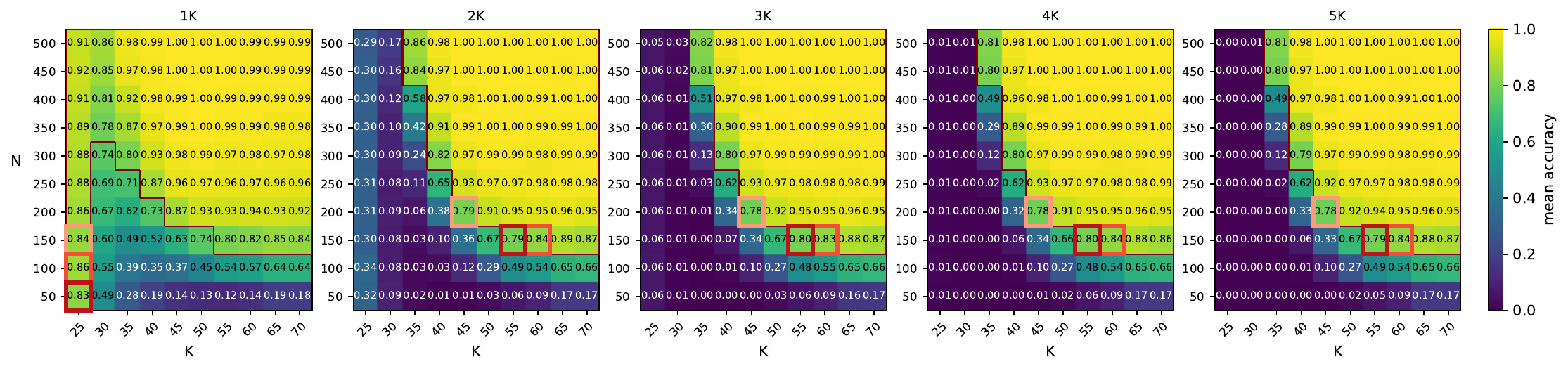}
        \vspace{-15pt}
        \caption{\small Alphabet language $A_{15}$}
    \end{subfigure}
    \begin{subfigure}[b]{\textwidth}
        \vspace{-0.5cm} %
        \centering
        \includegraphics[width=\linewidth]{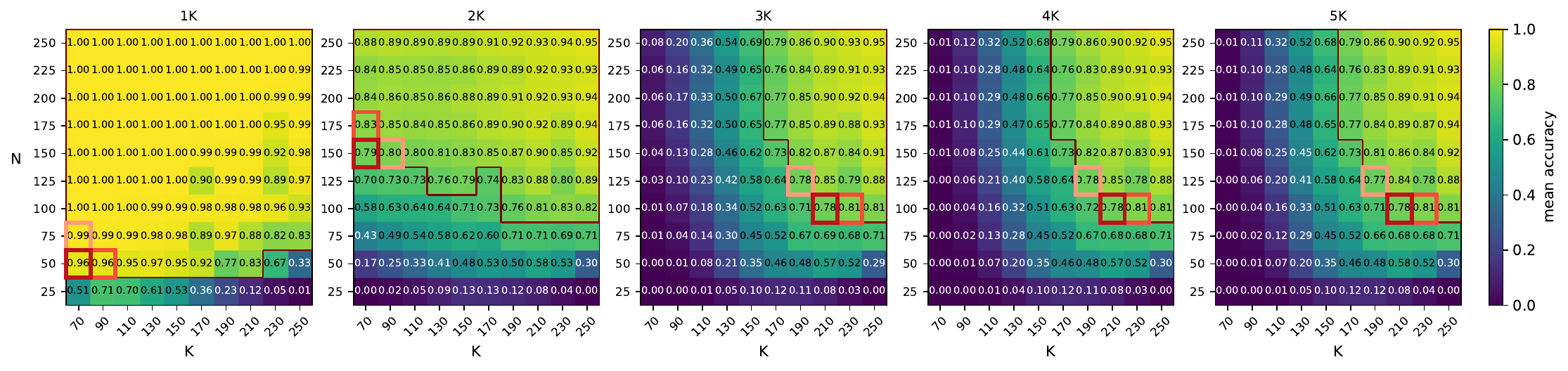}
        \vspace{-15pt}
        \caption{\small Induction Head task $\mathit{IH}_7$}
    \end{subfigure}
      \vspace{-0.5cm} %
    \caption{\small Tradeoff between training length and data complexity: evaluation on relative lengths $1K, 2K, 3K, 4K, 5K$, with $K$ the training length of the involved cells configuration. Red contouring for accuracy $\geq 0.75$; highlighted in red are the 3 configurations under this constraint that have the smallest token-budget.}
\end{figure}

\begin{figure}[H]
    \centering
    \begin{subfigure}[b]{\textwidth}
        \centering
        \includegraphics[width=\linewidth]{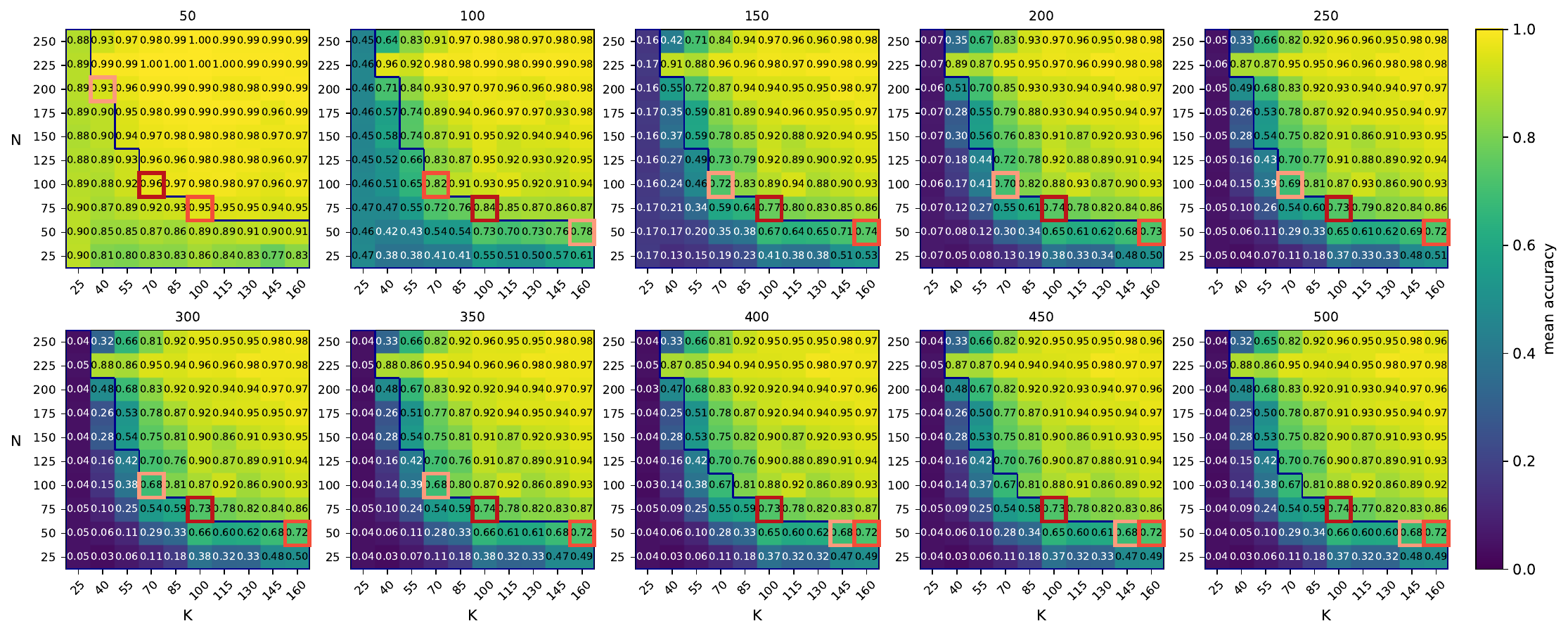}
        \vspace{-15pt}
        \caption{\small Alternating language $L_6$. Contouring cells with maximum token-budget $8000$}
    \end{subfigure}
    \vspace{0.5cm} %
    \begin{subfigure}[b]{\textwidth}
        \centering
        \includegraphics[width=\linewidth]{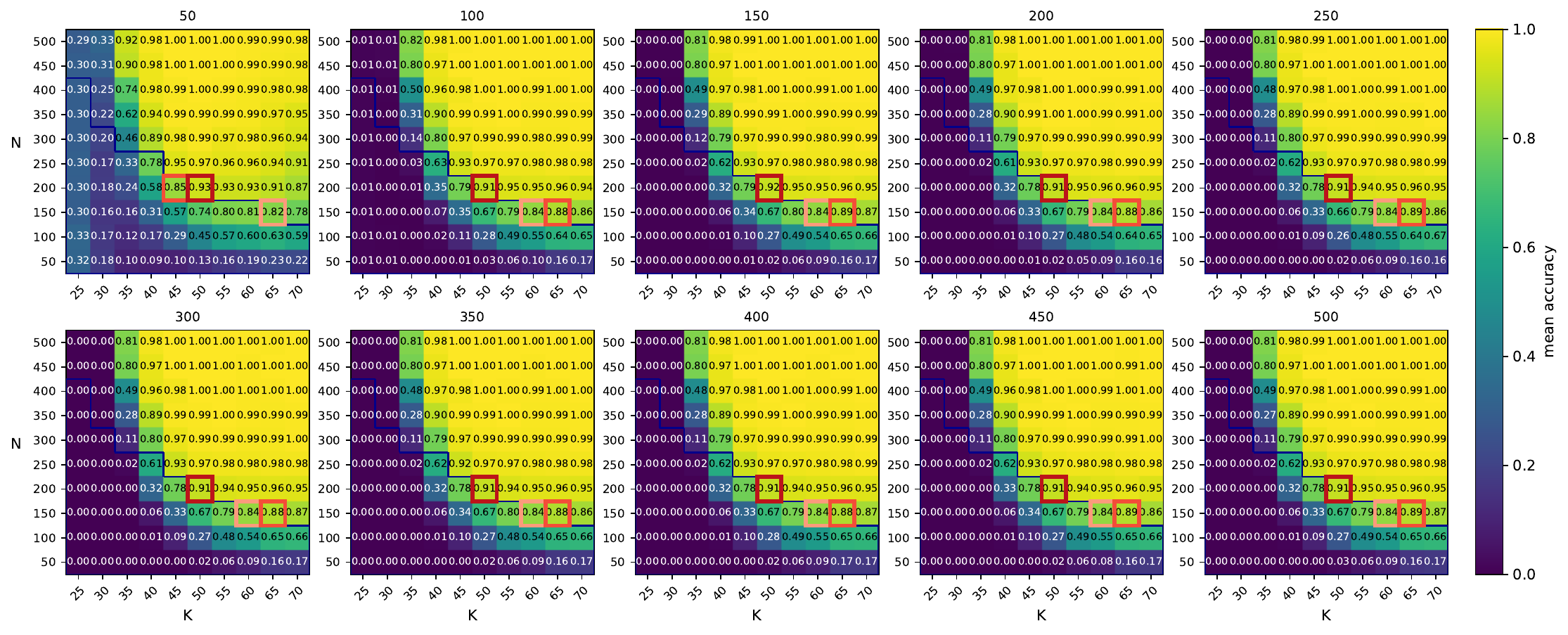}
        \vspace{-15pt}
        \caption{\small Alphabet language $A_{15}$. Contouring cells with maximum token-budget $10000$}
    \end{subfigure}
    \begin{subfigure}[b]{\textwidth}
        \vspace{-0.5cm} %
        \centering
        \includegraphics[width=\linewidth]{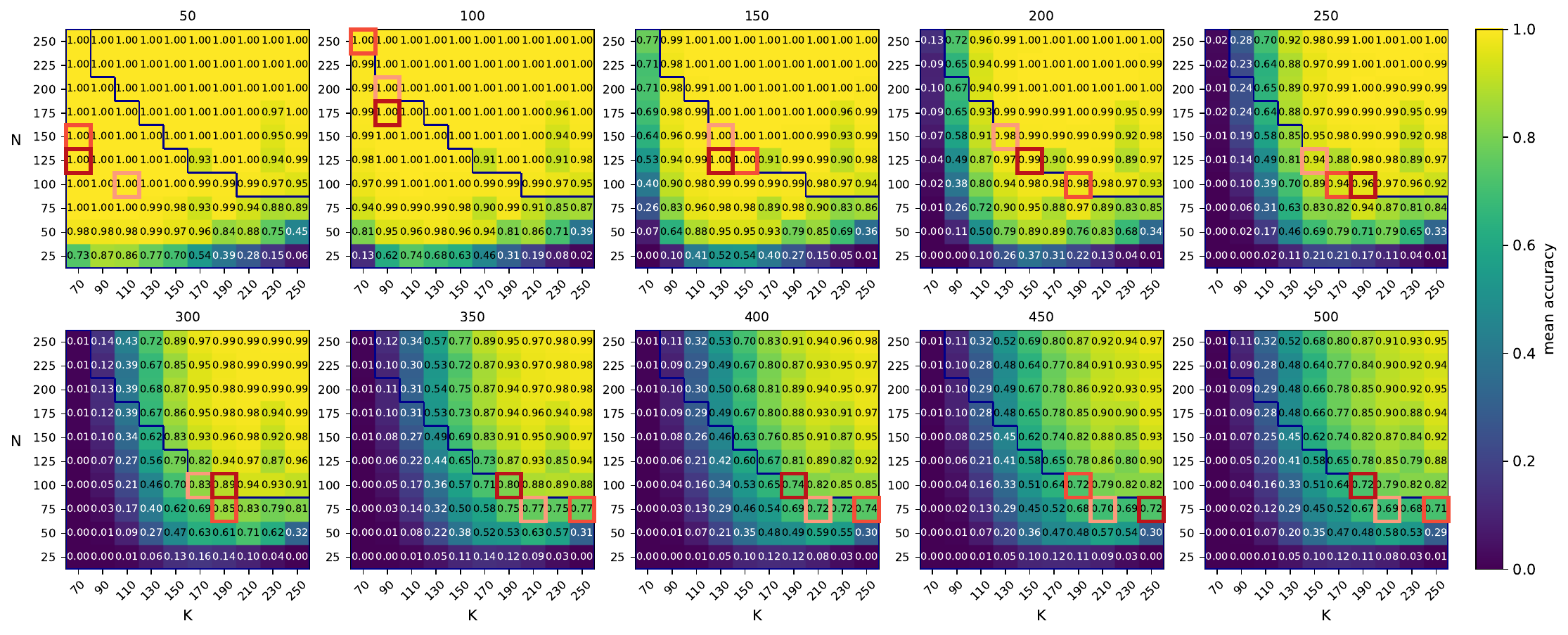}
        \vspace{-15pt}
        \caption{\small Induction Head task $\mathit{IH}_7$. Contouring cells with maximum token-budget $20000$}
    \end{subfigure}
      \vspace{-0.5cm} %
    \caption{\small Tradeoff between training length and data complexity: evaluation on absolute lengths, from length $50$ to length $500$ with additive incrementation of $50$. For their respective maximum token-budget threshold, highlighted in red are the 3 configurations that achieve the best accuracy.}
\end{figure}

\begin{figure}[H]
    \centering
    \begin{subfigure}[b]{\textwidth}
        \centering
        \includegraphics[width=\linewidth]{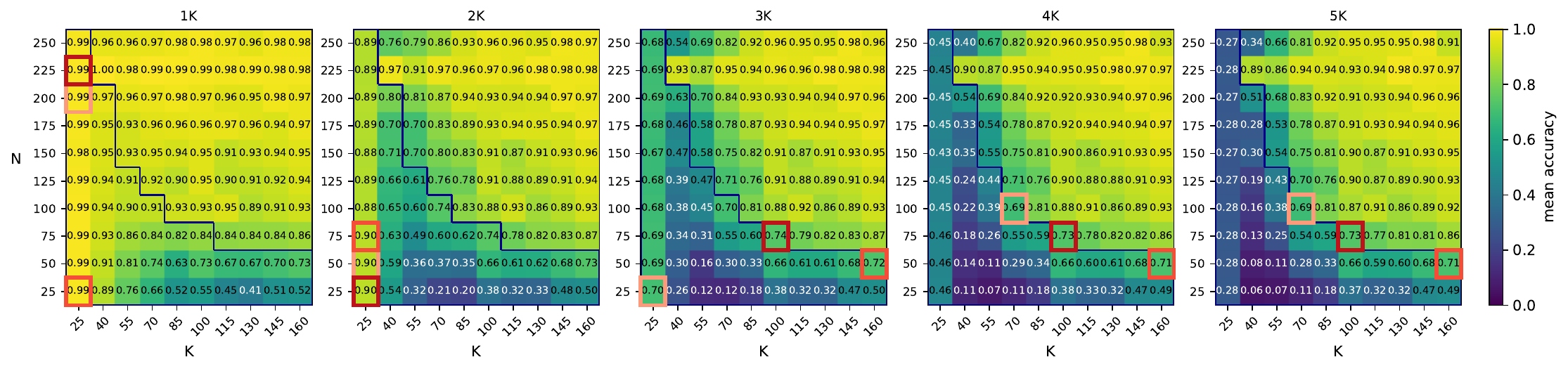}
        \vspace{-15pt}
        \caption{\small Alternating language $L_6$. Contouring cells with maximum token-budget $8000$}
    \end{subfigure}
    \vspace{0.5cm} %
    \begin{subfigure}[b]{\textwidth}
        \centering
        \includegraphics[width=\linewidth]{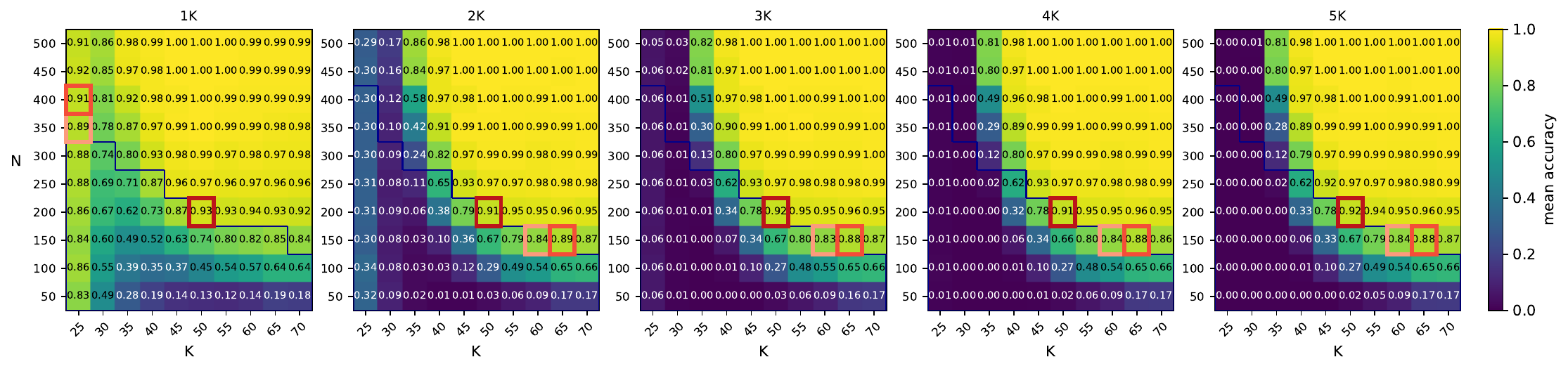}
        \vspace{-15pt}
        \caption{\small Alphabet language $A_{15}$. Contouring cells with maximum token-budget $10000$}
    \end{subfigure}
    \begin{subfigure}[b]{\textwidth}
        \vspace{-0.5cm} %
        \centering
        \includegraphics[width=\linewidth]{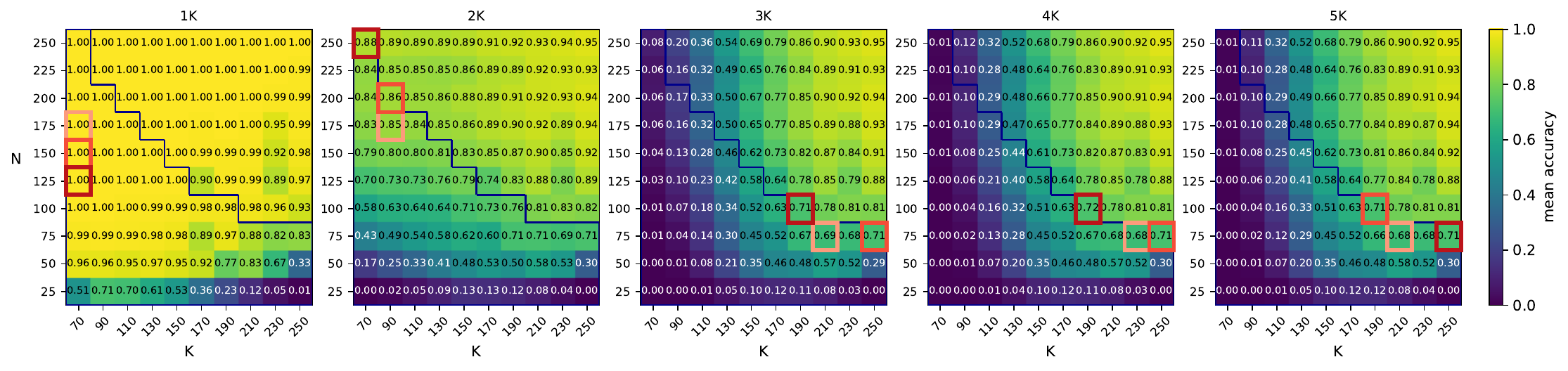}
        \vspace{-15pt}
        \caption{\small Induction Head task $\mathit{IH}_7$. Contouring cells with maximum token-budget $20000$}
    \end{subfigure}
      \vspace{-0.5cm} %
    \caption{\small Tradeoff between training length and data complexity: evaluation on relative lengths $1K, 2K, 3K, 4K, 5K$, with $K$ the training length of the involved cells configuration. For their respective maximum token-budget threshold, highlighted in red are the 3 configurations that achieve the best accuracy.}
\end{figure}

\newpage

\section{AI Usage}

In this work, we used generative AI tools for proofreading the paper and its mathematical content, revising the phrasing of mathematical claims, providing ingredients for proving mathematical claims, and assisting in the writing and proofreading of proofs. 
Additionally, we used generative AI tools for creating and editing software code, assisting with literature research for related work, and proposing edits for improving readability. 

We have \emph{not} used generative AI tools to help develop theoretical models or conceptual frameworks, propose or refine hypotheses, design or provide feedback on research  methodology or experiments, support qualitative and thematic data analysis, or interpret results.

We have reviewed all AI-assisted work. For example, we checked LLM-generated research advices for potential plagiarism through a manual literature survey, we checked mathematical soundness of every mathematical claims, we tested the soundness of the LLM-generated/assisted code experiments, and we supervised and/or overwrote AI drafted writing in the paper. We take responsibility for the final content of this work, including text, claims or artifacts produced with the aid of generative AI.

\end{document}